\documentclass{article}

     \PassOptionsToPackage{square,numbers}{natbib}

\usepackage[final]{neurips_2021_new}

\usepackage[utf8]{inputenc} 
\usepackage[T1]{fontenc}    
\usepackage{hyperref}       
\usepackage{url}            
\usepackage{booktabs}       
\usepackage{amsfonts}       
\usepackage{nicefrac}       
\usepackage{microtype}      
\usepackage{xcolor}         

\usepackage{amsmath,amssymb,amsthm}
\usepackage{dsfont}

\usepackage{algorithm}
\usepackage{algorithmic}
\usepackage{graphics}
\usepackage{comment}
\usepackage{xspace}
\usepackage{enumitem}

\definecolor{Bleu}{RGB}{30,144,255}
\hypersetup{colorlinks, citecolor=Bleu}
\newcommand{\eg}{\textit{e}.\textit{g}.,~}
\newcommand{\ie}{\textit{i}.\textit{e}.,~}

\newcommand{\algo}{\textsc{MisLid}\xspace}
\newcommand{\mineig}{2L^2}

\newtheorem{lemma}{Lemma}
\newtheorem*{lemma*}{Lemma}

\newtheorem{corollary}{Corollary}

\newtheorem{theorem}{Theorem}
\newtheorem*{theorem*}{Theorem}

\theoremstyle{definition}
\newtheorem{definition}{Definition}
\newtheorem{assumption}{Assumption}
\newtheorem{remark}{Remark}

\DeclareMathOperator*{\argmin}{arg\,min}

\DeclareMathOperator{\KL}{KL}
\DeclareMathOperator{\kl}{kl}

\providecommand{\abs}[1]{\left\lvert#1\right\rvert}
\providecommand{\norm}[1]{\left\lVert#1\right\rVert}

\newcommand{\indi}[1]{\mathds{1}\left\{#1\right\}}

\newcommand{\correction}[1]{#1}

\newcommand{\numel}{5}
\newcommand{\etabound}{\correction{(LK+1)\varepsilon}}
\newcommand{\etaboundsq}{\correction{(LK+1)^2\varepsilon^2}}

\usepackage[colorinlistoftodos, textwidth=18mm]{todonotes}

\newcommand{\todocrout}[1]{\todo[color=cyan!10]{\scriptsize CR: #1}}

\title{Dealing With Misspecification In Fixed-Confidence Linear Top-m Identification}

\author{%
  Cl\'{e}mence R\'{e}da\\
  Universit\'{e} de Paris, NeuroDiderot, Inserm, F-75019 Paris, France\\
  \texttt{clemence.reda@inria.fr}
  \And
  Andrea Tirinzoni\\
  Univ. Lille, Inria, CNRS, Centrale Lille, UMR 9189 CRIStAL, F-59000 Lille, France\\
  \correction{\texttt{andrea.tirinzoni@inria.fr}}
  \And
  R\'{e}my Degenne\\
  Univ. Lille, Inria, CNRS, Centrale Lille, UMR 9189 CRIStAL, F-59000 Lille, France\\
  \texttt{remy.degenne@inria.fr} \\
}

\usepackage{minitoc}

\begin{document}

\maketitle

\doparttoc 
\faketableofcontents 

\begin{abstract}
We study the problem of the identification of $m$ arms with largest means under a fixed error rate $\delta$ (fixed-confidence Top-$m$ identification), for misspecified linear bandit models. This problem is motivated by practical applications, especially in medicine and recommendation systems, where linear models are popular due to their simplicity and the existence of efficient algorithms, but in which data inevitably deviates from linearity. In this work, we first derive a tractable lower bound on the sample complexity of any $\delta$-correct algorithm for the general Top-$m$ identification problem. We show that knowing the scale of the deviation from linearity is necessary to exploit the structure of the problem. We then describe the first algorithm for this setting, which is both practical and adapts to the amount of misspecification. We derive an upper bound to its sample complexity which confirms this adaptivity and that matches the lower bound when $\delta\rightarrow 0$. Finally, we evaluate our algorithm on both synthetic and real-world data, showing competitive performance with respect to existing baselines.
\end{abstract}

\section{Introduction}\label{sec:introduction}

The multi-armed bandit (MAB) is a popular framework to model sequential decision making problems. At each round $t > 0$, a learner chooses an \emph{arm} $k_t$ among a finite set of $K \in \mathbb{N}$ possible options, and it receives a random reward $X^{k_t}_t \in \mathbb{R}$ drawn from a distribution $\nu^{k_t}$ with unknown mean $\mu^{k_t}$.
Among the many problem settings studied in this context, we focus on \emph{pure exploration}, where the learner aims at maximizing the information gain for answering a given query about the arms~\citep{bubeck2009pure}. In particular, we are interested in finding a subset of $m\geq 1$ arms with largest expected reward, which is known as the \emph{Top-$m$ identification} problem~\citep{kalyanakrishnan2010efficient}. This generalizes the widely-studied best-arm (i.e., Top-$1$) identification problem~\citep{even2003action}.
This problem has several important applications, including online recommendation and drug repurposing~\citep{mason2020finding,reda2021top}.
Two objectives are typically studied. On the one hand, in the \emph{fixed-budget} setting~\citep{audibert2010best}, the learner is given a finite amount of samples and must return a subset of $m$ best arms while minimizing the probability of error in identification. On the other hand, in the \emph{fixed-confidence} setting~\citep{even2003action}, the learner aims at minimizing the \emph{sample complexity} for returning a subset of $m$ best arms with a fixed maximum error rate $\delta \in (0,1)$, defined as the number of samples collected before the algorithm stops. This paper focuses on the latter. 

In practice, information about the arms is typically available (e.g., the characteristics of an item in a recommendation system, or the influence of a drug on protein production in a clinical application). This side information influence the expected rewards of the arms, thus adding \emph{structure} (\ie prior knowledge) to the problem.
This is in contrast to the classic \emph{unstructured} MAB setting, where the learner has no prior knowledge about the arms. Due to their simplicity and flexibility, linear models have become the most popular to represent this structure. Formally, in the \emph{linear bandit} setting~\citep{auer2002using}, the mean reward $\mu^k$ of each arm $k\in[K] := \{1,2,\dots,K\}$ is assumed to be an \correction{inner product }
between known $d$-dimensional arm features $\phi_k \in \mathbb{R}^d$ and an unknown parameter $\theta\in\mathbb{R}^d$.
This model has led to many provably-efficient algorithms for both best-arm~\citep{soare2014best,xu2018fully,fiez2019sequential,zaki2020explicit,degenne2020gamification} and Top-$m$ identification~\citep{katz2020empirical,reda2021top}. Unfortunately, the strong guarantees provided by these algorithms hold only when the expected rewards are perfectly linear in the given features, a property that is often violated in real-world applications.
In fact, when using linear models with real data, one inevitably faces the problem of \emph{misspecification}, i.e., the situation in which the data deviates from linearity. 

A \emph{misspecified linear bandit} model is often described as a linear bandit model with an additive term to encode deviation from linearity. Formally, the expected reward $\mu^k = \phi_k^\top \theta + \eta^k$ of each arm $k \in [K]$ can be decomposed into its linear part $\phi_k^\top \theta$ and its misspecification $\eta^k \in \mathbb{R}$.
Note the flexibility of this model: for $\|\eta\| = 0$, where $\eta = [\eta^1, \eta^2, \dots, \eta^K]^\top$, the problem is perfectly linear and thus highly structured, as the mean rewards of different arms are related through the common parameter $\theta$; whereas when the misspecification vector $\eta$ is large in all components, the problem reduces to an unstructured one, since knowing the linear part alone provides almost no information about the expected rewards.
Learning in this setting thus requires adapting to the scale of misspecification, typically under the assumption that some information about the latter is known (e.g., an upper bound $\varepsilon$ to $\|\eta\|$).
Due to its importance, this problem has recently gained increasing attention in the bandit community for regret minimization~\citep{ghosh2017misspecified,lattimore2020learning,foster2020adapting,pacchiano2020model,takemura2021parameter}.
However it has not been addressed in the context of pure exploration. In this paper, we take a step towards bridging this gap by studying fixed-confidence Top-$m$ identification in the context of misspecified linear bandits. Our detailed contributions are as follows.

\textbf{Contributions.} (1) We derive a tractable lower bound on the sample complexity of any $\delta$-correct algorithm for the general Top-$m$ identification problem. (2) Leveraging this lower bound, we show that knowing an upper bound $\varepsilon$ to $\|\eta\|$ is necessary for adapting to the scale of misspecification, in the sense that any $\delta$-correct algorithm without such information cannot achieve a better sample complexity than that obtainable when no structure is available. (3) We design the first algorithm for Top-$m$ identification in misspecified linear bandits. We derive an upper bound to its sample complexity that holds for any $\delta\in(0,1)$ and that matches our lower bound for $\delta\rightarrow 0$. Notably, our analysis reveals a nice adaptation to the value of $\varepsilon$, recovering state-of-the-art dependences in the linear case ($\varepsilon=0$), where the sample complexity scales polynomially in $d$ and not in $K$, and in the unstructured case ($\varepsilon$ large), where only polynomial terms in $K$ appear.
(4) We evaluate our algorithm on synthetic problems and real datasets from drug repurposing and recommendation system applications, while showing competitive performance with state-of-the-art methods.

\textbf{Related work.} While model misspecification has not been addressed in the pure exploration literature, several attempts to tackle this problem in the context of regret minimization exist. In~\citep{ghosh2017misspecified}, the authors show that, if $T$ is the learning horizon, for any bandit algorithm which enjoys $\mathcal{O}(d\sqrt{T})$ regret scaling on linear models, there exists a misspecified instance where the regret is necessarily linear.
As a workaround, the authors design a statistical test based on sampling a subset of arms prior to learning to decide whether a linear or an unstructured bandit algorithm should be run on the data.
Similar ideas are presented in~\citep{chatterji2020osom}, where the authors design a sequential test to switch online between linear and unstructured models.
More recently, elimination-based algorithms~\citep{lattimore2020learning,takemura2021parameter} and model selection methods~\citep{pacchiano2020model,foster2020adapting} have attracted increasing attention.
Notably, these algorithms adapt to the amount of misspecification $\varepsilon$ \emph{without} knowing it beforehand, at the cost of an additive linear term that scales with $\varepsilon$.
Moreover, while best-arm identification has been the focus of many prior works in the realizable linear setting, some suggesting asymptotically-optimal algorithms~\citep{degenne2020gamification,jedra2020optimal}, Top-$m$ identification has been seldom studied in terms of problem-dependent lower bounds.
Lower bounds for the unstructured Top-$m$ problem have been derived previously, focusing on explicit bounds~\citep{kaufmann2016complexity}, on getting the correct dependence in the problem parameters for any confidence $\delta$~\citep{chen2017nearly,simchowitz2017simulator}, or on asymptotic optimality (as $\delta \to 0$)~\citep{garivier2016optimal}.
Because of the combinatorial nature of the Top-$m$ identification problem, obtaining a tractable, tight, problem-dependent lower bound is not straightforward.


\section{Setting}\label{sec:setting}

At successive stages $t \in \mathbb{N}$, the learner samples an arm $k_t \in [K]$ based on previous observations and internal randomization (a random variable $U_t \in [0,1]$) and observes a reward $X_t^{k_t}$.
Let $\mathcal{F}_t := \sigma(\{U_1, k_1, X^{k_1}_1, \dots, U_t, k_t, X^{k_t}_t, U_{t+1}\})$ be the $\sigma$-algebra associated with past sampled arms and rewards until time $t$. Then $k_t$ is a $\mathcal{F}_{t-1}$-measurable random variable. The reward $X_t^{k_t}$ is sampled from $\nu^{k_t}$ and is independent of all past observations, conditionally on $k_t$.
We suppose that the noise is Gaussian with variance $1$, 
such that the observation when pulling arm $k_t$ at time $t$ is $X_t^{k_t} \sim \mathcal N(\mu^{k_t}, 1)$. The mean vector $\mu = (\mu^k)_{k \in [K]} \in \mathbb{R}^K$ then fully describes the reward distributions.

In a misspecified linear bandit, each arm $k \in [K]$ is described by a feature vector $\phi_k \in \mathbb{R}^d$. The corresponding feature matrix is denoted by $A := [\phi_1, \phi_2, \dots, \phi_K]^\top \in \mathbb{R}^{K \times d}$ and the maximum 
\correction{$\ell_2$}
-norm of these vectors is $L := \max_{k\in [K]}\norm{\phi_k}_2$.
We assume that the feature vectors span $\mathbb{R}^d$ (otherwise we could rewrite those vectors in a subspace of smaller dimension).
We assume that the learner is provided with a set of realizable models
\begin{equation}\label{eq:set-models}
	\mathcal{M} := \big\{\mu \in \mathbb{R}^K \mid \exists \theta \in \mathbb{R}^d\ \exists \eta \in \mathbb{R}^K,\ \mu = A \theta + \eta \ \land \  \|\mu\|_{\infty} \leq M \ \land \  \|\eta\|_\infty \leq \varepsilon \big\}, 
\end{equation}
where $M,\varepsilon \in \mathbb{R}$ are known upper bounds on the $\ell^{\infty}$-norm of the mean\footnote{The restriction to $\|\mu\|_{\infty} \leq M$ is required only for our analysis, while it can be safely dropped in practice.} and misspecification vectors, respectively.
Intuitively, $\mathcal{M}$ represents the set of bandit models whose mean vector $\mu$ is linear in the features $A$ only up to some misspecification $\eta$.

We consider  Top-$m$ identification in the fixed-confidence setting. Given a confidence parameter $\delta \in (0,1)$, the learner is required to output the $m\in[K]$ arms of the unknown bandit model $\mu\in\mathcal{M}$ with highest means with probability at least $1-\delta$.
The strategy of a bandit algorithm designed for Top-$m$ identification can be decomposed into three rules: a \emph{sampling} rule, which selects the arm $k_t$ to sample at a given learning round $t$ according to past observations;
 a \emph{stopping} rule, which determines the end of the learning phase, and is a stopping time with respect to the filtration $(\mathcal{F}_t)_{t > 0}$, denoted by $\tau_\delta$; finally, a \emph{decision} rule, which returns a $\mathcal{F}_{\tau_\delta}$-measurable answer to the pure exploration problem.
An answer is a set $\hat{S}_m \subseteq [K]$ with exactly $m$ arms: $\lvert\hat{S}_m\rvert = m$. 
In our context, the ``$m$ best arms of $\mu$'' might not be well defined since the set $S^\star(\mu):=\{k\in[K] \mid \mu^k \geq \max_{i\in[K]}^m \mu^i\}$\footnote{The expression $\max^m_{i \in S} f(i)$ denotes the $m^{th}$ maximal value in $\{f(i) \mid i \in S\}$.} might contain more than $m$ elements if some arms have the same mean. Thus, let $\mathcal S_m(\mu) = \{ S \subseteq S^\star(\mu) \mid |S| = m \}$ be the set containing all subsets of $m$ elements of $S^\star(\mu)$. 

\begin{definition}[$\delta$-correctness]
For $\delta\in(0,1)$, we say that an algorithm $\mathfrak{A}$ is $\delta$-correct on $\mathcal M$ if, for all $\mu \in \mathcal M$, $\tau_\delta < + \infty$ almost surely and
$
\mathbb{P}_\mu^{\mathfrak{A}}\left(\hat{S}_m \notin \mathcal S_m(\mu) \right)
\le \delta
\: .
$
\end{definition}


\section{Tractable lower bound for the general Top-$m$ identification problem}\label{sec:lower_bound_top_m}

Let $N^k_{t}$ denote the number of times arm $k$ has been sampled until time $t$ included. Suppose that the true model $\mu$ has exactly $m$ arms that are among the top-$m$, \ie that $|S^\star(\mu)| = m$ and $\mathcal{S}_{m}(\mu) = \{ S^\star(\mu) \}$.
Consider the following set of alternatives to $\mu$,
\begin{align*}
\Lambda_m(\mu) := \left\{\lambda\in\mathcal{M} \mid \mathcal{S}_{m}(\lambda) \cap \mathcal{S}_{m}(\mu) = \emptyset \right\} \: ,
\end{align*}
that is, the set of all bandit models $\lambda$ in $\mathcal{M}$ where the top-$m$ arms of $\mu$ are not among the top-$m$ arms of $\lambda$.
Note that, while we assumed that the set of top-$m$ arms in $\mu$ is unique, this might not be the case for $\lambda$.
Define the event $E_{\tau_\delta} := \{\hat{S}_m \in \mathcal{S}_{m}(\mu)\}$ that the answer returned by the algorithm at $\tau_\delta$ is correct for $\mu$ and consider any $\delta$-correct algorithm $\mathfrak{A}$.
Let us call $\KL$ the Kullback-Leibler divergence\footnote{We abuse notation by denoting distributions in the same one-dimensional exponential family by their means.} and $\kl$ the binary relative entropy.
Then, using the change-of-measure argument proposed in~\cite[][Theorem 1]{garivier2016optimal}, for any $\lambda \in \Lambda_m(\mu)$ and $\delta \le 1/2$,
\begin{align*}
	\sum_{k\in[K]} \mathbb{E}_{\mu}^{\mathfrak{A}}[N_\tau^k]\KL\left(\mu^{k},\lambda^k\right)
  \geq \kl\left(\mathbb{P}_{\mu}^\mathfrak{A}(E_{\tau_\delta}), \mathbb{P}_{\lambda}^\mathfrak{A}(E_{\tau_\delta})\right)
  \geq \kl(1-\delta, \delta) \geq \log \left( \frac{1}{2.4\delta} \right),
\end{align*}
where the second-last inequality follows from the $\delta$-correctness of the algorithm and the monotonicity of the function $\kl$. This holds for any $\lambda\in\Lambda_m(\mu)$, so we have that
\begin{align}
	\mathbb{E}_{\mu}^{\mathfrak{A}}[\tau]
  \geq \left( \sup_{\omega\in\Delta_K}\inf_{\lambda\in\Lambda_m(\mu)}\sum_{k\in[K]} \omega^k\mathrm{KL}\left(\mu^{k},\lambda^k\right) \right)^{-1}
    \log \left( \frac{1}{2.4\delta} \right)
  \: , \label{eq:lower_bound}
\end{align}
with $\Delta_K := \{ p \in [0,1]^K \mid \sum_{k=1}^K p_k = 1\}$ the simplex on $[K]$. We define the inverse complexity $H_\mu := \sup_{\omega\in\Delta_K}\inf_{\lambda\in\Lambda_m(\mu)}\sum_{k\in[K]} \omega^k\mathrm{KL}\left(\mu^{k},\lambda^k\right)$~.
Computing that lower bound might be difficult: while the Kullback-Leibler is convex for Gaussians, the set $\Lambda_m(\mu)$ over which it is minimized is non-convex.
Its description using $\mathcal S_m(\lambda)$ is combinatorial: we can write $\Lambda_m(\mu)$ as a union of convex sets, one for each subset of top-$m$ arms of $\lambda$, but this implies minimizing over $\left(^K_m\right)$ sets, which is not practical.
In order to rewrite this lower bound, we prove the following lemma in Appendix~\ref{app:lower_bound}.


\begin{lemma}\label{lemma:rewrite_lower_bound}
$\forall \mu, \lambda \in \mathbb{R}^K$s.t. $\abs{S^\star(\mu)} = m$, $\mathcal{S}_{m}(\lambda) \cap \mathcal{S}_{m}(\mu) = \emptyset \Leftrightarrow \exists i \notin S^\star(\mu)\ \exists j \in S^\star(\mu), \lambda^i > \lambda^j$.
\end{lemma}

Lemma~\ref{lemma:rewrite_lower_bound} allows us to go from an exponentially costly optimization problem, which implied minimizing over $\left(^K_m\right)$ sets, to optimizing across $m(K-m)$ halfspaces. Therefore, \correction{by replacing the set of alternative models as derived in Lemma~\ref{lemma:rewrite_lower_bound}}, the lower bound in Equation~\ref{eq:lower_bound} can be rewritten in the following more convenient form~:
\begin{theorem}\label{th:lower_bound}
	For any $\delta \le 1/2$, for any $\delta$-correct algorithm $\mathfrak{A}$ on $\mathcal{M}$, for any bandit instance $\mu \in \mathbb{R}^K$ such that $|S^\star(\mu)| = m$, the following lower bound holds on the stopping time $\tau_\delta$ of $\mathfrak{A}$ on instance $\mu$:
\begin{align*}
	\mathbb{E}_{\mu}^{\mathfrak{A}}[\tau_\delta]
  \ge \left( \sup_{\omega\in\Delta_K}\min_{i \notin S^\star(\mu)}\min_{j \in S^\star(\mu)}\inf_{\lambda\in\mathcal{M} : \lambda^i > \lambda^j}
    \sum_{k\in[K]} \omega^k\mathrm{KL}\left(\mu^{k},\lambda^k\right) \right)^{-1} \log \left( \frac{1}{2.4\delta} \right).
\end{align*}
\end{theorem}
Computing the lower bound now requires performing one maximization over the simplex (which can be still hard), and $m(K-m)$ minimizations over half-spaces $\{\lambda\in\mathcal{M} : \lambda^i > \lambda^j\}$, where $(i,j) \in \left( \mathcal{S}^\star(\mu) \right)^c \times \mathcal{S}^\star(\mu)$.
The minimizations are convex optimization problems and can be solved efficiently.
Our algorithm inspired from that bound will need to perform only those minimizations.

Note that a lower bound for Top-$m$ identification using the cited change-of-measure argument has been obtained in~\cite{kaufmann2016complexity}.
Aiming to be more explicit, it relies on alternative models where one of the best arms is switched with the $(m+1)^{th}$ best one (or one of the $K-m$ worst ones with the $m^{th}$ best one).
These models are a strict subset of $\Lambda_m(\mu)$. Hence this bound is not as tight as the one in Theorem~\ref{th:lower_bound}, which is why the algorithm we detail in the next sections will rely on the latter instead.

\correction{Note that with $\varepsilon=0$ and $m=1$, this lower bound is exactly the one for best arm identification in perfectly linear models~\cite{fiez2019sequential}. }As the misspecification $\varepsilon$ grows, the set $\mathcal M$ becomes larger \correction{and so does the set of alternative models $\Lambda_m(\mu)$}, \correction{thus} 
the lower bound grows. In the limit $\varepsilon \to + \infty$, the model becomes the same as the unstructured model. We show that in fact the lower bound becomes exactly equal to the unstructured lower bound as soon as $\varepsilon > \varepsilon_\mu$, a finite value.
\begin{lemma}\label{lemma:match_unstructured}
There exists $\varepsilon_\mu \in \mathbb{R}$ with $\varepsilon_\mu \le \max_k \mu^k - \min_k \mu^k$ such that if $\varepsilon > \varepsilon_\mu$, then the lower bound of Theorem~\ref{th:lower_bound} is equal to the unstructured top-$m$ lower bound.
\end{lemma}
The proof is in Appendix~\ref{app:lower_bound}. It considers finitely supported distributions over $\Lambda_m(\mu)$ that realize the equilibrium in the max-min game of the lower bound. As soon as one of these equilibrium distributions for the unstructured problem has its whole support in the misspecified model, the two complexities are equal.


\subsection{Adaptation to unknown misspecification is impossible}\label{sec:adaptation_is_impossible}

We now make an important observation: knowing that a problem is misspecified without knowing an upper bound $\varepsilon$ on $\|\eta\|_\infty$ is the same as not knowing anything about the structure of that problem.

The lower bound of Equation~\eqref{eq:lower_bound} is a function of the set $\mathcal M$ of realizable models $\mu$.
Let $B(\mu, \delta, \mathcal M)$ be the right-hand side of that equation, such that $\mathbb{E}_\mu^{\mathfrak{A}}[\tau_\delta] \ge B(\mu, \delta, \mathcal M)$ for any algorithm $\mathfrak{A}$ which is $\delta$-correct on $\mathcal M$.
Suppose that we have $\mathcal M_1 \subseteq \mathcal M$, a subset of the model, for which we would like to have lower sample complexity (possibly at the cost of a higher sample complexity on $\mathcal M \setminus \mathcal M_1$).
If $\mathcal M$ is the misspecified linear model with deviation $\varepsilon$, let us say that $\mathcal M_1$ is the set of problems with deviation lower than $\varepsilon_1 < \varepsilon$~; that is, we want the algorithm to be faster on more linear models.
This is not achievable. The lower bound states that it is \emph{not} possible for an algorithm to have lower sample complexity on $\mathcal M_1$ while being $\delta$-correct on $\mathcal M$. On every $\mu \in \mathcal M$, the lower bound is $B(\mu, \delta, \mathcal M)$.

An algorithm cannot adapt to the deviation to linearity: it has to use a parameter $\varepsilon$ set in advance, and its sample complexity will depend on that $\varepsilon$, not on the actual deviation of the problem.
Note that this observation does not contradict recent results for regret minimization~\citep[e.g.,][]{lattimore2020learning,takemura2021parameter}, which show that adapting to an unknown scale of misspecification is possible.
In fact, such results involve a ``weak'' form of adaptivity, where the algorithms provably leverage the linear structure at the price of suffering an additive \emph{linear} regret term of order $\mathcal{O}(\varepsilon \sqrt{d}T)$, where $T$ is the learning horizon.
Since the counterpart of $\delta$-correctness for regret minimization is ``the algorithm suffers sub-linear regret in $T$ for all instances of the given family'', this implies that algorithms with such ``weak'' adaptivity loose this important property of consistency.


\section{The \algo algorithm}\label{sec:the_algorithm}

We introduce \algo (Misspecified Linear Identification), an algorithm to tackle misspecification in linear bandit models for fixed-confidence Top-$m$ identification. We describe the algorithm in Section \ref{sub:algorithm}, while in Section \ref{sub:sample_complexity} we report its sample complexity analysis.


\subsection{Algorithm}\label{sub:algorithm}

The pseudocode of \algo is outlined in Algorithm \ref{alg:mislingame-topm}. On the one hand, the design of \algo builds on top of recent approaches for constructing pure exploration algorithms from lower bounds \citep{degenne2019non,degenne2020gamification,zaki2020explicit,jedra2020optimal}.
On the other hand, its main components and their analysis introduce several technical novelties to deal with misspecified Top-$m$ identification, that might be of independent interest for other settings. We describe these components below. 
\correction{Let us define $D_v := \mathrm{diag}(v^1, v^2, \dots, v^K)$ for any vector $v \in \mathbb{R}^K$, and $V_t := \sum_{s=1}^t \phi_{k_s}\phi_{k_s}^\top$.}

\begin{algorithm}[t]
\caption{\algo}\label{alg:mislingame-topm}
 	\begin{algorithmic}
 		\REQUIRE Set of models $\mathcal{M}$, online learner $\mathcal{L}$, stopping thresholds $\{\beta_{t,\delta}\}_{t\geq 1}$
 		\vspace{0.1cm}
 		\STATE{Compute a sequence of arms $k_1,\dots,k_{t_0}$ such that $\sum_{t=1}^{t_0}\phi_{k_t}\phi_{k_t}^T \succeq \mineig I_d$ \hfill{\color{gray!60!black}\small\textsc{// Initialization}}}
 		 \FOR{$t = 1,\dots,t_0$}
		\STATE{Pull $k_t$, receive $X^{k_t}_t$, and set $\omega_t \leftarrow e_{k_t}$ \hfill{\color{gray!60!black}\small\textsc{// Pull spanner}}}
 		\ENDFOR
 		\STATE{Compute empirical mean $\widehat{\mu}_{t_0}$ and its projection $\tilde{\mu}_{t_0} \leftarrow \argmin_{\lambda\in\mathcal{M}}\|\lambda - \widehat{\mu}_{t_0}\|_{D_{N_{t_0}}}^2$}
 		\FOR{$t=t_0+1,t_0+2,\dots,$}
 		 \STATE{\textbf{if} $\inf_{\lambda\in\Lambda_m(\tilde{\mu}_{t-1})}\|\tilde{\mu}_{t-1} - \lambda \|_{D_{N_{t-1}}}^2 > 2\beta_{t-1,\delta}$ \textbf{then} \hfill{\color{gray!60!black}\small\textsc{// Stopping rule}}}
 		\STATE{\quad Stop and return $\mathcal{S}_m^\star(\tilde{\mu}_{t-1})$}
 		\STATE{\textbf{end if}}
 		\STATE{Obtain $\omega_t$ from $\mathcal{L}$}
 		\STATE{Compute closest alternative: $\lambda_t \leftarrow \argmin_{\lambda\in\Lambda_m(\tilde{\mu}_{t-1})} \|\tilde{\mu}_{t-1} - \lambda \|_{D_{\omega_t}}^2$}
		\STATE{Update $\mathcal{L}$ with gain \correction{$g_t : \omega \mapsto \sum_{k\in[K]}\omega^k \left(|\tilde{\mu}_{t-1}^k - \lambda_t^k| + \sqrt{c_{t-1}^k}\right)^2$}}
		\hfill{\color{gray!60!black}\small\textsc{// Update learner}}
		\STATE{Pull $k_t \thicksim \omega_t$ and receive reward $X^{k_t}_t$ \hfill{\color{gray!60!black}\small\textsc{// Action sampling}}}
 		\STATE{Update $\widehat{\mu}_{t}$ and compute projection $\tilde{\mu}_{t} \leftarrow \argmin_{\lambda\in\mathcal{M}}\|\lambda - \widehat{\mu}_{t}\|_{D_{N_{t}}}^2$ \hfill{\color{gray!60!black}\small\textsc{// Estimation}}}
 		\ENDFOR
 	\end{algorithmic}
\end{algorithm}

\textbf{Initialization phase.} 
\algo starts by pulling a deterministic sequence of $t_0$ arms that make the minimum eigenvalue of the resulting design matrix $V_{t_0}$ larger than $\mineig$.
Since the rows of $A$ span $\mathbb{R}^d$, such sequence can be easily found by taking any subset of $d$ arms that span the whole space (e.g., by computing a barycentric spanner~\cite{awerbuch2008online}) and pulling them in a round robin fashion until the desired condition is met. This is required to make the design matrix invertible.
While the literature typically avoid this step by regularizing (\eg ~\cite{abbasi2011improved}), in our misspecified setting it is crucial not to do so to obtain tight concentration results for the estimator of $\mu$, as explained in the next paragraph. \correction{See Appendix~\ref{subapp:initialization} for a discussion of the length \correction{$t_0$} of that initialization phase.}

\textbf{Estimation.} 
At each time step $t \geq t_0$, \algo maintains an estimator $\tilde{\mu}_{t}$ of the true bandit model $\mu$.
This is obtained by first computing the empirical mean $\widehat{\mu}_{t}$, such that $\widehat{\mu}_{t}^k = \frac{1}{N_t^k}\sum_{s=1}^t \indi{k_s=k} X_s^{k_s}$, and then projecting it onto the family of realizable models $\mathcal{M}$ according to the $D_{N_t}$-weighted norm, i.e., $\tilde{\mu}_{t} := \argmin_{\lambda\in\mathcal{M}}\|\lambda - \widehat{\mu}_{t}\|_{D_{N_{t}}}^2$.
Since each $\lambda\in\mathcal{M}$ can be decomposed into $\lambda = A\theta' + \eta'$ for some $\theta'\in\mathbb{R}^d$ and $\eta'\in\mathbb{R}^K$, this can be solved efficiently as the minimization of a quadratic objective in $K+d$ variables subject to the linear constraints $\|\eta'\|_\infty\leq \varepsilon$ and $\|A\theta'+\eta'\|_\infty \leq M$. The second constraint is only required for the analysis, while it often has a negligible impact in practice.
Thus, we shall drop it in our implementation, which yields two independent optimization problems for the projection $\tilde{\mu}_{t} = A\tilde{\theta}_t + \tilde{\eta}_t$: one for $\tilde{\theta}_t$, whose solution is available in closed form as the standard least-squares estimator $\tilde{\theta}_t = \hat{\theta}_t := V_t^{-1}\sum_{s=1}^t X_s^{k_s} \phi_{k_s}$, and one for $\tilde{\eta}_t$, which is another quadratic program with $K$ variables (see Appendix~\ref{app:algo_algorithm}). 

A crucial component in the concentration of these estimators, and a key novelty of our work, is the adoption of an orthogonal parametrization of mean vectors.
In particular, we leverage the following observation: any mean vector $\mu = A\theta + \eta$ can be equivalently represented, at any time $t$, as $\mu = A\theta_t + \eta_t$, where $\theta_t = V_t^{-1}\sum_{s=1}^t \mu^{k_s} \phi_{k_s}$ is the orthogonal projection (according to the design matrix $V_t$) of $\mu$ onto the feature space and $\eta_t = \mu - A\theta_t$ is the residual.
Then, it is possible to show that $\|\hat{\theta}_t - \theta_t\|_{V_t}^2$ is \emph{exactly} the self-normalized martingale considered in \citep{abbasi2011improved} and, thus, it enjoys the \emph{same} bound we have in linear bandits with no misspecification \correction{(refer to Appendix~\ref{app:orthogonal})}.
This is an important advantage over prior works \citep{lattimore2020learning,zanette2020learning} that, in order to concentrate $\hat{\theta}_t$ to $\theta$, need to inflate the concentration rate by a factor $\varepsilon^2 t$, which often makes the bound too large to be practical \correction{for misspecified models with $\varepsilon \gg 0$}. 
\correction{It allows us to also avoid superlinear terms of the form $\epsilon^2t\log(t)$ which are present in related works and which would prevent us from deriving good problem-dependent guarantees.}

\textbf{Stopping rule.} 
\algo uses the standard stopping rule adopted in most existing algorithms for pure exploration~\citep{garivier2016optimal,degenne2019non,shang2020fixed}.
What makes it peculiar is the definition of the thresholds $\beta_{t,\delta}$.
\algo requires a careful combination of concentration inequalities for (1) linear bandits, to make the algorithm adapt well to linear models with low $\varepsilon$, and (2) unstructured bandits, to guarantee asymptotic optimality.
The precise definition of $\beta_{t,\delta}$ is shown in the following result.
\begin{lemma}[\algo is $\delta$-correct]\label{lem:delta-correct}
Let $W_{-1}$ be the negative branch of the Lambert W function and let $\overline{W}(x) = - W_{-1}(-e^{-x}) \approx x + \log x$. For $\delta\in(0,1)$, define
\begin{align}
\beta_{t,\delta}^{\mathrm{uns}}
&:= 2 K \overline{W}\left( \frac{1}{2K}\log\frac{2e}{\delta} + \frac{1}{2}\log(8eK\log t)\right),
\\
\beta_{t,\delta}^{\mathrm{lin}}
&:= \frac{1}{2}\left( 4\sqrt{t}\varepsilon
  + \sqrt{2}\sqrt{1 + \log\frac{1}{\delta} + \left(1+\frac{1}{\log(1/\delta)} \right) \frac{d}{2}\log\!\left(1+\frac{t}{2d}\log\frac{1}{\delta} \right)} \right)^2
\: .
\end{align}
Then, for the choice $\beta_{t,\delta} := \min\{\beta_{t,\delta}^{\mathrm{uns}}, \beta_{t,\delta}^{\mathrm{lin}}\}$, \algo is $\delta$-correct.
\end{lemma}
This result is a simple consequence of two (linear and unstructured) concentration inequalities. See Appendix~\ref{app:sample_complexity_upper_bound}.

\textbf{Sampling strategy and online learners.} 
The sampling strategy of \algo aims at achieving the optimal sample complexity from the lower bound in Theorem~\ref{th:lower_bound}. As popularized by recent works~\citep{degenne2019non,degenne2020gamification,zaki2020explicit}, instead of relying on inefficient max-min oracles to repeatedly solve the optimization problem of Theorem~\ref{th:lower_bound} \citep{fiez2019sequential,jedra2020optimal}, we compute it incrementally by employing no-regret online learners. At each step $t$, the learner $\mathcal{L}$ plays a distribution over arms $\omega_t \in \Delta_K$ and it is updated with a gain function \correction{$g_t$} whose precise definition will be specified shortly. Then, \algo directly samples the next arm to pull from the distribution $\omega_t$, instead of using tracking as in the majority of previously mentioned works.
Similarly to what was recently shown by~\cite{tirinzoni2020asymptotically} for regret minimization in linear bandits, sampling will be crucial in our analysis to reduce dependencies on $K$ and, in particular, to obtain only logarithmic dependencies in the realizable linear case.

Regarding the choice of $\mathcal{L}$, two important properties are worth mentioning. First, \algo requires only a \emph{single} learner, while existing asymptotically optimal algorithms for pure exploration~\citep{degenne2019non,degenne2020gamification} need to allocate one learner for each possible answer.
Since the number of answers is $\left(^K_m\right)$, a direct extension of these algorithms to the Top-m setting would yield an impractical method with exponential (in $K$) number of learners, hence space complexity, and possibly sample complexity.\footnote{The fact that the optimization problem of the lower bound decomposes into $m(K-m)$ minimizations does not reduce the number of possible answers, which is still combinatorial in $K$.} 
Second, the choice of $\mathcal{L}$ is highly flexible since any learner that satisfies the following property suffices.
\begin{definition}[No-regret learner]\label{def:no-regret}
	A learner $\mathcal{L}$ over $\Delta_K$ is said to be no-regret if, for any $t\geq 1$ and any sequence of gains $\{g_s(\omega)\}_{s\leq t}$ bounded in absolute value by $B \in \mathbb{R}^{+}$, there exists a positive constant $C_{\mathcal{L}}(K,B)$ such that
$
\max_{w\in\Delta_K}\sum_{s=1}^t \big( g_s(w) -g_s(w_s) \big) \leq C_{\mathcal{L}}(K,B)\sqrt{t} \: .
$
\end{definition}
Examples of algorithms in this class are Exponential Weights \citep{cesa2006prediction} and AdaHedge~\citep{erven2011adaptive}. The latter shall be our choice for the implementation since it does not use $B$ as a parameter but adapts to it, and thus does not suffer from a possibly loose bound on $B$.

\textbf{Optimistic gains.}
Finally, we need to specify how the gains \correction{$g_t$} are computed. Clearly, if $\mu$ were known, one would directly use \correction{$g_t : \omega \mapsto \inf_{\lambda \in \Lambda_m(\mu)}\Vert \mu - \lambda \Vert_{D_\omega}^2$}. Since $\mu$ is unknown and must be estimated, we set $g_t(\omega)$ to an \emph{optimistic} proxy for that quantity.
In particular, we choose a sequence of bonuses $\{c_t^k\}_{t\geq t_0,k\in[K]}$ such that, with high probability, $g_t(\omega_t) := \sum_{k\in[K]}\omega_t^k \left(|\tilde{\mu}_{t-1}^k - \lambda_t^k| + \sqrt{c_{t-1}^k}\right)^2 \geq \inf_{\Lambda \in \lambda_m(\mu)}\Vert \mu - \lambda \Vert_{D_{\omega_t}}^2$, for $\lambda_t := \argmin_{\lambda\in\Lambda_m(\tilde{\mu}_{t-1})} \|\tilde{\mu}_{t-1} - \lambda \|_{D_{\omega_t}}^2$.
As for the stopping thresholds, we construct $c_{t}^k$ by a careful combination of structured and unstructured concentration bounds:
\begin{align*}
c_t^k := \min\left\{8 \etaboundsq + 4 \alpha_{\correction{t^2}}^{\mathrm{lin}}  \Vert \phi_k \Vert_{V_{t}^{-1}}^2, \frac{2\alpha_{\correction{t^2}}^{\mathrm{uns}}}{N_t^k}, 4M^2 \right\},
\end{align*}
where $\alpha_t^{\mathrm{uns}} := \beta_{t,1/(\numel t^3)}^{\mathrm{uns}}$ and $\alpha_t^{\mathrm{lin}} := \log(\numel t^2) + d\log\!\left(1+t/(2d)\right)$.
We show in Appendix~\ref{app:sample_complexity_upper_bound} that this choice of $c_t^k$ suffices to guarantee optimism with high probability.

\subsection{Sample complexity}\label{sub:sample_complexity}

\begin{theorem}\label{th:sample_complexity}
	\algo has expected sample complexity $\mathbb{E}_\mu[\tau_\delta] \le T_0(\delta) + 2$, where $T_0(\delta)$ is \correction{the} solution to the equation in $t$
\begin{align}
\beta_{t,\delta}
\ge  t H_\mu + \widehat{\mathcal O} \left(
  \min\{t \correction{K^2 \varepsilon^2} {+} d\sqrt{t} \ell_t, \sqrt{K t} \ell_t\};
  \log K \sqrt{t};
  \sqrt{\min\{t \correction{K^2 \varepsilon^2} {+} d \ell_t, K \ell_t\}\log(1/\delta)}\right) \: , \label{eq:sample_complexity}
\end{align}
where $\ell_t := \log t$, $H_\mu$ is the inverse complexity appearing in the lower bound (see Equation~\ref{eq:lower_bound}), and $\widehat{\mathcal O}(a ; b ; c)$ represent a sum of terms, each of which is $\mathcal O$ of one of the expressions shown.
\end{theorem}
See Appendix~\ref{app:sample_complexity_upper_bound} for the proof. Since $\beta_{t,\delta}^{\mathrm{uns}} \approx \log(1/\delta)$ for small $\delta$, \correction{$T_0(\delta) = H_{\mu}^{-1}\log(1/\delta) + C_{\mu}o(\log(1/\delta))$, where $C_{\mu}$ is a problem-dependent constant. Then $\liminf_{\delta \to 0} \mathbb{E}_{\mu}[\tau_{\delta}]/\log(1/\delta) = $} $\liminf_{\delta\to 0} T_0(\delta)/\log(1/\delta) = H_\mu^{-1}$ and thus the upper bound matches the lower bound in that limit\correction{: \algo is asymptotically optimal}.
The only polynomial factors in $K$ are in a minimum with a term that depends on $\varepsilon$\correction{. In} 
the linear setting, when $\varepsilon = 0$, we have only logarithmic \correction{(and no polynomial)} dependence on the number of arms, which is on par with the state of the art \citep{tirinzoni2020asymptotically,jedra2020optimal,kirschner2020asymptotically}.
Moreover, the bound exhibits an adaptation to the value of $\varepsilon$. If $\varepsilon$ 
is small, then the minimums in $\beta_{t,\delta}$ and in the inequality~\eqref{eq:sample_complexity} are equal to the ``linear'' values which involve $K\varepsilon$ and $d$ instead of $K$. \correction{As $\varepsilon$ grows, the upper bound transitions to terms matching the optimal unstructured bound.}

\textbf{Decoupling the stopping and sampling analyses.} Our analysis decomposes into two parts: first, a result on the stopping rule, then, a discussion of the sampling rule. The algorithm is shown to verify that, under a favorable event, if it does not stop at time $t$,
\begin{align*}
2 \beta_{t,\delta} 
	\ge \inf_{\lambda \in \Lambda_m(\mu)} \norm{\mu - \lambda}^2_{D_{N_t}} - \mathcal O(\sqrt{t})
\ge 2 t H_\mu - \mathcal O(\sqrt{t})
\: .
\end{align*}
The sample complexity result is a consequence of that bound on $t$.
The first inequality is due solely to the stopping rule, and the second one only to the sampling mechanism. The expression $\inf_{\lambda \in \Lambda_m(\mu)} \norm{\mu - \lambda}^2_{D_{N_t}}$ does not feature any variable specific to the algorithm: we can combine any stopping rule and any sampling rule, as long as they each verify the corresponding inequality.

\textbf{A more aggressive optimism.} The optimistic gains that we have chosen, $g_t(\omega) = \sum_{k\in[K]}\omega^k (|\tilde{\mu}_{t-1}^k - \lambda_t^k| + \sqrt{c_{t-1}^k})^2$, are tuned to ensure asymptotic optimality (with a factor 1 in the leading term). If we instead accept to be asymptotically optimal up to a factor 2, we can use the gains
$
g_t(\omega) = \sum_{k\in[K]}\omega^k \left((\tilde{\mu}_{t-1}^k - \lambda_t^k)^2 + c_{t-1}^k\right)
$ .
When using those, the learner takes decisions which are much closer to those it would take if using the empirical gains $\sum_{k\in[K]}\omega^k (\tilde{\mu}_{t-1}^k - \lambda_t^k)^2$ and the theoretical bound, while worse in the leading factor, has better lower order terms. The aggressive optimism sometimes has significantly better practical performance (see Experiment (C) in Figure~\ref{fig:experimentABC}).




\section{Experimental evaluation}\label{sec:experimental_evaluation}


Since our algorithm is the first to apply to Top-$m$ identification in misspecified linear models, we compare it against an efficient linear algorithm, LinGapE~\citep{xu2018fully} (that is, its extension to Top-$m$ as described in~\cite{reda2021top}, which coincides with LinGapE for $m=1$), and an unstructured one, LUCB~\citep{kalyanakrishnan2012pac}. 
In all experiments, we consider $\delta=5\%$.~\footnote{\correction{All the code and scripts are available at \href{https://github.com/clreda/misspecified-top-m}{\texttt{https://github.com/clreda/misspecified-top-m}}.}}
For each algorithm, we show boxplots reporting the average sample complexity on the $y$-axis, and the error frequency $\hat{\delta}$ across $500$ (resp. $100$) repetitions for simulated (resp. real-life) instances 
rounding up to the $5^{th}$ decimal place. Individual outcomes are shown as gray dots.
\correction{It has frequently been noted in the fixed-confidence literature that stopping thresholds which guarantee $\delta$-correctness tend to be too conservative and to yield empirical error frequencies that are actually much lower than $\delta$. Moreover, these thresholds are different from linear to unstructured models. }In order to ensure a good trade-off between performance and computing speed, \correction{and fairness between tested algorithms,} we use a heuristic value for the stopping rule $\beta_{t,\delta} := \ln((1+\ln(t+1))/\delta)$ unless otherwise specified.
For each experiment, we report the number of arms ($K$), the dimension of features ($d$), the size of the answer ($m$), the misspecification ($\varepsilon$) and the gap between the $m^{th}$ and $(m+1)^{th}$ best arms ($\Delta := \max^m_{a \in [K]} \mu^{a}- \max^{m+1}_{b \in [K]} \mu^{b}$).
The computational resources used, data licenses and further experimental details can be found in Appendix~\ref{app:experiments}.

\textbf{(A) Simulated misspecified instances.} 
($K=10$, $d=5$, $m=3$, \correction{$\varepsilon \in \{0,5\}$}, $\Delta\approx 0.28$) First, we fix a linear instance
$\mu := A \theta$ by randomly sampling the values of $\theta \in \mathbb{R}^{d}$ and $A \in \mathbb{R}^{K \times d}$ from a zero-mean Gaussian distribution, and renormalizing them by their respective \correction{$\ell_{\infty}$} norm. Then, for \correction{$\varepsilon \in \{0,5\}$}, we build a misspecified linear instance $\mu_\varepsilon = A \theta + \eta_\varepsilon$, such that, if $(4)$ is the index of the fourth best arm, $\forall k \neq (4), \eta_\varepsilon^k = 0$, and $\eta_\varepsilon^{(4)} = \varepsilon$. 
Note that any value of $\varepsilon < \Delta$ does not switch the third and fourth arms in the set of best arms of $\mu_\varepsilon$, contrary to values greater than $\Delta$. The greater $\varepsilon$ is, the more different the answers from the linear and misspecified models are. This experiment was inspired by~\cite{ghosh2017misspecified}, where a similar model is used to prove a lower bound in the setting of regret minimization. 
See the leftmost two plots on Figure~\ref{fig:experimentABC}.
As expected, LUCB is always $\delta$-correct, but suffers from a significantly larger sample complexity than its structured counterparts.
\correction{Moreover, LinGapE does not preserve the $\delta$-correctness under large misspecification level $\varepsilon=5$ (with error rate $\hat{\delta}=0.96$), which illustrates the effect of $\varepsilon$ on the answer set.
Note that it is not due to the choice of stopping threshold, as running it with the theoretically-supported threshold derived in~\cite{abbasi2011improved} also yields an empirical error rate $\hat{\delta}=1$}. \algo proves to be competitive against LinGapE. Note that the case $\varepsilon=0$ is a perfectly linear instance. See Table~\ref{tab:experimentA} in Appendix for numerical results for algorithms LinGapE and \algo.

\begin{figure}
	\centering
	\includegraphics[scale=0.13]{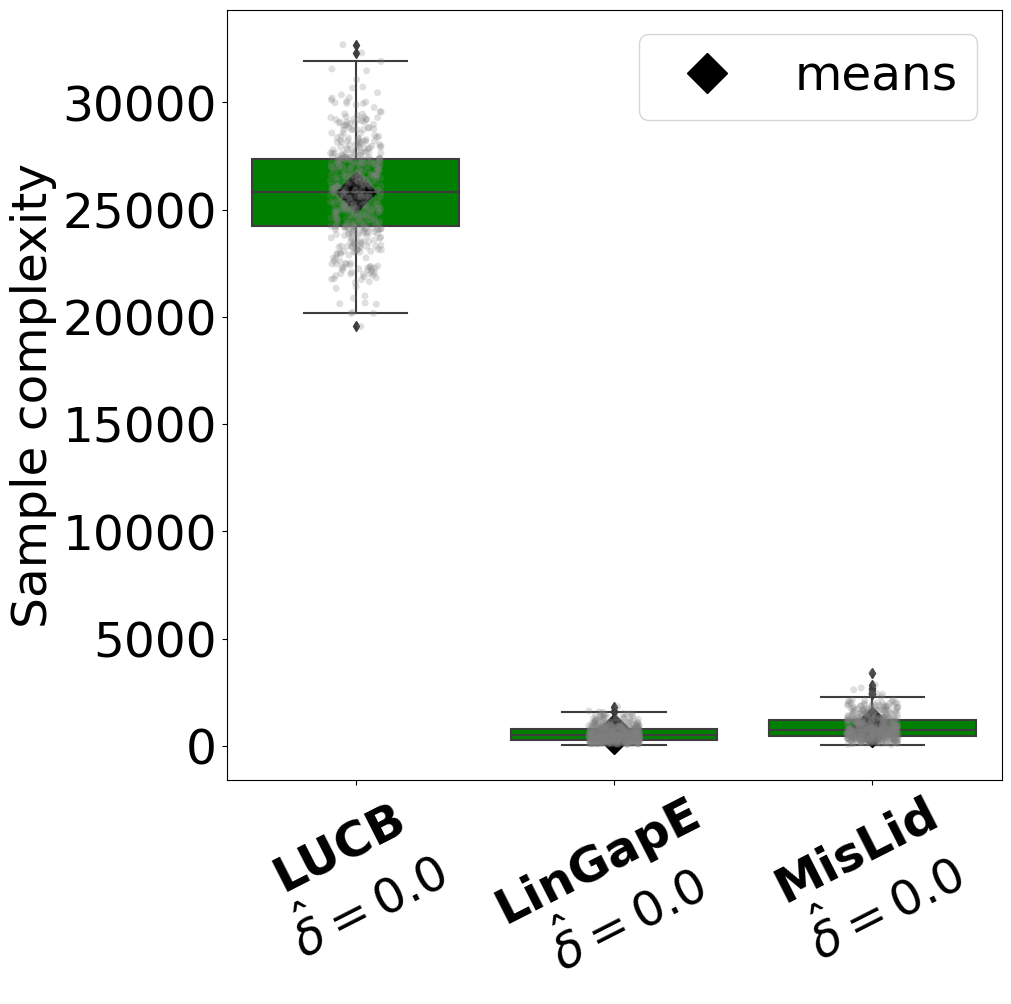}
	\includegraphics[scale=0.13]{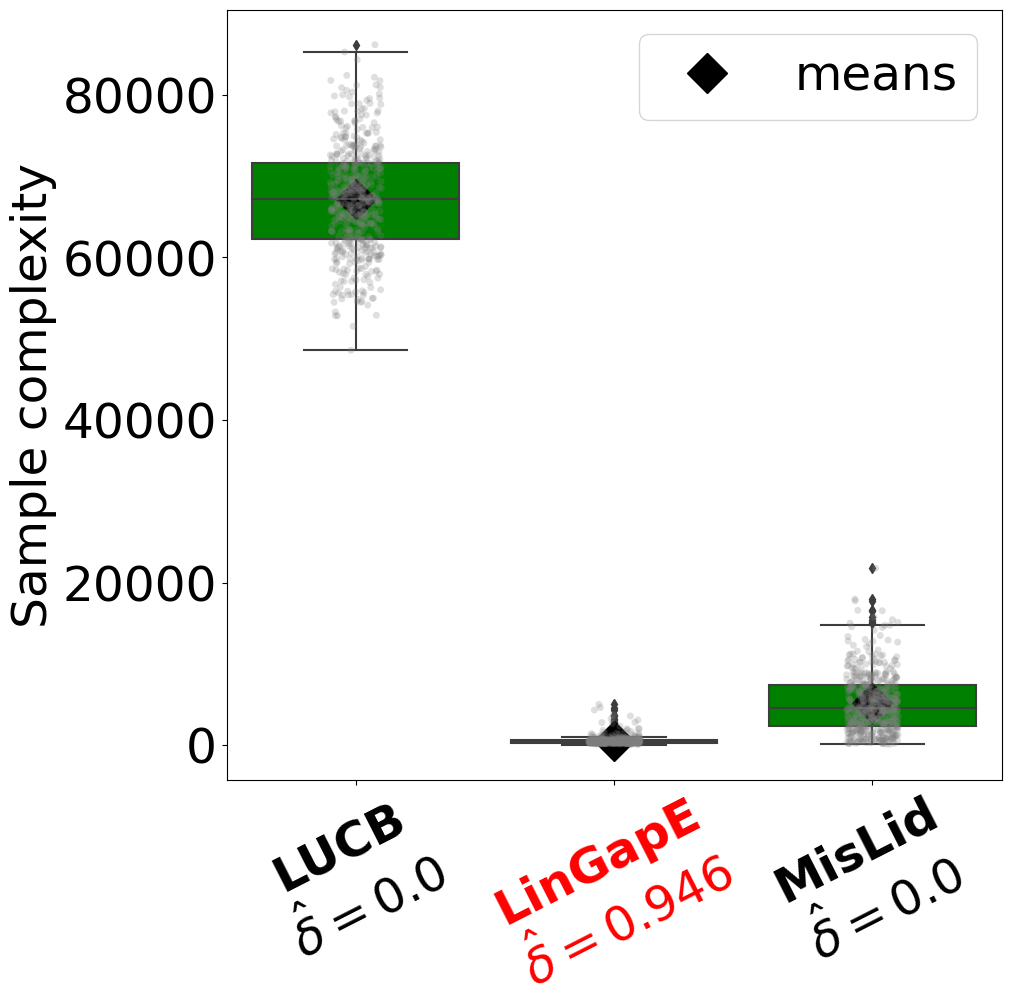}
	\includegraphics[scale=0.13]{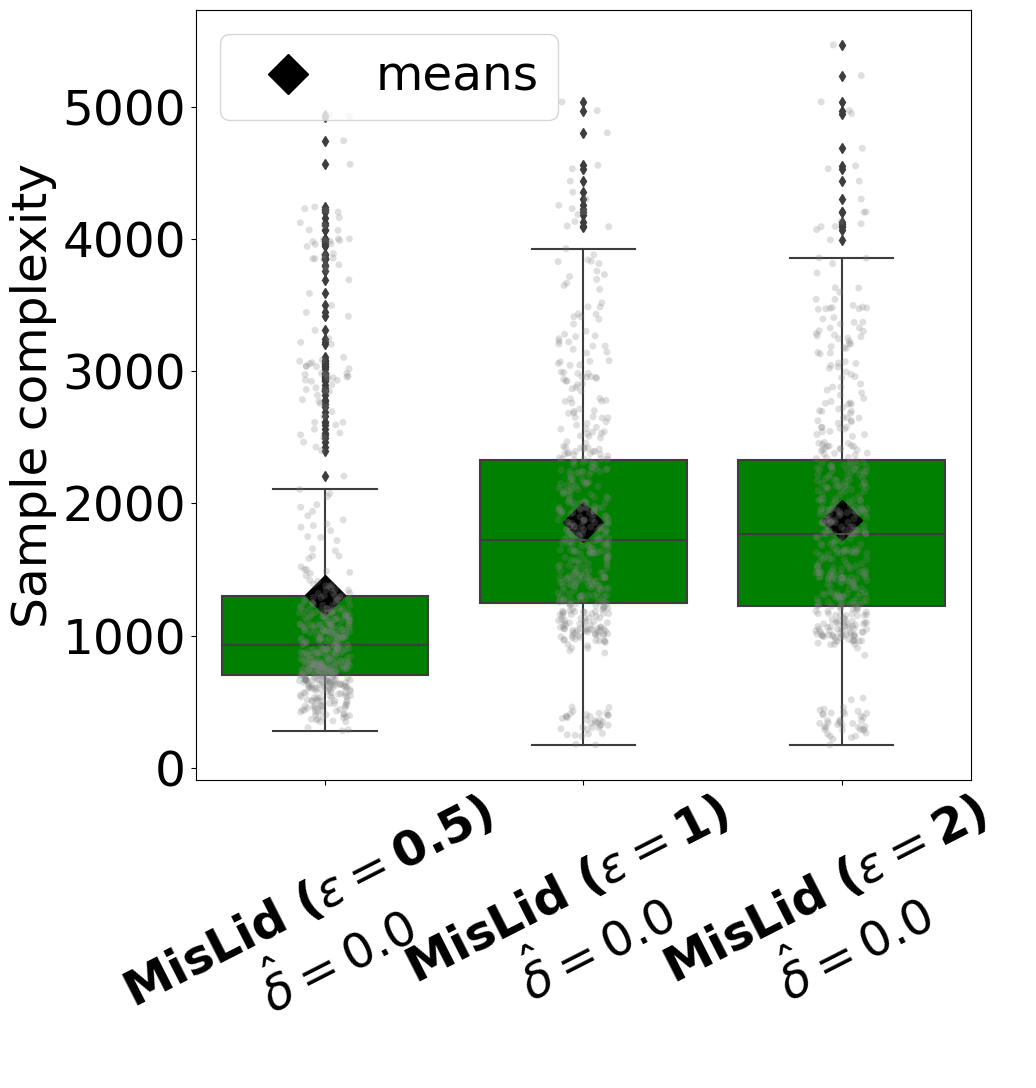}
  	\includegraphics[scale=0.12]{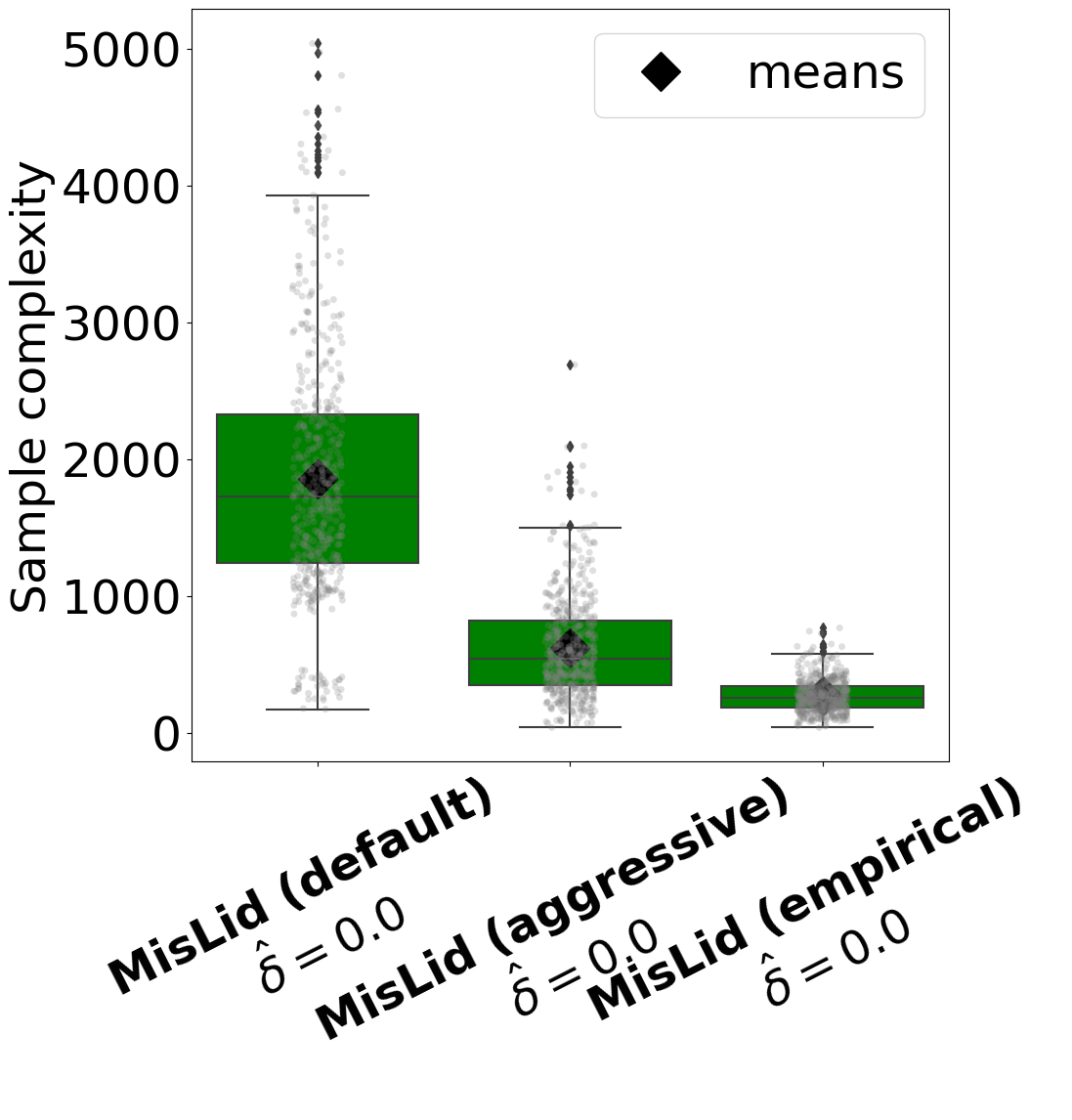}
	\caption{\correction{Experiment (A) for $\varepsilon \in \{0, 5\}$} (\textit{first two from the left}). Experiment (B) with $\varepsilon \in \{ 0.5, 1, 2\}$. 
	Experiment (C) to compare different optimistic gains (\textit{right}).}
	\label{fig:experimentABC}
\end{figure}


\textbf{(B) Discrepancy between user-selected $\varepsilon$ and true $\|\eta\|$.} ($K=15$, $d=8$, $m=3$, $\varepsilon \in \{0.5,1,2\}$, $\Delta \approx 0.4$) \algo crucially relies on a user-provided upper bound on the scale of deviation from linearity.
We test its robustness against perturbations to the input value $\varepsilon$ compared to the value $\varepsilon^\star := \|\eta\|_{\infty}$ in the misspecified model $\mu := A \theta + \eta$. Values are sampled randomly for $\theta,A,\eta$, and the associated vectors are normalized by their $\ell^{\infty}$ norm (for $\eta$, by $\|\eta\|_\infty/\varepsilon^\star$), where $\varepsilon^\star=1 > \Delta$ is the true deviation to linearity.
The results, shown in the third plot of Figure~\ref{fig:experimentABC}, display the behavior predicted by Lemma~\ref{lemma:match_unstructured}. Indeed, as the user-provided value $\varepsilon$ increases, the associated sample complexity increases as well. The plateau in sample complexity when $\varepsilon$ is large enough is noticeable. 
Cases $\varepsilon \in \{1, 2\}$ display a sample complexity close to that of unstructured bandits.

\textbf{(C) Comparing different optimisms.} ($K=15$, $d=8$, $m=3$, $\varepsilon=1$, $\Delta \approx 0.4$) We use the same bandit model as in Experiment (B), and use $\varepsilon = \varepsilon^\star = 1$. We compare the aggressive optimism described in Section~\ref{sub:sample_complexity}, no optimism (that is, $\forall k \in [K], \forall t > 0, c^k_t = 0$), and the default optimistic gains given in Section~\ref{sub:algorithm}. See the rightmost plot in Figure~\ref{fig:experimentABC}. The algorithm with no optimism is denoted ``empirical'', and is significantly faster than the optimistic variants. 

\textbf{(D) Application to drug repurposing.} ($K=10$, $d=5$, $m=5$, $\hat{\varepsilon}\approx 0.02$, $\Delta\approx 0.062$) 
We use the drug repurposing problem for epilepsy proposed by~\cite{reda2021top} to investigate the practicality of our method.
In order to speed up LUCB, we consider the PAC version of Top-$m$ identification, choosing as stopping threshold $0.06 \approx \Delta$, such that the algorithm stops earlier while returning the exact set of $m$ best arms.
Following~\cite[][Appendix F.$4$]{papini2021leveraging}, we extract a linear model from the data by fitting a neural network and taking the features learned in the last layer. We compute $\varepsilon$ as the $\ell^{\infty}$ norm of the difference between the predictions of this linear model and the average rewards from the data, which yields $\hat{\varepsilon}=0.02$.
Since the misspecification is way below the minimum gap, and the linear model thus accurately fits the data, the results (leftmost plot in Figure~\ref{fig:experimentDE}) show that MisLid and LinGapE perform comparably on this instance. Moreover, both are an order of magnitude better than an unstructured bandit algorithm sample complexity-wise. Please refer to Table~\ref{tab:experimentDE} in Appendix for numerical results for LinGapE and \algo.

\textbf{(E) Application to a larger instance of online recommendation.} 
($K=103$, $d=8$, $m=4$, $\hat{\epsilon}\approx 0.206$, $\Delta \approx 0.022$)
As in Experiment~(D), a linear representation is extracted for an instance of online recommendation of music artists to users (Last.fm dataset~\citep{cantador2011workshop}). 
We compute a proxy for $\varepsilon$ and feed the value $\Delta$ to the stopping threshold in LUCB in a similar fashion. Differently from Experiment~(D), this yields a misspecification that is much larger than the minimum gap.
To improve performance on these instances, we modified \algo. To reduce the sample complexity, we use empirical gains instead of optimism.
To reduce the computational complexity, we check the stopping rule infrequently (on a geometric grid) and use only a random subset of arms in each round 
to compute the sampling rule (see Appendix~\ref{app:experiments} for details and an empirical comparison to the theoretically supported \algo).
See the 
rightmost plot in Figure~\ref{fig:experimentDE}. 
\correction{This plot} particularly illustrates our introductory claim: an unstructured bandit algorithm is $\delta$-correct, but too slow in practice for misspecified instances, whereas the guarantee on correctness for a linear bandit does not hold anymore on these models with large misspecification. Numerical results for LinGapE and \algo are listed in Table~\ref{tab:experimentDE} in Appendix.

\begin{figure}
	\centering
  	\includegraphics[scale=0.13]{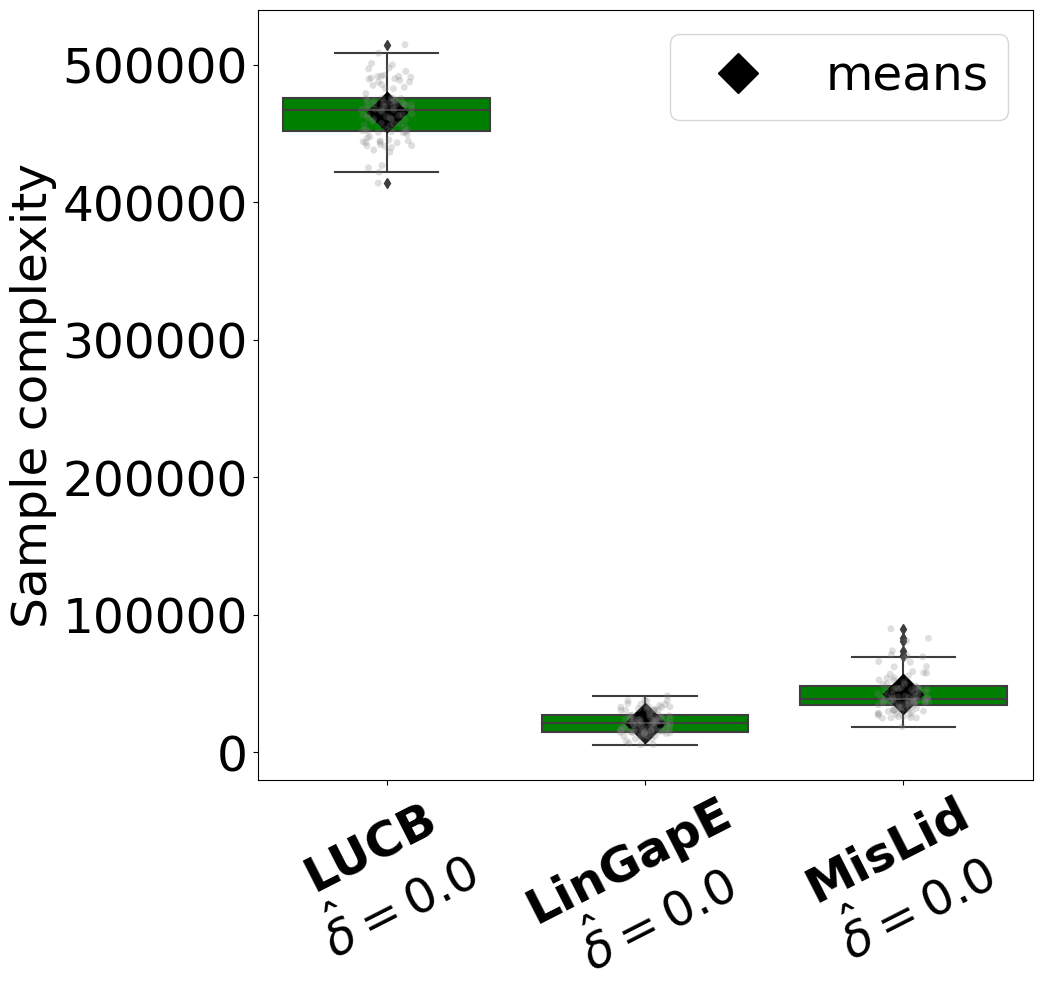}
	\hspace{1cm}
	\includegraphics[scale=0.13]{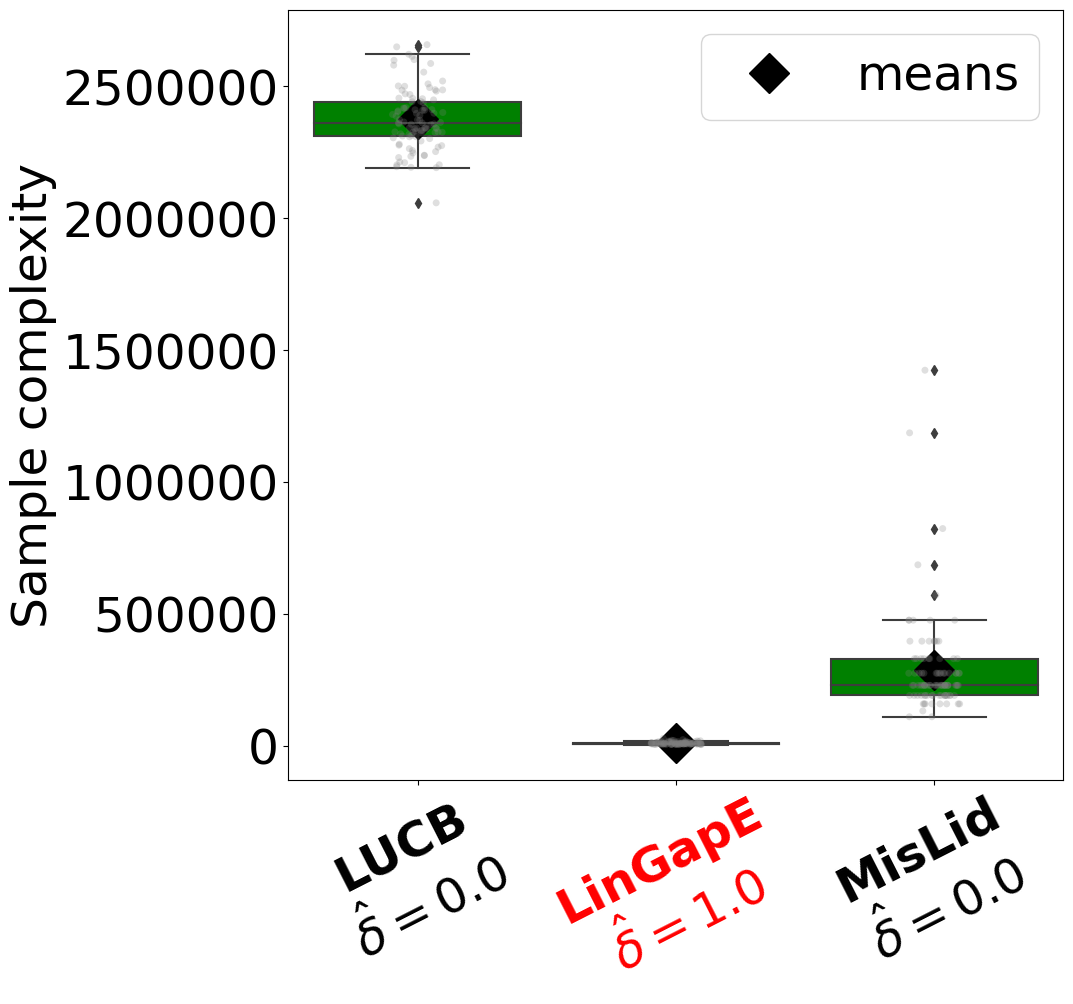}
	\caption{Experiment (D) for drug repurposing in epilepsy (\textit{left}). \correction{Experiment (E)}
	for online recommendation.} 
	\label{fig:experimentDE}
\end{figure}



\section{Discussion}\label{sec:discussion}

We have designed the first algorithm to tackle misspecification in fixed-confidence Top-$m$ identification, which has applications in online recommendation.
However, the algorithm relies exclusively on the features provided in the input data, and as such might be subjected to bias and lack of fairness in its recommendation, depending on the dataset.
The proposed algorithm can be applied to misspecified models which deviate from linearity (\ie $\varepsilon > 0$), encompassing unstructured settings (for large values of $\varepsilon$) and linear models (\ie $\varepsilon = 0$).

Our tests on variants of our algorithm suggest that the optimistic estimates have a big influence on the sample complexity. Removing the optimism completely and using the empirical gains leads to the best performance.
We conjecture that other components of the algorithm like the learner are conservative enough for the optimism to be superfluous.
The main limitation of our method is its computational complexity: at each \correction{round}, $\mathcal{O}(Km)$ convex optimization problems need to be solved 
for both the sampling and stopping rules, which can be expensive if the number of arms is large.
However, the ``interesting'' arms 
are much less numerous and we observed empirically that the sample complexity is not increased significantly if we consider only a few arms. 
In general, theoretically supported methods to replace the alternative set by computationally simpler approximations would greatly help in reducing the computational cost of our algorithm.


Since the sampling of our algorithm is designed to minimize a lower bound, we can expect it to suffer from the same shortcomings as that bound. It is known that the bound in question does not capture some lower order (in $1/\delta$) effects, in particular those due to the multiple-hypothesis nature of the test we perform, which can be very large for small times. Work to take these effects into account to design algorithms has started recently \citep{katz2020empirical,katz2020true,wagenmaker2021experimental} and we believe that it is an essential avenue for further improvements in misspecified linear identification.

\newpage

\begin{ack}
Cl\'{e}mence R\'{e}da was supported by the ``Digital health challenge'' Inserm-CNRS joint program, the French Ministry of Higher Education and Research [ENS.X19RDTME-SACLAY19-22], and the French National Research Agency [ANR-19-CE23-0026-04] (BOLD project).
\end{ack}

\bibliographystyle{apalike}
\bibliography{references.bib}

\begin{thebibliography}{}

\bibitem[Abbasi-Yadkori et~al., 2011]{abbasi2011improved}
Abbasi-Yadkori, Y., P{\'a}l, D., and Szepesv{\'a}ri, C. (2011).
\newblock Improved algorithms for linear stochastic bandits.
\newblock In {\em Advances in Neural Information Processing Systems}, pages
  2312--2320.

\bibitem[Audibert et~al., 2010]{audibert2010best}
Audibert, J.-Y., Bubeck, S., and Munos, R. (2010).
\newblock Best arm identification in multi-armed bandits.
\newblock In {\em COLT}, pages 41--53.

\bibitem[Auer, 2002]{auer2002using}
Auer, P. (2002).
\newblock Using confidence bounds for exploitation-exploration trade-offs.
\newblock {\em Journal of Machine Learning Research}, 3(Nov):397--422.

\bibitem[Awerbuch and Kleinberg, 2008]{awerbuch2008online}
Awerbuch, B. and Kleinberg, R. (2008).
\newblock Online linear optimization and adaptive routing.
\newblock {\em Journal of Computer and System Sciences}, 74(1):97--114.

\bibitem[Bubeck et~al., 2009]{bubeck2009pure}
Bubeck, S., Munos, R., and Stoltz, G. (2009).
\newblock Pure exploration in multi-armed bandits problems.
\newblock In {\em International conference on Algorithmic learning theory},
  pages 23--37. Springer.

\bibitem[Cantador et~al., 2011]{cantador2011workshop}
Cantador, I., Brusilovsky, P., and Kuflik, T. (2011).
\newblock Second workshop on information heterogeneity and fusion in
  recommender systems (hetrec2011).
\newblock pages 387--388.

\bibitem[Cesa-Bianchi and Lugosi, 2006]{cesa2006prediction}
Cesa-Bianchi, N. and Lugosi, G. (2006).
\newblock {\em Prediction, learning, and games}.
\newblock Cambridge university press.

\bibitem[Chatterji et~al., 2020]{chatterji2020osom}
Chatterji, N.~S., Muthukumar, V., and Bartlett, P.~L. (2020).
\newblock {OSOM:} {A} simultaneously optimal algorithm for multi-armed and
  linear contextual bandits.
\newblock In {\em {AISTATS}}, volume 108 of {\em Proceedings of Machine
  Learning Research}, pages 1844--1854. {PMLR}.

\bibitem[Chen et~al., 2017]{chen2017nearly}
Chen, L., Li, J., and Qiao, M. (2017).
\newblock Nearly instance optimal sample complexity bounds for top-k arm
  selection.
\newblock In {\em Artificial Intelligence and Statistics}, pages 101--110.
  PMLR.

\bibitem[De~Rooij et~al., 2014]{de2014follow}
De~Rooij, S., Van~Erven, T., Gr{\"u}nwald, P.~D., and Koolen, W.~M. (2014).
\newblock Follow the leader if you can, hedge if you must.
\newblock {\em The Journal of Machine Learning Research}, 15(1):1281--1316.

\bibitem[Degenne and Koolen, 2019]{degenne2019pure}
Degenne, R. and Koolen, W.~M. (2019).
\newblock Pure exploration with multiple correct answers.
\newblock {\em arXiv preprint arXiv:1902.03475}.

\bibitem[Degenne et~al., 2019]{degenne2019non}
Degenne, R., Koolen, W.~M., and M{\'{e}}nard, P. (2019).
\newblock Non-asymptotic pure exploration by solving games.
\newblock In {\em NeurIPS}, pages 14465--14474.

\bibitem[Degenne et~al., 2020a]{degenne2020gamification}
Degenne, R., M{\'e}nard, P., Shang, X., and Valko, M. (2020a).
\newblock Gamification of pure exploration for linear bandits.
\newblock In {\em International Conference on Machine Learning}, pages
  2432--2442. PMLR.

\bibitem[Degenne et~al., 2020b]{degenne2020structure}
Degenne, R., Shao, H., and Koolen, W. (2020b).
\newblock Structure adaptive algorithms for stochastic bandits.
\newblock In {\em International Conference on Machine Learning}, pages
  2443--2452. PMLR.

\bibitem[Erven et~al., 2011]{erven2011adaptive}
Erven, T., Koolen, W.~M., Rooij, S., and Gr{\"u}nwald, P. (2011).
\newblock Adaptive hedge.
\newblock {\em Advances in Neural Information Processing Systems},
  24:1656--1664.

\bibitem[Even-Dar et~al., 2003]{even2003action}
Even-Dar, E., Mannor, S., and Mansour, Y. (2003).
\newblock Action elimination and stopping conditions for reinforcement
  learning.
\newblock In {\em Proceedings of the 20th International Conference on Machine
  Learning (ICML-03)}, pages 162--169.

\bibitem[Fiez et~al., 2019]{fiez2019sequential}
Fiez, T., Jain, L., Jamieson, K.~G., and Ratliff, L. (2019).
\newblock Sequential experimental design for transductive linear bandits.
\newblock In {\em Advances in Neural Information Processing Systems},
  volume~32.

\bibitem[Foster et~al., 2020]{foster2020adapting}
Foster, D.~J., Gentile, C., Mohri, M., and Zimmert, J. (2020).
\newblock Adapting to misspecification in contextual bandits.
\newblock {\em Advances in Neural Information Processing Systems}, 33.

\bibitem[Garivier and Kaufmann, 2016]{garivier2016optimal}
Garivier, A. and Kaufmann, E. (2016).
\newblock Optimal best arm identification with fixed confidence.
\newblock In {\em Conference on Learning Theory}, pages 998--1027. PMLR.

\bibitem[Ghosh et~al., 2017]{ghosh2017misspecified}
Ghosh, A., Chowdhury, S.~R., and Gopalan, A. (2017).
\newblock Misspecified linear bandits.
\newblock In {\em Proceedings of the AAAI Conference on Artificial
  Intelligence}, volume~31.

\bibitem[Jedra and Proutiere, 2020]{jedra2020optimal}
Jedra, Y. and Proutiere, A. (2020).
\newblock Optimal best-arm identification in linear bandits.
\newblock {\em arXiv preprint arXiv:2006.16073}.

\bibitem[Kalyanakrishnan and Stone, 2010]{kalyanakrishnan2010efficient}
Kalyanakrishnan, S. and Stone, P. (2010).
\newblock Efficient selection of multiple bandit arms: Theory and practice.
\newblock In {\em ICML}.

\bibitem[Kalyanakrishnan et~al., 2012]{kalyanakrishnan2012pac}
Kalyanakrishnan, S., Tewari, A., Auer, P., and Stone, P. (2012).
\newblock Pac subset selection in stochastic multi-armed bandits.
\newblock In {\em ICML}, volume~12, pages 655--662.

\bibitem[Katz-Samuels et~al., 2020]{katz2020empirical}
Katz-Samuels, J., Jain, L., Karnin, Z., and Jamieson, K.~G. (2020).
\newblock An empirical process approach to the union bound: Practical
  algorithms for combinatorial and linear bandits.
\newblock In {\em Advances in Neural Information Processing Systems},
  volume~33, pages 10371--10382.

\bibitem[Katz-Samuels and Jamieson, 2020]{katz2020true}
Katz-Samuels, J. and Jamieson, K. (2020).
\newblock The true sample complexity of identifying good arms.
\newblock In {\em International Conference on Artificial Intelligence and
  Statistics}, pages 1781--1791. PMLR.

\bibitem[Kaufmann et~al., 2016]{kaufmann2016complexity}
Kaufmann, E., Capp{\'e}, O., and Garivier, A. (2016).
\newblock On the complexity of best-arm identification in multi-armed bandit
  models.
\newblock {\em The Journal of Machine Learning Research}, 17(1):1--42.

\bibitem[Kirschner et~al., 2020]{kirschner2020asymptotically}
Kirschner, J., Lattimore, T., Vernade, C., and Szepesv{\'a}ri, C. (2020).
\newblock Asymptotically optimal information-directed sampling.
\newblock {\em arXiv preprint arXiv:2011.05944}.

\bibitem[Lattimore and Szepesv{\'a}ri, 2020]{lattimore2020bandit}
Lattimore, T. and Szepesv{\'a}ri, C. (2020).
\newblock {\em Bandit algorithms}.
\newblock Cambridge University Press.

\bibitem[Lattimore et~al., 2020]{lattimore2020learning}
Lattimore, T., Szepesvari, C., and Weisz, G. (2020).
\newblock Learning with good feature representations in bandits and in rl with
  a generative model.
\newblock In {\em International Conference on Machine Learning}, pages
  5662--5670. PMLR.

\bibitem[Magureanu et~al., 2014]{magureanu2014lipschitz}
Magureanu, S., Combes, R., and Proutiere, A. (2014).
\newblock Lipschitz bandits: Regret lower bound and optimal algorithms.
\newblock In {\em Conference on Learning Theory}, pages 975--999. PMLR.

\bibitem[Mason et~al., 2020]{mason2020finding}
Mason, B., Jain, L., Tripathy, A., and Nowak, R. (2020).
\newblock Finding all $\varepsilon$-good arms in stochastic bandits.
\newblock {\em Advances in Neural Information Processing Systems}, 33.

\bibitem[Orabona, 2019]{orabona2019modern}
Orabona, F. (2019).
\newblock A modern introduction to online learning.
\newblock {\em arXiv preprint arXiv:1912.13213}.

\bibitem[Pacchiano et~al., 2020]{pacchiano2020model}
Pacchiano, A., Phan, M., Abbasi~Yadkori, Y., Rao, A., Zimmert, J., Lattimore,
  T., and Szepesvari, C. (2020).
\newblock Model selection in contextual stochastic bandit problems.
\newblock In {\em Advances in Neural Information Processing Systems},
  volume~33, pages 10328--10337.

\bibitem[Papini et~al., 2021]{papini2021leveraging}
Papini, M., Tirinzoni, A., Restelli, M., Lazaric, A., and Pirotta, M. (2021).
\newblock Leveraging good representations in linear contextual bandits.
\newblock In Meila, M. and Zhang, T., editors, {\em Proceedings of the 38th
  International Conference on Machine Learning}, volume 139 of {\em Proceedings
  of Machine Learning Research}, pages 8371--8380. PMLR.

\bibitem[R{\'e}da et~al., 2021]{reda2021top}
R{\'e}da, C., Kaufmann, E., and Delahaye-Duriez, A. (2021).
\newblock Top-m identification for linear bandits.
\newblock In {\em International Conference on Artificial Intelligence and
  Statistics}, pages 1108--1116. PMLR.

\bibitem[Shang et~al., 2020]{shang2020fixed}
Shang, X., Heide, R., Menard, P., Kaufmann, E., and Valko, M. (2020).
\newblock Fixed-confidence guarantees for bayesian best-arm identification.
\newblock In {\em International Conference on Artificial Intelligence and
  Statistics}, pages 1823--1832. PMLR.

\bibitem[Simchowitz et~al., 2017]{simchowitz2017simulator}
Simchowitz, M., Jamieson, K., and Recht, B. (2017).
\newblock The simulator: Understanding adaptive sampling in the
  moderate-confidence regime.
\newblock In {\em Conference on Learning Theory}, pages 1794--1834. PMLR.

\bibitem[Soare et~al., 2014]{soare2014best}
Soare, M., Lazaric, A., and Munos, R. (2014).
\newblock Best-arm identification in linear bandits.
\newblock In {\em Advances in Neural Information Processing Systems},
  volume~27.

\bibitem[Takemura et~al., 2021]{takemura2021parameter}
Takemura, K., Ito, S., Hatano, D., Sumita, H., Fukunaga, T., Kakimura, N., and
  Kawarabayashi, K.-i. (2021).
\newblock A parameter-free algorithm for misspecified linear contextual
  bandits.
\newblock In {\em International Conference on Artificial Intelligence and
  Statistics}, pages 3367--3375. PMLR.

\bibitem[Tirinzoni et~al., 2020]{tirinzoni2020asymptotically}
Tirinzoni, A., Pirotta, M., Restelli, M., and Lazaric, A. (2020).
\newblock An asymptotically optimal primal-dual incremental algorithm for
  contextual linear bandits.
\newblock {\em Advances in Neural Information Processing Systems}, 33.

\bibitem[Wagenmaker et~al., 2021]{wagenmaker2021experimental}
Wagenmaker, A., Katz-Samuels, J., and Jamieson, K. (2021).
\newblock Experimental design for regret minimization in linear bandits.
\newblock In {\em International Conference on Artificial Intelligence and
  Statistics}, pages 3088--3096. PMLR.

\bibitem[Xu et~al., 2018]{xu2018fully}
Xu, L., Honda, J., and Sugiyama, M. (2018).
\newblock A fully adaptive algorithm for pure exploration in linear bandits.
\newblock In {\em International Conference on Artificial Intelligence and
  Statistics}, pages 843--851. PMLR.

\bibitem[Zaki et~al., 2020]{zaki2020explicit}
Zaki, M., Mohan, A., and Gopalan, A. (2020).
\newblock Explicit best arm identification in linear bandits using no-regret
  learners.
\newblock {\em arXiv preprint arXiv:2006.07562}.

\bibitem[Zanette et~al., 2020]{zanette2020learning}
Zanette, A., Lazaric, A., Kochenderfer, M.~J., and Brunskill, E. (2020).
\newblock Learning near optimal policies with low inherent bellman error.
\newblock In {\em {ICML}}, volume 119 of {\em Proceedings of Machine Learning
  Research}, pages 10978--10989. {PMLR}.

\end{thebibliography}

\newpage

\section*{Checklist}

\begin{enumerate}

\item For all authors...
\begin{enumerate}
  \item Do the main claims made in the abstract and introduction accurately reflect the paper's contributions and scope?
	  \answerYes{See Section~\ref{sec:lower_bound_top_m} (lower bound for the Top-$m$ identification problem), Section~\ref{sec:adaptation_is_impossible} (adaptivity to the misspecification), Section~\ref{sec:the_algorithm} (algorithm) and Section~\ref{sec:experimental_evaluation} (experiments).}
  \item Did you describe the limitations of your work?
	  \answerYes{See paragraph $2$ in Section~\ref{sec:discussion}.}
  \item Did you discuss any potential negative societal impacts of your work?
	  \answerYes{See paragraph $1$ in Section~\ref{sec:discussion}.}
  \item Have you read the ethics review guidelines and ensured that your paper conforms to them?
    \answerYes{}
\end{enumerate}

\item If you are including theoretical results...
\begin{enumerate}
  \item Did you state the full set of assumptions of all theoretical results?
	  \answerYes{See Section~\ref{sec:setting}.}
	\item Did you include complete proofs of all theoretical results?
    \answerYes{See Appendices.}
\end{enumerate}

\item If you ran experiments...
\begin{enumerate}
  \item Did you include the code, data, and instructions needed to reproduce the main experimental results (either in the supplemental material or as a URL)?
	  \answerYes{Refer to the following GitHub repository: \href{https://github.com/clreda/misspecified-top-m}{\texttt{https://github.com/clreda/misspecified-top-m}}.}
  \item Did you specify all the training details (e.g., data splits, hyperparameters, how they were chosen)?
	  \answerYes{See introductory and experiment-specific paragraphs in Section~\ref{sec:experimental_evaluation}.}
	\item Did you report error bars (e.g., with respect to the random seed after running experiments multiple times)?
		\answerYes{See boxplots in Section~\ref{sec:experimental_evaluation}.}
	\item Did you include the total amount of compute and the type of resources used (e.g., type of GPUs, internal cluster, or cloud provider)?
		\answerYes{See Appendix~\ref{app:experiments}.}
\end{enumerate}

\item If you are using existing assets (e.g., code, data, models) or curating/releasing new assets...
\begin{enumerate}
  \item If your work uses existing assets, did you cite the creators?
	  \answerYes{See paragraphs $D$ and $E$ in Section~\ref{sec:experimental_evaluation}. Both real-life datasets are publicly available online.}
  \item Did you mention the license of the assets?
	  \answerYes{See Appendix~\ref{app:experiments}.}
  \item Did you include any new assets either in the supplemental material or as a URL? \answerNA{}
  \item Did you discuss whether and how consent was obtained from people whose data you're using/curating?
    \answerNA{}
  \item Did you discuss whether the data you are using/curating contains personally identifiable information or offensive content?
    \answerNA{}
\end{enumerate}

\item If you used crowdsourcing or conducted research with human subjects...
\begin{enumerate}
  \item Did you include the full text of instructions given to participants and screenshots, if applicable?
    \answerNA{}
  \item Did you describe any potential participant risks, with links to Institutional Review Board (IRB) approvals, if applicable?
    \answerNA{}
  \item Did you include the estimated hourly wage paid to participants and the total amount spent on participant compensation?
    \answerNA{}
\end{enumerate}

\end{enumerate}

\newpage


\appendix

\part{Appendix}


\parttoc
\newpage

\section{Notation}
\label{app:notation}


\begin{table}[h!]
	\caption{Notation table}
	\label{tab:notation}
	\hspace{-1cm}
	\begin{tabular}{ll}
		\toprule
		\multicolumn{2}{c}{}\\
		\cmidrule(r){1-2}
		Name     & Description\\
		\midrule
		$d \in \mathbb{N}^*$	&	Dimension of the feature vectors\\
		$K \in \mathbb{N}^*$	&	Number of arms\\
		$[K] := \{1,2,\dots,K\}$	& Enumeration\\
		$m \in [K-1]$	&	Number of best arms to return\\
		$\mathds{1}\{c\}$	&	Kronecker's symbol, equal to $1$ iff. claim $c$ is true\\
		$\varepsilon \in \mathbb{R}^{*+}$	&	Upper bound on the \correction{$\ell_{\infty}$} norm of the deviation to linearity\\
		$M \in \mathbb{R}^{*+}$	&	Upper bound on the \correction{$\ell_{\infty}$} norm on the mean vector\\
		$L \in \mathbb{R}^{*+}$	&	Upper bound on the \correction{$\ell_{2}$} norm on the arm feature vectors\\
		$\delta \in (0,1)$	&	Upper bound for the probability of error in identification\\
		$e_k \in \mathbb{R}^k, k \in \mathbb{N}$	&	$k^{th}$ vector of the canonical basis of $\mathbb{R}^k$\\ 
		$\Delta_K = \{p \in [0,1]^K \mid \sum_{k=1}^{K} p^k = 1\}$	&	Set of probability distributions over finite set of size $K$\\

		\midrule

		$\phi_k \in \mathbb{R}^d, k \in [K]$	&	Feature vector for arm $k$\\
		$A = [\phi_1, \phi_2, \dots, \phi_K]^\top \in \mathbb{R}^{K \times d}$	&	Feature matrix of arm contexts\\
		$\Delta_K = \{p \in [0,1]^K \mid \sum_{k=1}^{K} p^k = 1\}$	&	Set of probability distributions on finite set of size $K$\\
		$V_{\omega} :=\sum_{k \leq K} \omega_k \phi_k \phi_k^\top, \omega \in \Delta_K$	&	Design matrix associated with $\omega$\\
		$V_t := \sum_{s \leq t} \phi_{k_s} \phi_{k_s}^\top, t > 0$	&	Design matrix at time $t$\\
		$\mathcal{M} \subset \mathbb{R}^K$	&	\begin{tabular}{@{}l@{}}
								Set of realizable models:\\
								$\{ \mu \in \mathbb{R}^K \mid \exists \theta \in \mathbb{R}^d, \eta \in \mathbb{R}^K : \mu = A \theta + \eta, \norm{\mu}_{\infty} \leq M, \norm{\eta}_{\infty} \leq \varepsilon \}$
								\end{tabular}\\
			$\mu \in \mathcal{M}$	&	True mean vector: $\mu = A \theta + \eta$\\
		$N^k_t \in \mathbb{N}, k \in [K], t > 0$	&	Number of times arm $k$ has been sampled until time $t$ included\\
		$N_t = [N^1_t, N^2_t, \dots, N^K_t]^\top \in \mathbb{N}^K$	&	Vector of numbers of samplings for each arm at time $t$ included\\
		$D_{N} \in \mathbb{R}^{K \times K}, N \in \mathbb{R}^K$	&	Diagonal matrix with coefficients $N^1, N^2, \dots, N^K$\\
		$k_s, s > 0$	&	Arm sampled at time $s$\\
		$X^{k}_s, s > 0, k \in [K]$	&	Reward observed at time $s$ from arm $k$\\
		$\tau_{\delta}, \delta \in (0,1)$	&	Stopping time under $\delta$-correctness\\
		$E_{\tau_{\delta}}$	&	Event on $\delta$-correctness: $E_{\tau_{\delta}} := \left\{ \hat{S}_m \in \mathcal{S}_m(\mu) \right\}$\\

		\midrule

		$\widehat{\mu}_t \in \mathbb{R}^K, t > 0$	&	Empirical mean vector at time $t$: $\widehat{\mu}_t^a := \frac{1}{N^a_t}\sum_{s \leq t} X^a_s\mathds{1}\{k_s=a\}$\\
		$\tilde{\mu}_t \in \mathbb{R}^K, t > 0$	&	Projection of $\widehat{\mu}_t$ onto set $\mathcal{M}$ at time $t$\\
		$\hat{S}_m \subseteq [K], m \in [K-1]$	&	Answer to Top-$m$ identification as returned by the algorithm\\
		$S^\star(\mu) \subseteq [K], \mu \in \mathbb{R}^K$	&	\begin{tabular}{@{}l@{}}
										Set of best arms compared to the $m^{th}$ greatest mean:\\
										$S^\star(\mu) := \left\{k \in [K] \mid \mu^k \geq \max^m_{i \in [K]} \mu^i \right\}$
										\end{tabular}\\
		$\mathcal{S}_m(\mu), \mu \in \mathcal{M}, m \in [K-1]$	&	\begin{tabular}{@{}l@{}}
										Set of all subsets of size $m$ in $S^\star(\mu)$:\\
										$\mathcal{S}_m(\mu) := \left\{ S \subseteq S^\star(\mu) \mid \abs{S} = m \right\}$
										\end{tabular}\\
		$\Lambda_m(\mu), \mu \in \mathcal{M}$	&	\begin{tabular}{@{}l@{}}
								Set of alternative models to model $\mu$:\\
								$\Lambda_m(\mu) := \{ \lambda \in \mathcal{M} \mid \mathcal{S}_m(\lambda) \cap \mathcal{S}_m(\mu) = \emptyset\}$
								\end{tabular}\\
		$H_\mu, \mu \in \mathcal{M}$	&	\begin{tabular}{@{}l@{}}
							Inverse complexity constant:\\
							$H_{\mu} := \sup_{\omega \in \Delta_K} \inf_{\lambda \in \Lambda_m(\mu)} \sum_{k \in [K]} \omega^k \KL(\mu^k, \lambda^k)$
							\end{tabular}\\

		\midrule

		$\KL$	&	Kullback-Leibler divergence\\
		$\kl$	&	Binary relative enthropy\\
		$W_{-1}$	&	Negative branch of the Lambert $W$ function\\
		$\overline{W} : x \mapsto -W_{-1}(-e^{-x})$	&	\\

		\midrule

		$\mathcal{L}$	&	Learner algorithm\\
		$g_t(\omega), \omega \in \mathbb{R}^K, t > 0$	&	Gains fed to the learner at time $t$\\
		$c^k_t, k \in [K], t > 0$	&	\begin{tabular}{@{}l@{}}
							Optimistic bonus, such that $(\tilde{\mu}^k_t-\mu^k)^2 \leq c^k_t$ for any $k \in [K]$\\
							and large enough $t > 0$, with high probability
							\end{tabular}\\
		\bottomrule
	\end{tabular}
\end{table}

Please refer to Table~\ref{tab:notation}. Moreover, if $\omega \in \mathbb{R}^{K}$, at $t>0$, we also introduce the following notation related to orthogonal parameterizations (see Appendix \ref{app:orthogonal}):

\begin{itemize}
	\item $A_{\omega} := D_{\omega}^{1/2}A \in \mathbb{R}^{K \times d}$.
	\item $P_{\omega} := A_{\omega}(A_{\omega}^\top A_{\omega})^{\dagger}A_{\omega}^\top \in \mathbb{R}^{K \times K}$.
	\item $R_{\omega} := I_K - P_{\omega} \in \mathbb{R}^{K \times K}$, where $I_K$ is the identity matrix of dimension $K$.
	\item $V_t = A_{N_t}^\top A_{N_t} = A^\top D_{N_t} A = \sum_{k \in [K]} N^k_t \phi_k \phi_k^\top = \sum_{s \leq t} \phi_{k_s} \phi_{k_s}^\top$.
	\item $\widehat{\theta}_t := (A_{N_t}^\top A_{N_t})^{\dagger}A_{N_t}^\top D_{N_t}^{1/2} \hat{\mu}_t$, which is the standard least-squares estimator, where $^\dagger$ denotes the matrix pseudo-inverse.
	\item $\tilde{\theta}_t$ and $\tilde{\eta}_t$, parameters for the linear and misspecification parts of the projection $\tilde{\mu}_t$ of empirical mean $\widehat{\mu}_t$ onto set $\mathcal{M}$, such that $\tilde{\mu}_t = A \tilde{\theta}_t + \tilde{\eta}_t$.
	\item $\theta_t := (A_{N_t}^\top A_{N_t})^{\dagger} A_{N_t}^\top D_{N_t}^{1/2} \mu$, such that $A \theta_t = D_{N_t}^{-1/2}P_{N_t}D^{1/2}_{N_t}\mu$ if $D_{N_t}$ is invertible. $\theta_t$ is the linear part of the orthogonal parametrization of $\mu$ at time $t$ (see paragraph ``Estimation'' in Section~\ref{sub:algorithm} in the main paper).
	\item $\eta_t := \mu - A\theta_t$, equal to $D_{N_t}^{-1/2} R_{N_t} D_{N_t}^{1/2} \mu$ if $D_{N_t}$ is invertible, is the misspecification part of the orthogonal parametrization of model $\mu$ at time $t$.
	\item $S_t := D_{N_t}(\widehat{\mu}_t - \mu) \in \mathbb{R}^K$.
\end{itemize}

\section{The orthogonal parameterization and its properties}
\label{app:orthogonal}


Throughout the appendix, we shall adopt an orthogonal parametrization for mean vectors in the model $\mathcal{M}$.
In particular, we leverage the following observation: any mean vector $\mu = A\theta + \eta$ can be equivalently represented, at any time $t$, as $\mu = A\theta_t + \eta_t$, where 
$$\theta_t := (A_{N_t}^\top A_{N_t})^{\dagger} A_{N_t}^\top D_{N_t}^{1/2} \mu = V_t^{-1}\sum_{s=1}^t \mu^{k_s} \phi_{k_s}$$
is the orthogonal projection (according to the design matrix $V_t$) of $\mu$ onto the feature space and $\eta_t = \mu - A\theta_t$ is the residual. We now introduce some important properties of this parameterization.

\paragraph{Projecting the empirical mean}

When we use the orthogonal projection described above on the empirical mean $\widehat{\mu}_t$, the resulting linear part is exactly the standard least squares estimator. That is,
\begin{align*}
\widehat{\theta}_t := (A_{N_t}^\top A_{N_t})^{\dagger}A_{N_t}^\top D_{N_t}^{1/2} \hat{\mu}_t
\end{align*}

\paragraph{Projection matrices}

For $\omega \in \mathbb{R}^K_{\geq 0}$, let us define the projection matrix $P_{\omega} := A_{\omega}(A_{\omega}^\top A_{\omega})^{\dagger}A_{\omega}^\top \in \mathbb{R}^{K \times K}$ and the residual matrix $R_{\omega} := I_K - P_{\omega} \in \mathbb{R}^{K \times K}$. It is easy to check that both are orthogonal projection matrices, i.e., they are symmetric and idempotent ($P_\omega^2 = P_\omega$ and $R_\omega^2 = R_\omega$). Moreover, $P_\omega R_\omega = R_\omega P_\omega = 0$. Equipped with these matrices, we have the following useful identities:
\begin{align*}
A_{N_t} \theta_t = P_{N_t}D^{1/2}_{N_t}\mu = P_{N_t}D^{1/2}_{N_t}A\theta_t,
\end{align*}
\begin{align*}
D_{N_t}^{1/2}\eta_t =  R_{N_t} D_{N_t}^{1/2} \mu = R_{N_t} D_{N_t}^{1/2} \eta_t.
\end{align*}

\paragraph{Distances between mean vectors in the model}

Often we will need to compute quantities of the form $\norm{ \lambda - \mu }_{D_{N_t}}^2$ for different mean vectors in the model. The following lemma shows how to leverage their orthogonal decomposition to split the norm into a distance between their linear parts and a distance between their deviation from linearity.

\begin{lemma}[Linear/non-linear decomposition]\label{lemma:lin-dev-decomposition}
For any $\lambda\in\mathcal{M}$ and $t\geq 1$, there exist $\theta_t'\in\mathbb{R}^d$ and $\eta_t' \in\mathbb{R}^K$ such that $\lambda = A\theta_t' + \eta_t'$ and
\begin{align*}
	\norm{ \lambda - \mu }_{D_{N_t}}^2
	&= \norm{ \theta'_t - \theta_t }^2_{V_t} + \norm{ R_{N_t} D_{N_t}^{1/2}{\eta}_t' - R_{N_t} D_{N_t}^{1/2} \eta_t }_2^2 \: ,
	\\ \norm{ \lambda - \widehat{\mu}_t }_{D_{N_t}}^2
	&= \norm{ \theta'_t - \widehat{\theta}_t }^2_{V_t} + \norm{ R_{N_t} D_{N_t}^{1/2}{\eta}_t' - R_{N_t} D_{N_t}^{1/2} \widehat{\mu}_t }_2^2 \: .
\end{align*}
\end{lemma}
\begin{proof}
By leveraging the properties of the orthogonal decomposition and of the matrices $P_{N_t},R_{N_t}$ (in particular, $P_{N_t}R_{N_t}=0$ and $P_{N_t} + R_{N_t} = I_K$),
\begin{align*}
	\norm{ \lambda - \mu }_{D_{N_t}}^2
	&= \norm{ P_{N_t}D_{N_t}(\lambda - \mu) + R_{N_t}D_{N_t}(\lambda - \mu) }_{2}^2
	\\ &= \norm{ P_{N_t} D_{N_t}^{1/2}\lambda - P_{N_t} D_{N_t}^{1/2} \mu }_2^2 + \norm{ R_{N_t} D_{N_t}^{1/2}\lambda - R_{N_t} D_{N_t}^{1/2} \mu }_2^2
\\
	&= \norm{ P_{N_t} A_{N_t} {\theta}_t' - P_{N_t} A_{N_t} \theta_t }_2^2 + \norm{ R_{N_t} D_{N_t}^{1/2}{\eta}_t' - R_{N_t} D_{N_t}^{1/2} \eta_t }_2^2
\\
	&= \norm{ {\theta}_t' - \theta_t }^2_{V_t} + \norm{ R_{N_t} D_{N_t}^{1/2}{\eta}_t' - R_{N_t} D_{N_t}^{1/2} \eta_t }_2^2 \: .
\end{align*}
The second result can be shown analogously by noting that the projection of $\widehat{\mu}_t$ onto the linear space spanned by $A$ is exactly the least-squares estimator $\widehat{\theta}_t$.
\end{proof}

\paragraph{The non-linear part of orthogonal parameterizations}

When applying the orthogonal parameterization to a mean vector $\mu = A\theta + \eta$ with $\|\eta\|_\infty \leq \varepsilon$, while we get some crucial properties for the linear part $\theta_t$ (like concentration, see Appendix \ref{app:concentration_results}), it may be that the resulting non-linear part $\eta_t$ is such that $\|\eta_t\|_\infty > \epsilon$. However, the following result shows that $\eta_t$ cannot be too distant from $\eta$ and, in particular, that $\|\eta_t\|_\infty$ still decreases with $\epsilon$.

\begin{lemma}[Maximum deviation]\label{lem:bound-eta-t}
Let $t$ any time step such that $V_t$ is invertible. Consider the orthogonal parameterization $(\theta_t,\eta_t)$ for $\mu = A\theta + \eta$ with $\|\eta\|_\infty \leq \varepsilon$. Then,
\begin{align*}
\|\eta_t\|_\infty \leq \etabound.
\end{align*}
\end{lemma}
\begin{proof}
By definition of the orthogonal parameterization, it is easy to see that $\eta_t - \eta = A(\theta - \theta_t)$. Moreover,
\begin{align*}
\theta_t 
&:= (A_{N_t}^\top A_{N_t})^{\dagger} A_{N_t}^\top D_{N_t}^{1/2} \mu
= (A_{N_t}^\top A_{N_t})^{\dagger} A_{N_t}^\top D_{N_t}^{1/2} (A\theta + \eta)
\\ &= \theta + (A_{N_t}^\top A_{N_t})^{\dagger} A_{N_t}^\top D_{N_t}^{1/2}\eta
	= \theta + V_t^{-1} A^\top D_{N_t}\eta = \theta + V_t^{-1} \sum_{\correction{k \in [K]}} N_t^k \phi_k \eta^k.
\end{align*}
	Therefore\correction{, for any arm $k \in [K]$:}
\begin{align*}
	\abs{\eta_t^k - \eta^k}
	&= \left| \phi_k^\top V_t^{-1} \sum_{j \correction{\in [K]}} N_t^j \phi_j \eta^j \right|
	\stackrel{(a)}{\leq} \|\phi_k\|_2 \left\| V_t^{-1} \sum_{j \correction{\in [K]}} N_t^j \phi_j \eta^j\right\|_2
	\\ &=  \|\phi_k\|_2 \left\| \sum_{j \correction{ \in [K]}} N_t^j \phi_j \eta^j\right\|_{V_t^{-2}}
	\stackrel{(b)}{\leq}  \|\phi_k\|_2 \epsilon \sum_{j \correction{\in [K]}} N_t^j\left\| \phi_j \right\|_{V_t^{-2}}
   \stackrel{(c)}{\leq} \|\phi_k\|_2 \epsilon K,
\end{align*}
	where (a) is from Cauchy-Schwartz inequality, (b) uses the sub-additivity of the norm, and (c) uses that, for each $j\in[K]$, $V_t = \sum_{\correction{q \in [K]}} N_t^{\correction{q}} \phi_{\correction{q}}\phi_{\correction{q}}^\top \succeq N_t^j \phi_j\phi_j^\top$ \correction{(in the sense of the partial order on positive definite matrices)}. Using that features are bounded by $L$ in \correction{$\ell_2$}-norm,
\begin{align*}
\|\eta_t - \eta\|_\infty \leq LK\epsilon,
\end{align*}
from which the result easily follows.
\end{proof}


\paragraph{The linear parts of different parametrizations}

We consider mainly two parametrizations of $\mu$: the orthogonal parametrization with respect to $N_t$ and another $(\theta, \eta)$ for which $\Vert \eta \Vert_\infty \le \varepsilon$. We will now relate the linear parts of these two parametrizations.

\begin{lemma}\label{lem:different_param_linear}
Let $t$ any time step such that $V_t$ is invertible. Consider the orthogonal parameterization $(\theta_t,\eta_t)$ for $\mu = A\theta + \eta$ with $\|\eta\|_\infty \leq \varepsilon$. Then
\begin{align*}
\Vert \theta_t - \theta \Vert_{V_t} \le \sqrt{t} \varepsilon \: .
\end{align*}

\end{lemma}

\begin{proof}
We use the expression $\theta_t = \theta + V_t^{-1} A^\top D_{N_t}\eta$ derived in the last paragraph, the fact that $P_{N_t}$ is a projection and lastly $\Vert \eta \Vert_\infty \le \varepsilon$:
\begin{align*}
\Vert \theta_t - \theta \Vert_{V_t}
	&= \norm{ V_t^{-1} A^\top D_{N_t}\eta }_{V_t}
\\
&= \sqrt{\eta^\top D_{N_t} A V_t^{-1} A^\top D_{N_t} \eta}
\\
	&= \norm{ D_{N_t}^{1/2} \eta }_{P_{N_t}}
	\le \norm{ D_{N_t}^{1/2} \eta }
= \Vert \eta \Vert_{D_{N_t}}
\le \sqrt{t} \varepsilon
\: .
\end{align*}

\end{proof}

\section{Tractable lower bound for the general Top-$m$ identification problem}
\label{app:lower_bound}


We present here the proofs for the claims made in the main paper in Section~\ref{sec:lower_bound_top_m}.

\subsection{Proof of Lemma~\ref{lemma:rewrite_lower_bound} and Theorem~\ref{th:lower_bound}}
\label{sub:proof_of_lemma_1_and_theorem_1}

\begin{lemma*}
\textnormal{(Lemma~\ref{lemma:rewrite_lower_bound} in the main paper)}
$\forall \mu, \lambda \in \mathbb{R}^K$s.t. $\abs{S^\star(\mu)} = m$,
\begin{align*}
\mathcal{S}_{m}(\lambda) \cap \mathcal{S}_{m}(\mu) = \emptyset
\quad \Leftrightarrow \quad
	\exists i \notin S^\star(\mu)\ \exists j \in S^\star(\mu), \lambda^i > \lambda^j \: .
\end{align*}
\end{lemma*}

\begin{proof}
To see this, first suppose that the condition on the right holds.
	That is, there exist $(i,j) \in \left( S^\star(\mu) \right)^c \times S^\star(\mu)$, where $\abs{S^\star(\mu)}=m$, such that $\lambda^i > \lambda^j$. Then, we have two cases. If $j$ does not belong to any of the top-$m$ sets of $\lambda$, that is, $j \not\in S^{\star}(\lambda)$, the result follows trivially since $j$ belongs to the top-$m$ set of $\mu$ $S^{\star}(\mu)$ and $\mathcal{S}_m(\mu) = \{S^\star(\mu)\}$.
If, on the other hand, $j$ belongs to at least one top-$m$ set of $\lambda$, that is, $j \in S^{\star}(\lambda)$, then $i \in S^{\star}(\lambda)$ as well since $\lambda^i > \lambda^j$. But $i \not\in S^{\star}(\mu)$, which proves that $\mathcal{S}_{m}(\lambda) \cap \mathcal{S}_{m}(\mu) = \emptyset$.
Suppose now that $\mathcal{S}_{m}(\lambda) \cap \mathcal{S}_{m}(\mu) = \emptyset$ holds and, by contradiction, that $\forall i \notin S^\star(\mu)\ \forall j \in S^\star(\mu), \lambda^i \leq \lambda^j$. This trivially implies that $S^\star(\mu)$ is a valid top-$m$ set of $\lambda$.
That is, $\mathcal{S}_{m}(\lambda) \cap \mathcal{S}_{m}(\mu) \neq \emptyset$ and we have our desired contradiction.
\end{proof}

\begin{theorem*}
\textnormal{(Theorem~\ref{th:lower_bound} in the main paper)}
For any $\delta \le 1/2$, for any $\delta$-correct algorithm $\mathfrak{A}$ on $\mathcal{M}$, for any bandit instance $\mu \in \mathcal{M}$ 
such that $|S^\star(\mu)| = m$, the following lower bound holds on the stopping time $\tau_\delta$ of $\mathfrak{A}$ on instance $\mu$:
\begin{align*}
	\mathbb{E}_{\mu}^{\mathfrak{A}}[\tau_\delta]
  \ge \left( \sup_{\omega\in\Delta_K}\min_{i \notin S^\star(\mu)}\min_{j \in S^\star(\mu)}\inf_{\lambda\in\mathcal{M} : \lambda^i > \lambda^j}
    \sum_{k\in[K]} \omega^k\mathrm{KL}(\mu^{k},\lambda^k) \right)^{-1} \log \left( \frac{1}{2.4\delta} \right).
\end{align*}
\end{theorem*}

\begin{proof}
	We start from Equation~\ref{eq:lower_bound} (main paper), and using Lemma~\ref{lemma:rewrite_lower_bound}, we can rewrite the $\inf$ operator. That yields the desired expression.
\end{proof}

\subsection{Proof of Lemma~\ref{lemma:match_unstructured}}
\label{sub:proof_of_lemma_2}

Let $\Lambda_m(\mu, \mathcal M') \subseteq \mathcal M'$ denote the set of alternative models to $\mu \in \mathbb{R}^K$ in the model $\mathcal{M}'$ (which might be different from $\mathcal{M}$).
Consider the lower bound problem
\begin{align*}
	H_{\mu}(\mathcal{M}') := \sup_{\omega \in \Delta_K} \inf_{\lambda \in \Lambda_m(\mu, \mathcal{M}')} \sum_{k \in [K]} \omega^k \KL(\mu^k, \lambda^k) \: .
\end{align*}
A pair of equilibrium strategies for that problem is composed of $\omega \in \Delta_K$ and $q \in \mathcal P(\Lambda_m(\mu, \mathcal M'))$ (which is the set of probability distributions on $\Lambda_m(\mu, \mathcal M')$). Let $Q_{\mathcal M'}$ be the set of equilibrium distributions. For $q \in Q_{\mathcal M'}$, let $\Lambda_q \subseteq \Lambda_m(\mu, \mathcal{M}')$ be its support.

\begin{lemma}\label{lem:model_inclusion}
Let $\mathcal M_1, \mathcal M_2$ be models such that $\mathcal M_1 \subseteq \mathcal M_2$. For any $q \in Q_{\mathcal M_2}$, if $\Lambda_q \subseteq \mathcal M_1$, then $H_{\mu}(\mathcal M_1) = H_{\mu}(\mathcal M_2)$.
\end{lemma}

\begin{proof}
First, we have $H_\mu(\mathcal M_1) \ge H_{\mu}(\mathcal M_2)$ since $\mathcal M_1 \subseteq \mathcal M_2$. If $\Lambda_q \subseteq \mathcal M_1$, then using successively $q \in \mathcal P(\Lambda(\mu, \mathcal M_1))$ and $q \in Q_{\mathcal{M}_2}$ ,
\begin{align*}
H_{\mu}(\mathcal M_1)
&= \sup_{\omega \in \Delta_K} \inf_{\lambda \in \Lambda_m(\mu, \mathcal{M}_1)} \sum_{k \in [K]} \omega^k \KL(\mu^k, \lambda^k)
\\
&\le \sup_{\omega \in \Delta_K} \mathbb{E}_{\lambda \sim q} \sum_{k \in [K]} \omega^k \KL(\mu^k, \lambda^k)
= H_{\mu}(\mathcal M_2) \: .
\end{align*}
\end{proof}

For $\lambda \in \mathbb{R}^K$, let $\abs{\lambda}_\varepsilon = \inf \{\norm{\eta}_\infty \mid \exists \theta \in \mathbb{R}^d, \ \lambda = A \theta + \eta \}$.
Let us now consider $\mathcal M$ as defined in Equation~\ref{eq:set-models} in the main paper, with misspecification upper bound $\varepsilon \geq 0$.

\begin{lemma}\label{lem:model_inclusion_M}
	Let $\mathcal M' \subseteq \{\lambda \in \mathbb{R}^K \mid \norm{\lambda}_\infty \le M\}$ be a set of models such that $\mathcal M \subseteq \mathcal M'$ and $\varepsilon > \varepsilon_\mu(\mathcal M') := \inf_{q \in Q_{\mathcal M'}} \sup_{\lambda \in \Lambda_q} \abs{\lambda}_\varepsilon$.\footnote{Note that indeed quantity $\varepsilon_\mu(\mathcal M')$ depends on $\mu$, since $Q_{\mathcal M'}$ is defined with respect to $\mu$.}
Then $H_\mu(\mathcal M) = H_\mu(\mathcal M')$.
\end{lemma}
\begin{proof}
If $\varepsilon > \inf_{q \in Q_{\mathcal M'}} \sup_{\lambda \in \Lambda_q} \abs{\lambda}_\varepsilon$, then there exists $q \in Q_{\mathcal M'}$ such that for all $\lambda \in \Lambda_q$, $\abs{\lambda}_\varepsilon \le \varepsilon$. Hence $\Lambda_q \subseteq \mathcal M$ and we apply Lemma~\ref{lem:model_inclusion}.
\end{proof}

For any model $\mathcal M'$, there exist equilibrium strategies for which $q$ is supported on $K$ points \cite{degenne2019pure}. Hence $\varepsilon_\mu(\mathcal M')$ is always finite.

Let $\mathcal M_u := \mathbb{R}^K$ be the set of unstructured models, and for $a,b \in \mathbb{R}$, $\mathcal M_{[a,b]} := \{\lambda \in \mathbb{R}^K \mid \forall k \in [K], \lambda^k \in [a,b]\}$ be the set of models that verify a boundedness assumption.

\begin{lemma}\label{lem:restriction_unstructured_model}
Let $\mu^{(K)} := \min_j \mu^j$ and $\mu^{(1)} := \max_j \mu^j$.
For all $\mu \in \mathbb{R}^K$, $H_\mu(\mathcal M_u) = H_\mu(\mathcal M_{[\mu^{(K)}, \mu^{(1)}]})$ .
\end{lemma}

\begin{proof}
Let us consider any $\lambda \in \Lambda_m(\mu, \mathcal{M}_u)$, such that there exists $k \in [K]$ with $\lambda^k \not\in [\mu^{(K)}, \mu^{(1)}]$. Let us define $\tilde{\lambda}$ as the projection of $\lambda$ onto $[\mu^{(K)}, \mu^{(1)}]^K$. Then $\tilde{\lambda}$ satisfies $\tilde{\lambda} \in \Lambda_m(\mu, \mathcal{M}_{[\mu^{(K)}, \mu^{(1)}]}) \subseteq \Lambda_m(\mu, \mathcal{M}_u)$, and by monotonicity of the Kullback-Leibler divergence in one-parameter exponential families, for all $k \in [K]$, $\KL(\mu^k, \tilde{\lambda}^k) \le \KL(\mu^k, \lambda^k)$. Thus for all $\omega \in \Delta_K$
\begin{align*}
\sum_{k \in [K]} \omega^k \KL(\mu^k, \tilde{\lambda}^k)
\le \sum_{k \in [K]} \omega^k \KL(\mu^k, \lambda^k) \: .
\end{align*}
For $q \in Q_{\mathcal M_u}$, let $\tilde{q}$ be the distribution in which every support point $\lambda$ of $q$ is transported onto its projection $\tilde{\lambda}$. Then for all $\omega \in \triangle_K$,
\begin{align*}
\mathbb{E}_{\lambda \sim \tilde{q}} \sum_{k \in [K]} \omega^k \KL(\mu^k, \lambda^k)
\le \mathbb{E}_{\lambda \sim q} \sum_{k \in [K]} \omega^k \KL(\mu^k, \lambda^k)
\: ,
\end{align*}
from which we obtain that $\tilde{q}$ has lower objective value than $q$. Since $q \in Q_{\mathcal M_u}$, then $\tilde{q} \in Q_{\mathcal M_u}$ as well. By construction, its support verifies $\Lambda_{\tilde{q}} \subseteq \mathcal M_{[\mu^{(K)}, \mu^{(1)}]}$. We conclude with Lemma~\ref{lem:model_inclusion}.
\end{proof}

Applying Lemma~\ref{lem:model_inclusion_M} to $\mathcal M_{[\mu^{(K)}, \mu^{(1)}]}$, together with Lemma~\ref{lem:restriction_unstructured_model}, we finally obtain Lemma~\ref{lemma:match_unstructured} from the main paper, restated here using the notations we introduced:
\begin{lemma*}
	If $\varepsilon > \mu^{(1)} - \mu^{(K)}$ then $H_\mu(\mathcal M) = H_\mu(\mathcal M_{[\mu^{(K)}, \mu^{(1)}]}) = H_\mu(\mathcal M_u)$~.
\end{lemma*}

\subsection{Computing the closest alternative}
\label{sub:computing_the_closest_alternative}

In order to compute the closest alternative to $\mu \in \mathcal M$ in the half-space $\{\lambda \in \mathcal{M} \mid \lambda^k \ge \lambda^j\}$, the optimization problem we need to solve is
\begin{align*}
	\inf_{\theta, \eta}~ &\frac{1}{2} \norm{A \theta + \eta - \mu}_{D_{N_t}}^2
\\
\text{s.t } & (e_k - e_j)^\top (A \theta + \eta) \ge 0
\\
& \|A\theta + \eta\|_{\infty} \leq M
\\
& \|\eta\|_{\infty} \leq \varepsilon \: .
\end{align*}
In our implementation, and thus in the remainder of this section, we shall drop the boundedness constraint $\|A\theta + \eta\|_{\infty} \leq M$ which has typically a negligible effect on the algorithm's behavior.

\paragraph{Quadratic problem}

We express the problem as function of the variable $(\theta^\top, \eta^\top)^\top$. Up to the constant term, this problem is equivalent to
\begin{align*}
\inf_{\theta, \eta} &\left( \begin{array}{c} \theta \\ \eta \end{array}\right)^\top \left( \begin{array}{cc} A^\top D_N A  & A^\top D_N \\ D_N A & D_N\end{array}\right) \left( \begin{array}{c} \theta \\ \eta \end{array}\right)
- \left( \begin{array}{c} \theta \\ \eta \end{array}\right)^\top \left( \begin{array}{c} A^\top D_N \mu\\ D_N \mu \end{array}\right)
\\
\text{s.t. }& \left( \begin{array}{c} A^\top(e_j - e_k) \\ e_j - e_k \end{array}\right)^\top\left( \begin{array}{c} \theta \\ \eta \end{array}\right) \le 0
\\
& \|\eta\|_{\infty} \leq \varepsilon \: .
\end{align*}
In the code, we directly solve the problem under this form using a quadratic problem solver.

\paragraph{Computing the closest alternative}

We now detail the form of the solutions analytically (as much as possible).
Let $j, k \in [K]$, $j \neq k$. We want to compute the closest alternative in the half-space $\{\lambda \in \mathcal{M} \mid \lambda^k \ge \lambda^j\}$ to $\mu \in \mathbb{R}^K$. That is, we compute the solution to
\begin{align*}
	\inf_{\theta, \eta}~ &\frac{1}{2} \norm{A \theta + \eta - \mu}_{D_{N_t}}^2
\\
\text{s.t } & (e_k - e_j)^\top (A \theta + \eta) \ge 0
\\
& \eta \in \mathcal C
\end{align*}
Here, to highlight the generality of the following derivation, we replace the $\ell_\infty$ norm constraint on $\eta$ with any convex set $\mathcal{C}$.
To simplify the notation, we denote by $D_N$ the diagonal matrix with $N_t$ on the diagonal and $u := e_j - e_k$.
The problem above is then written as
\begin{align*}
	\inf_{\theta, \eta}~ &\frac{1}{2} \norm{D_N^{1/2}A \theta + D_N^{1/2}\eta - D_N^{1/2}\mu}_2^2
\\
\text{s.t } & u^\top (A \theta + \eta) \le 0
\\
& \eta \in \mathcal C
\end{align*}

\begin{assumption}
	At $t_0$, $A^\top D_{N_{t_0}} A = V_{t_0}$ is invertible.
\end{assumption}
See paragraph ``Initialization phase'' in Subsection~\ref{sub:algorithm} to see how that assumption is ensured in practice.
We now suppose that $t \ge t_0$.
Minimizing first in $\theta$ at fixed $\eta$, we solve the problem
\begin{align*}
	\inf_{\theta}~ &\frac{1}{2} \norm{D_N^{1/2}A \theta + D_N^{1/2}\eta - D_N^{1/2}\mu}_2^2
\\
\text{s.t } & u^\top (A \theta + \eta) \le 0
\end{align*}
The Lagrangian is $L(\theta, \alpha) = \frac{1}{2} \norm{D_N^{1/2}A \theta + D_N^{1/2}\eta - D_N^{1/2} \mu}_2^2 + \alpha u^\top (\eta + A \theta)$ with $\alpha \ge 0$.
We get that at the optimal $\theta$,
\begin{align*}
A^\top D_N (A \theta + \eta - \mu)
= -\alpha A^\top u
\quad \implies \quad
\theta
= (A^\top D_N A)^{-1}A^\top(-\alpha u + D_N\mu - D_N\eta)
\: .
\end{align*}
At the optimum, from the KKT conditions, either $\alpha = 0$ and $u^\top (A \theta + \eta) \le 0$, or $\alpha > 0$ and $u^\top A \theta = -u^\top \eta$.

\paragraph{Case $\alpha=0$.} If $\alpha = 0$, then $\theta = (A^\top D_N A)^{-1}A^\top D_N (\mu - \eta)$, $D_N^{1/2}(A \theta + \eta - \mu) = (D_N^{1/2}A(A^\top D_N A)^{-1}A^\top D_N^{1/2} - I) D_N^{1/2}(\mu - \eta)$ and the value of the optimization problem is the norm of this quantity.

Let $P_N = D_N^{1/2}A (A^\top D_N A)^{-1} A^\top D_N^{1/2}$.
Note: it is symmetric and idempotent ($P_N^2 = P_N$), meaning that it is an orthogonal projection.
Let $R_N = I - P_N$ be the residual matrix. We also have $R_N^2 = R_N$. Furthermore, $P_N R_N = R_N P_N = 0$.

With these notations, $D_N^{1/2}A\theta = P_N D_N^{1/2}(\mu - \eta)$, $D_N^{1/2}(A \theta + \eta - \mu) = - R_N D_N^{1/2}(\mu - \eta)$ and the value of the optimization problem is $\frac{1}{2}\Vert R_N D_N^{1/2}\eta - R_N D_N^{1/2}\mu \Vert^2$. The case $\alpha = 0$ is possible only if the constraint is then satisfied, that is if $u^\top (A \theta + \eta) \le 0$ at the optimum, i.e. if $u^\top(A^\top (A^\top D_N A)^{-1}A^\top D_N\mu + (I - A^\top (A^\top D_N A)^{-1}A^\top D_N) \eta) \le 0$. The problem we need to solve in that case is
\begin{align*}
	\min_{\eta_N}~ \ &\frac{1}{2} \norm{R_N D_N^{1/2}\eta - R_N D_N^{1/2} \mu}_2^2 \\
\text{s.t. } & u^\top (I - A^\top (A^\top D_N A)^{-1}A^\top D_N) \eta \le - u^\top A^\top (A^\top D_N A)^{-1}A^\top D_N\mu\\
& \eta \in \mathcal C
\end{align*}
If $\mathcal C$ is convex this is a convex optimization problem. It can happen that there is no feasible point, which simply means that there is no solution with $\alpha = 0$.

\paragraph{Case $\alpha \ne 0$.} Consider now the case $\alpha > 0$. We get
\begin{align*}
&u^\top A \theta = - u^\top \eta
\\
\implies \quad
&u^\top A (A^\top D_N A)^{-1}A^\top(-\alpha u + D_N\mu - D_N\eta)
= - u^\top \eta
\\
\implies \quad
&\alpha u^\top A (A^\top D_N A)^{-1}A^\top u
= u^\top A (A^\top D_N A)^{-1}A^\top D_N (\mu - \eta) + u^\top \eta
\end{align*}

Then
\begin{align*}
D_N^{1/2} A \theta
&= D_N^{1/2} A (A^\top D_N A)^{-1}A^\top(-\alpha u + D_N\mu - D_N\eta)
\\
&= - \alpha D_N^{1/2} A (A^\top D_N A)^{-1}A^\top u + P_N D_N^{1/2} (\mu - \eta)
\\
D_N^{1/2}(A \theta + \eta - \mu)
&= - \alpha D_N^{1/2} A (A^\top D_N A)^{-1}A^\top u - R_N D_N^{1/2}(\mu - \eta)
\\
&= -\frac{u^\top A (A^\top D_N A)^{-1}A^\top D_N (\mu - \eta) + u^\top \eta}{u^\top A (A^\top D_N A)^{-1}A^\top u} D_N^{1/2} A (A^\top D_N A)^{-1}A^\top u
\\&\quad - R_N D_N^{1/2}(\mu - \eta)
\end{align*}

We can now see that $D_N^{1/2}(A\theta + \eta - \mu)$ is linear in $\eta$ and the objective value $\frac{1}{2}\norm{D_N^{1/2}(A \theta + \eta - \mu)}^2$ is quadratic in $\eta$. 
We need to solve a quadratic optimization problem under the constraint $\eta \in \mathcal C$.
Let's now simplify that optimization problem. We first show that the cross term in $\frac{1}{2}\norm{ D_N^{1/2}(A \theta + \eta - \mu) }_2^2 = \frac{1}{2}\norm{ - \alpha D_N^{1/2} A (A^\top D_N A)^{-1}A^\top u - R_N D_N^{1/2}(\mu - \eta) }_2^2$ is zero. Note: if $D_N$ is invertible, then $\frac{1}{2}\norm{ - \alpha D_N^{1/2} A (A^\top D_N A)^{-1}A^\top u - R_N D_N^{1/2}(\mu - \eta) }_2^2 = \frac{1}{2}\norm{ - \alpha P_N D_N^{-1/2} u - R_N D_N^{1/2}(\mu - \eta) }_2^2$ and the fact that the cross term is 0 is a simple consequence of $P_N R_N = R_N P_N = 0$.
\begin{align*}
&(R_N D_N^{1/2}(\mu - \eta))^\top D_N^{1/2} A (A^\top D_N A)^{-1}A^\top u
\\
&= ((I - P_N) D_N^{1/2}(\mu - \eta))^\top D_N^{1/2} A (A^\top D_N A)^{-1}A^\top u
\\
&= (\mu - \eta)^\top D_N  A (A^\top D_N A)^{-1}A^\top u - (\mu - \eta)^\top D_N^{1/2} P_N D_N^{1/2} A (A^\top D_N A)^{-1}A^\top u
\\
&= (\mu - \eta)^\top D_N  A (A^\top D_N A)^{-1}A^\top u
- (\mu - \eta)^\top D_N A (A^\top D_N A)^{-1}A^\top D_N A (A^\top D_N A)^{-1}A^\top u
\\
&= 0
\: .
\end{align*}
Now that we established that the cross term is zero, the objective value is simply the sum of two square terms,
\begin{align*}
	\frac{1}{2}\norm{ D_N^{1/2}(A \theta + \eta - \mu) }_2^2
&= \frac{1}{2} \alpha^2 u^\top A (A^\top D_N A)^{-1}A^\top u + \frac{1}{2}(\mu - \eta)^\top D_N^{1/2} R_N D_N^{1/2} (\mu - \eta)
\\
&= \frac{1}{2}\frac{\left(u^\top A (A^\top D_N A)^{-1}A^\top D_N (\mu - \eta) + u^\top \eta\right)^2}{u^\top A (A^\top D_N A)^{-1}A^\top u} 
\\&\quad + \frac{1}{2}(\mu - \eta)^\top D_N^{1/2} R_N D_N^{1/2} (\mu - \eta)
\\
&= \frac{1}{2}\eta^\top Q \eta + q^\top \eta + C
\end{align*}
where $C$ doesn't depend on $\eta$ and
\begin{align*}
Q &= D_N^{1/2}R_N D_N^{1/2}
\\&\quad+ \frac{1}{u^\top A (A^\top D_N A)^{-1}A^\top u}\left((I - D_N A (A^\top D_N A)^{-1} A^\top) u\right)\left((I - D_N A (A^\top D_N A)^{-1} A^\top) u\right)^\top
\\
q &= \frac{u^\top A (A^\top D_N A)^{-1}A^\top D_N \mu}{u^\top A (A^\top D_N A)^{-1}A^\top u} (I - D_N A (A^\top D_N A)^{-1} A^\top) u - D_N^{1/2} R_N D_N^{1/2} \mu
\: .
\end{align*}
Again if $D_N$ is invertible these have simpler expressions:
\begin{align*}
Q &= D_N^{1/2}\left(R_N + \frac{1}{u^\top D_N^{-1/2} P_N D_N^{-1/2}u}(R_N D_N^{-1/2} u)(R_N D_N^{-1/2} u)^\top\right)D_N^{1/2}
\\
q &= D_N^{1/2}R_N \left(\frac{u^\top D_N^{-1/2} P_N D_N^{1/2} \mu}{u^\top D_N^{-1/2} P_N D_N^{-1/2} u} D_N^{-1/2} u - D_N^{1/2} \mu\right)
\: .
\end{align*}

We are looking for a solution to
\begin{align*}
\arg\min_{\eta \in \mathcal C}~\frac{1}{2}\eta^\top Q \eta + q^\top \eta \: .
\end{align*}
This is a quadratic objective. The difficulty of finding the minimum depends on $\mathcal C$.

\paragraph{Summary.} To compute the closest alternative in a half-space, we compute the solution to two quadratic problems corresponding to the possibilities that Lagrangian multiplier $\alpha$ satisfies either $\alpha = 0$ or $\alpha > 0$. Then we retain the solution with the minimal objective value.

\section{The \algo algorithm}
\label{app:algo_algorithm}


\subsection{Initialization}
\label{subapp:initialization}

\algo starts by pulling a deterministic sequence of $t_0$ arms that make the minimum eigenvalue of the resulting design matrix $V_{t_0}$ larger than $\mineig$.
Since the rows of $A$ span $\mathbb{R}^d$, such sequence can be found by taking any subset of $d$ arms that span the whole space (e.g., by computing a barycentric spanner~\cite{awerbuch2008online}) and pulling them in a round robin fashion until the desired condition is met.

In order to get an approximation of the length $t_0$ of the initialization phase, let us denote $\sigma_{\min}(M)$ the minimal singular value of a matrix $M$.
Let us consider $\mathcal{B} = \{ b_1, b_2, \dots, b_d \} \subseteq [K]$, $|\mathcal{B}|=d$, the barycentric spanner of size $d$ computed on matrix $A$. Then, if we stopped the round-robin sampling such that each arm in the barycentric spanner is sampled exactly $u_0$ times, $V_{t_0} = u_0\sum_{k \in \mathcal{B}}\phi_k\phi_k^\top$. To ensure that $V_t \succeq 2 L^2 I_d$, we need $u_0\sigma_{\min}\left( \sum_{k \in \mathcal{B}} \phi_k\phi_k^\top \right) \geq 2L^2$. Let $\Gamma'(A) := \min \left\{ \sigma_{\min}\left( \sum_{k \in \mathcal B} \phi_k\phi_k^\top \right) \mid \mathcal B \text{ $d$-sized spanner of $A$ }\right\}$. Then $u_0 = \left\lceil \frac{2 L^2}{\Gamma'(A)} \right\rceil$ is large enough.

We obtain the bound $t_0 \le d \left\lceil \frac{2 L^2}{\Gamma'(A)} \right\rceil$.

\subsection{Projection of the empirical mean onto the set of realizable models $\mathcal{M}$}
\label{subapp:projection_onto_M}

As done in Equation~\ref{eq:set-models} in the main paper, we define the set of realizable models as
\begin{align*}
	\mathcal{M} := \left\{ \mu = A\theta + \eta \in \mathbb{R}^K \mid \exists \theta\in\mathbb{R}^d \exists \eta\in\mathbb{R}^K, \norm{\eta}_\infty \leq \varepsilon \land \norm{A\theta+\eta}_\infty \leq M \right\} \: .
\end{align*}
We require our estimates of $\mu$ to be in this set, but the estimate at time $t$ $\widehat{\mu}_t$ might not satisfy the constraint on its $\ell_{\infty}$ norm (\ie $\norm{\widehat{\mu}_t}_\infty > M$). We then directly project the empirical mean vector onto $\mathcal{M}$. Define
\begin{align}\label{eq:estimator-projected}
(\tilde{\theta}_t,\tilde{\eta}_t) :=
	\argmin_{\theta',\eta' : A\theta'+\eta' \in \mathcal{M}}~ \norm{A\theta' + \eta' - \widehat{\mu}_t}_{D_{N_t}}^2.
\end{align}

\begin{lemma}\label{lem:estimator_properties}
	Let $\tilde{\mu}_t = A\tilde{\theta}_t + \tilde{\eta}_t$,\footnote{Note that the equation $\widehat{\theta}_t=\tilde{\theta}_t$ mentioned in Section~\ref{sub:algorithm} in the main paper no longer holds, because we consider the boundedness assumption $\mu \in \mathcal{M} \implies \norm{\mu}_{\infty} \leq M$} where $(\tilde{\theta}_t,\tilde{\eta}_t)$ are the solution of~\eqref{eq:estimator-projected}. Then, all the following hold:
\begin{align*}
	\norm{\tilde{\mu}_t-\mu}_{D_{N_t}}^2 &\leq \norm{\mu - \widehat{\mu}_t}_{D_{N_t}}^2 \: ,
	\\ \norm{\tilde{\theta}_t - \theta_t}_{V_t}^2 &\leq \norm{ \widehat{\theta}_t - \theta_t }_{V_t}^2 \: ,
	\\ \norm{ R_{N_t}D_{N_t}^{1/2}\tilde{\eta}_t - R_{N_t}D_{N_t}^{1/2}\eta_t }_2^2 &\leq \norm{ R_{N_t}D_{N_t}^{1/2}\widehat{\mu}_t - R_{N_t}D_{N_t}^{1/2}\eta_t }_2^2 \: ,
	\\ \norm{\tilde{\theta}_t - \theta}_{V_t}^2 &\leq \norm{ \widehat{\theta}_t - \theta }_{V_t}^2 \: ,
	\\ \norm{ R_{N_t}D_{N_t}^{1/2}\tilde{\eta}_t - R_{N_t}D_{N_t}^{1/2}\eta }_2^2 &\leq \norm{ R_{N_t}D_{N_t}^{1/2}\widehat{\mu}_t - R_{N_t}D_{N_t}^{1/2}\eta }_2^2 \: 
\end{align*}
\end{lemma}
\begin{proof}
The first inequality is easy to check by using $\mu\in\mathcal{M}$ together with the non-expansion of the projection in the optimized norm.

The proof of the other inequalities extends Lemma $9$ in \citep{tirinzoni2020asymptotically}.
Note that, using Lemma~\ref{lemma:lin-dev-decomposition}, an equivalent formulation of \eqref{eq:estimator-projected} is
\begin{align*}
(\tilde{\theta}_t,\tilde{\eta}_t) &:=
	\argmin_{\theta',\eta' : A\theta'+\eta' \in \mathcal{M}} \left\{ \norm{ P_{N_t} A_{N_t}\theta' - P_{N_t} D_{N_t}^{1/2}\widehat{\mu}_t}_2^2 + \norm{R_{N_t} D_{N_t}^{1/2} \eta'  - R_{N_t} D_{N_t}^{1/2}\widehat{\mu}_t}_2^2 \right\}
\\
	&= \argmin_{\theta',\eta' : A\theta'+\eta' \in \mathcal{M}} \left\{ \norm{\theta' - \widehat{\theta}_t }_{V_t}^2 + \norm{  \eta'  - \widehat{\mu}_t}^2_{D_{N_t}^{1/2}R_{N_t}D_{N_t}^{1/2} } \right\}
\end{align*}
	This is the minimization of a convex function over a convex set. {{For any $\theta' \in \mathbb{R}^d$, $\eta' \in \mathbb{R}^K$,}} let $f(\theta') = \norm{\theta' - \widehat{\theta}_t}_{V_t}^2$ and $g(\eta') = \norm{ \eta'  - \widehat{\mu}_t}^2_{D_{N_t}^{1/2}R_{N_t}D_{N_t}^{1/2}}$. Therefore, using the first-order optimality conditions for convex functions (see, e.g., Theorem 2.8 in \citep{orabona2019modern}), $(\tilde{\theta}_t,\tilde{\eta}_t)$ are minimizers if and only if for each $\theta',\eta' : A\theta' + \eta' \in \mathcal{M}$,
\begin{align*}
\langle \nabla_\theta f(\tilde{\theta}_t), \theta' - \tilde{\theta}_t \rangle \geq 0 
\quad &\implies \quad
(\tilde{\theta}_t - \widehat{\theta}_t)^T V_t (\theta' - \tilde{\theta}_t) \geq 0
\\ 
\langle \nabla_\eta g(\tilde{\eta}_t), \eta' - \tilde{\eta}_t \rangle \geq 0 
\quad &\implies \quad
(\tilde{\eta}_t - \widehat{\mu}_t)^T D_{N_t}^{1/2}R_{N_t}D_{N_t}^{1/2} (\eta' - \tilde{\eta}_t) \geq 0
\end{align*}
Note that $\mu = A\theta_t + \eta_t$, thus the orthogonal parametrization $(\theta_t,\eta_t)$ is such that $A\theta_t + \eta_t \in \mathcal{M}$. Thus, $(\theta_t,\eta_t)$ are feasible solutions. This implies
\begin{align*}
	\norm{ \widehat{\theta}_t - \theta_t }_{V_t}^2 
	= \norm{ \widehat{\theta}_t - \tilde{\theta}_t }_{V_t}^2 + \norm{ \tilde{\theta}_t - \theta_t }_{V_t}^2 + 2(\widehat{\theta}_t - \tilde{\theta}_t)^T V_t (\tilde{\theta}_t - \theta_t) 
	\geq \norm{ \tilde{\theta}_t - \theta_t }_{V_t}^2
\end{align*}
and
\begin{align*}
	\norm{ \widehat{\mu}_t - \eta_t }_{D_{N_t}^{1/2}R_{N_t}D_{N_t}^{1/2}}^2 
	&= \norm{ \widehat{\mu}_t - \tilde{\eta}_t }_{D_{N_t}^{1/2}R_{N_t}D_{N_t}^{1/2}}^2  + \norm{ \tilde{\eta}_t - \eta_t }_{D_{N_t}^{1/2}R_{N_t}D_{N_t}^{1/2}}^2
	\\&\quad + 2(\widehat{\mu}_t - \tilde{\eta}_t)^T D_{N_t}^{1/2}R_{N_t}D_{N_t}^{1/2} (\tilde{\eta}_t - \eta_t)
	\\ &\geq \norm{ \tilde{\eta}_t - \eta_t }_{D_{N_t}^{1/2}R_{N_t}D_{N_t}^{1/2}}^2.
\end{align*}
Rearranging concludes the proof of the second and the third inequalities. To show the last two inequalities, simply use the same argument by noting that $(\theta,\eta)$ is also a feasible solution (since $\mu = A\theta +\eta \in \mathcal{M}$).
\end{proof}

\section{Concentration results}
\label{app:concentration_results}


\subsection{Concentration of the linear part}

In this section we derive concentration results for

\begin{align*}
	\norm{ \widehat{\theta}_t - \theta_t }_{V_t}^2
	= \norm{ V_t^{-1} A^\top S_t }^2_{V_t}
	= \norm{ A^\top S_t }_{V_t^{-1}}^2
\: .
\end{align*}
We rewrite the quantities involved to make obvious that this is the usual self-normalized quantity from the linear bandit literature \citep{abbasi2011improved}:
\begin{align*}
A^\top S_t
= \sum_{s = 1}^t \left(X_s^{k_s} - \mu^{k_s}\right) A^\top e_{k_s}
= \sum_{s=1}^t \left(X_s^{k_s} - \mu^{k_s}\right) \phi_{k_s}
\quad \text{ and } \quad
V_t
= \sum_{s=1}^t \phi_{k_s} \phi_{k_s}^\top
\: .
\end{align*}

We restate here Theorem~20.4 (in combination with the Equation~20.9) of \cite{lattimore2020bandit}, which states a result due to \cite{abbasi2011improved}.
\begin{theorem}\label{th:confidence_beta}
	Suppose that for all $k \in [K]$, $\norm{\phi_k}_2 \le L$. For all $x > 0$ and $\delta\in(0,1]$,
\begin{align*}
\mathbb{P} \left(\exists t\in \mathbb{N},\,
	\frac{1}{2}\norm{ A^\top S_t }_{(V_t + x I_d)^{-1}}^2 \geq \log\frac{1}{\delta} + \frac{d}{2}\log\!\left(1+\frac{t L^2}{x d} \right)\right)
\leq \delta\, .
\end{align*}
\end{theorem}

\begin{corollary}\label{cor:conc-theta}
	If we ensure that $V_{t_0} \succeq x I_d$ (in the sense of positive definite matrices), then
\begin{align*}
\mathbb{P} \left(\exists t > t_0,\,
	\frac{1}{2}\norm{ \widehat{\theta}_t - \theta_t }_{V_t}^2 \geq 2\log\frac{1}{\delta} + d\log\!\left(1+\frac{t L^2}{x d} \right)\right)
\leq \delta\, .
\end{align*}
\end{corollary}
\begin{proof}
If $V_t \succeq x I_d$ then $2 V_t \succeq V_t + x I_d$ and 
\begin{align*}
	\norm{\widehat{\theta}_t - \theta_t }_{V_t}^2
=
	\norm{ A^\top S_t }_{V_t^{-1}}^2
	&\le 2 \norm{ A^\top S_t }_{(V_t + x I_d)^{-1}}^2
\: .
\end{align*}
\end{proof}

The $2 \log (1/\delta)$ term is fine for some steps of the analysis but not for the stopping rule. For the stopping rule concentration inequality, we need $\log(1/\delta)$.

\begin{corollary}\label{cor:linear_stopping_threshold}
Suppose that $V_{t_0} \succeq x I_d$. Then
\begin{align*}
\mathbb{P} \left(\exists t \geq t_0,\,
	\frac{1}{2}\norm{ \widehat{\theta}_t - \theta_t }_{V_t}^2
	\geq 1 + \log\frac{1}{\delta}
	+ \left(1+\frac{1}{\log(1/\delta)} \right) \frac{d}{2}\log\!\left(1+\frac{t L^2}{x d}\log\frac{1}{\delta} \right)\right)
\leq \delta\, .
\end{align*}
\end{corollary}

\begin{proof}
	Suppose that $V_{t_0} \succeq x I_d$ and let $\gamma(\delta) := \log(1/\delta)^{-1}$. For any $t \geq t_0$,
\begin{align*}
	\norm{ \widehat{\theta}_t - \theta_t }_{V_t}^2
=
	\norm{ A^\top S_t }_{V_t^{-1}}^2
	&\le (1+\gamma(\delta)) \norm{ A^\top S_t }_{(V_t + x\gamma(\delta) I_d)^{-1}}^2
\: .
\end{align*}
Then we conclude by applying Theorem \ref{th:confidence_beta}.
\end{proof}

\subsection{Unstructured concentration}

Let $W_{-1}$ be the negative branch of the Lambert W function and let $\overline{W}(x) = - W_{-1}(-e^{-x})$.
Note that for $x \ge 1$, $x + \log x \le \overline{W}(x) \le x + \log x + \min\{\frac{1}{2}, \frac{1}{\sqrt{x}}\}$.

%

\begin{lemma}\label{lem:unstructured_concentration}
For $t > 1$, with probability $1 - \delta$,
\begin{align*}
	\frac{1}{2} \norm{ \widehat{\mu}_t - \mu }_{D_{N_t}}^2
\le 2 K \overline{W}\left( \frac{1}{2K}\log\frac{e}{\delta} + \frac{1}{2}\log(8eK\log t)\right)
\; .
\end{align*}
\end{lemma}
\begin{proof}
See \cite[][Appendix A, Theorem 4]{degenne2020structure} for that form of the lemma, which is a small reformulation of a result due to \cite{magureanu2014lipschitz}.
\end{proof}

The concentration inequality of Lemma~\ref{lem:unstructured_concentration} is also valid for $\norm{ \tilde{\mu}_t - \mu }_{D_{N_t}}^2$ since the first inequality of Lemma~\ref{lem:estimator_properties} states that $\norm{ \tilde{\mu}_t - \mu }_{D_{N_t}}^2 \le \norm{ \widehat{\mu}_t - \mu }_{D_{N_t}}^2$.


\subsection{Elliptic potential lemmas}
\label{sub:elliptic_lemma}

All lemmas in this section are derived under the following assumption.
\begin{assumption}\label{ass:V_t_large}
	For $t \ge t_0$, $V_t \succeq 2 L^2 I_d$. 
\end{assumption}

In the remainder of the section, we consider $\omega_t \in \Delta^K$, for any time $t > 0$.

\begin{lemma}\label{lem:elliptical_s-1_w}
Under Assumption~\ref{ass:V_t_large}, with probability $1 - \delta$,
\begin{align*}
	\sum_{s=t_0+1}^t \sum_{k=1}^K \omega_s^k \norm{ \phi_k }_{V_{s-1}^{-1}}^2
&\le \sqrt{2 t \log \frac{1}{\delta}} + d \log \left( 1 + \frac{t}{d} \right)
\: .
\end{align*}
\end{lemma}

\begin{proof}
\begin{align*}
	\sum_{s=t_0+1}^t \sum_{k=1}^K \omega_s^k \norm{ \phi_k }_{V_{s-1}^{-1}}^2
	&= \sum_{s=t_0+1}^t \left(\sum_{k=1}^K \omega_s^k \norm{ \phi_k }_{V_{s-1}^{-1}}^2 - \norm{ \phi_{k_s} }_{V_{s-1}^{-1}}^2 \right)
	+ \sum_{s=t_0+1}^t \norm{ \phi_{k_s} }_{V_{s-1}^{-1}}^2
\: .
\end{align*}
The first term is the sum of a martingale difference sequence with bounded increments
\begin{align*}
	\mathbb{E}\left[ \sum_{k=1}^K \omega_s^k \norm{ \phi_k }_{V_{s-1}^{-1}}^2 - \norm{ \phi_{k_s} }_{V_{s-1}^{-1}}^2 \ |\ \mathcal F_{s-1} \right]
&= 0 \: ,
\\
	\abs{ \sum_{k=1}^K \omega_s^k \norm{ \phi_k }_{V_{s-1}^{-1}}^2 - \norm{ \phi_{k_s} }_{V_{s-1}^{-1}}^2  }
&\le 1\: .
\end{align*}
since $V_{s-1} \succeq 2 L^2 I_d$ and $\Vert \phi_k\Vert \le L$. By the Azuma-Hoeffding inequality, with probability $1 - \delta$,
\begin{align*}
	\sum_{s=t_0+1}^t \left(\sum_{k=1}^K \omega_s^k \norm{ \phi_k }_{V_{s-1}^{-1}}^2 - \norm{ \phi_{k_s} }_{V_{s-1}^{-1}}^2 \right)
&\le \sqrt{2 t \log \frac{1}{\delta}}
\: .
\end{align*}
The second term is an elliptic potential, bounded in Lemma~\ref{lem:elliptic_from_t0_with_ks} below.
\end{proof}


\begin{lemma}\label{lem:elliptic_from_t0_with_ks}
Under Assumption~\ref{ass:V_t_large}, for $t>t_0$,
\begin{align*}
	\sum_{s=t_0+1}^t \norm{ \phi_{k_s} }_{V_{s-1}^{-1}}^2
&\le d \log \left( 1 + \frac{t}{d} \right)
\; .
\end{align*}
\end{lemma}
\begin{proof}
Let $V_{s:t}$ denote the design matrix using only rounds from $s$ to $t$. We use Lemma~\ref{lem:elliptical_s-1},
\begin{align*}
	\sum_{s=t_0+1}^t \norm{ \phi_{k_s} }_{V_{s-1}^{-1}}^2
	\le \sum_{s=t_0+1}^t \norm{ \phi_{k_s} }_{(2 L^2 I_d + V_{t_0+1:s-1})^{-1}}^2
\le  d \log \left( 1 + \frac{t}{d} \right)
\: .
\end{align*}
\end{proof}

\begin{lemma}\label{lem:elliptical_s-1}
Under Assumption~\ref{ass:V_t_large}, for $t>t_0$,
\begin{align*}
	\sum_{s=t_0+1}^t \norm{ \phi_{k_s} }_{(V_{t_0+1:s-1} + 2 L^2 I_d)^{-1}}^2
&\le d \log \left( 1 + \frac{t}{d} \right)
\: .
\end{align*}
\end{lemma}
\begin{proof}
By definition of $L$, for all $k \in [K]$, $\phi_k \phi_k^\top \preceq L^2 I_d$. From Lemma~\ref{lem:elliptical} below, we have
\begin{align*}
\sum_{s=t_0+1}^t \norm{ \phi_{k_s} }_{(V_{t_0+1:s-1} + 2 L^2 I_d)^{-1}}^2
&= \sum_{s=t_0+1}^t \norm{ \phi_{k_s} }_{(V_{t_0+1:s-1} + L^2 I_d + L^2 I_d)^{-1}}^2 
\\
&\le \sum_{s=t_0+1}^t \norm{ \phi_{k_s} }_{(V_{t_0+1:s} + L^2 I_d)^{-1}}^2
\\
&\le d \log \left( 1 + \frac{2 t}{d} \right)
\: .
\end{align*}
\end{proof}

A general statement (extracted from \cite{degenne2020gamification} but widely known, see for example \cite{lattimore2020bandit}) is
\begin{lemma}\label{lem:elliptical}
Let $(\omega_t)_{t\geq 1}$ be a sequence in the simplex $\Delta_K$ and $x > 0$. Let $W_t := \sum_{s=1}^t \omega_s$ and $V_{W_t} := \sum_{s=1}^t \sum_{k=1}^K \omega_s^k \phi_k \phi_k^\top$. Then
\begin{align*}
	\sum_{s=1}^t \sum_{k = 1}^K \omega_s^k \norm{ \phi_k }^2_{(V_{W_s} + x I_d)^{-1}}
\le d \log\left(1 +\frac{t L^2}{d \eta}\right)
\: .
\end{align*}
\end{lemma}
\begin{proof}
Define the function $f(W)= \log\det(V_W + x I_d)$ for any $W \in (\mathbb{R}^+)^K$. It is a concave function since the function
$V\mapsto \log\det(V)$ is a concave function over the set of positive definite matrices (see Exercise 21.2 of \cite{lattimore2020bandit}). Its partial derivative with respect to the coordinate $k$ at $W$ is
\[
	\nabla_k f(W) = \norm{ \phi_k }^2_{(V_W + x I_d)^{-1}}\,.
\]
Hence using the concavity of $f$ we have
\begin{align*}
	\sum_{k=1}^K \omega_s^k \norm{ \phi_k }^2_{(V_{W_s} + x I_d)^{-1}}
  = (W_s - W_{s-1})^\top \nabla_a f(W_s)
  \le f(W_s) - f(W_{s-1})\, ,
\end{align*}
which implies that
\begin{align*}
	\sum_{s=1}^t \sum_{k=1}^K \omega_s^k \norm{ \phi_k }^2_{V_{W_s} + x I_d}
  \le f(W_t)-f(W_0)
  = \log\left(\frac{\det(V_{W_t} + x I_d) }{\det(x I_d)}\right)
  \le d \log\left(1 +\frac{t L^2}{d x}\right)\,,
\end{align*}
where for the last inequality we use the inequality of arithmetic and geometric means in combination with $\text{Tr}(V_{W_t}) \le t L^2$\,.
\end{proof}

\begin{lemma}\label{lem:unstructured_elliptical}
Let $C>0$ be a constant. With probability $1 - \delta$,
\begin{align*}
\sum_{s=t_0+1}^t \sum_{k=1}^K \omega_{s-1}^k \min \left\{C,\frac{1}{N_{s-1}^k} \right\}
\le C\sqrt{2 t \log \frac{1}{\delta}} + K (C + 1 + \log t)
\end{align*}
\end{lemma}
\begin{proof}
The first term is due to a martingale argument to bound $\sum_s \left(\sum_k \omega_{s-1}^k \min \left\{C,\frac{1}{N_{s-1}^k} \right\} - \min \left\{C, \frac{1}{N_{s-1}^{k_s}} \right\}\right)$. Then
\begin{align*}
\sum_{s=t_0+1}^t \min \left\{C, \frac{1}{N_{s-1}^{k_s}}\right\}
\le CK + \sum_{k=1}^K \mathbb{I}\left\{N_{t-1}^k > 0\right\} \sum_{j=1}^{N_{t-1}^k} \frac{1}{j}
\le K (C + 1 + \log t)
\: .
\end{align*}

\end{proof}

\subsection{Martingale concentration}
\label{sub:martingale_concentration}

\begin{lemma}\label{lemma:martingale-sampling}
	Let $\mu \in \mathcal M$ (with upper bounds $M$ and $\varepsilon$) and $Z_s(\lambda) := ({\mu}^{k_s} - \lambda^{k_s})^2 - \mathbb{E}_{k\sim\omega_s}[({\mu}^{k} - \lambda^{k})^2]$. For any $\delta'\in(0,1)$,
\begin{align*}
\mathbb{P}\left\{ \exists t\geq 1 : \sup_{\lambda\in\mathcal{M}} \left|\sum_{s=1}^t Z_s(\lambda)\right| > r(t,\delta')\right\} \leq \delta',
\end{align*}
where
\begin{align*}
r(t,\delta') := 2M^2 \sqrt{\frac{t}{2}\left(\log\frac{4t^2}{\delta'} + d\log\frac{6(M + \varepsilon)Lt}{\sqrt{\Gamma(A)}} + K\log\max\{4 \varepsilon t, 1\}\right)} + 2 + 8M \: ,
\end{align*}
$\Gamma(A) := \max_{\omega \in \Delta_K} \sigma_{\min}\left( \sum_{k=1}^K \omega^k \phi_k \phi_k^\top \right)$, and $\sigma_{\min}(M)$ is the minimal eigenvalue of matrix $M$.
\end{lemma}
\begin{proof}

First note that $\omega_s$ is $\mathcal{F}_{s-1}$-measurable. Thus, for any fixed $\lambda$,
\begin{align*}
\mathbb{E}[Z_s(\lambda) | \mathcal{F}_{s-1}] = \mathbb{E}[({\mu}^{k_s} - \lambda^{k_s})^2| \mathcal{F}_{s-1}] - \mathbb{E}_{k\sim\omega_s}[({\mu}^{k} - \lambda^{k})^2] = 0,
\end{align*}
which implies that $\{Z_s\}_{s\geq 1}$ is a martingale difference sequence. Moreover, it is easy to check that $|Z_s(\lambda)| \leq 4M^2$. Unfortunately, we cannot directly use this martingale property to concentrate the desired term since $\lambda$ is adaptively chosen as a function of the whole history up to time $t$. As a solution, we shall use a covering argument on the whole model family $\mathcal{M}$. 

Suppose that we have a finite $\xi$-cover $\bar{\mathcal{M}}_\xi$ of $\mathcal{M}$, i.e., for any $\lambda\in\mathcal{M}$, there exists $\bar{\lambda} \in \bar{\mathcal{M}}_\xi$ such that $\|\lambda - \bar{\lambda}\|_\infty \leq \xi$. For such a couple $(\lambda, \bar{\lambda})$, this implies that, for any $s\geq 1, k\in[K]$,
\begin{align*}
\big| ({\mu}^{k} - \lambda^{k})^2 - ({\mu}^{k} - \bar{\lambda}_k)^2\big| &=  \big|  (\bar{\lambda}_k - \lambda^{k})^2 + 2({\mu}^{k} - \bar{\lambda}_k)(\bar{\lambda}_k - \lambda^{k}) \big|\\ &\leq (\bar{\lambda}_k - \lambda^{k})^2 + 2|{\mu}^{k} - \bar{\lambda}_k | |\bar{\lambda}_k - \lambda^{k}| \leq \xi^2 + 4M\xi.
\end{align*}
Moreover, using this bound in the definition of $Z_s(\lambda)$, 
\begin{align*}
\left|\sum_{s=1}^t Z_s(\lambda) - \sum_{s=1}^t Z_s(\bar{\lambda})\right| \leq 2t\xi^2 + 8tM\xi.
\end{align*}
Let $h(t)$ be some function to be specified later. With some abuse of notation w.r.t. the derivation above, we shall instantiate a different $\xi_t$-cover for each time step $t$. Then
\begin{align*}
&\mathbb{P}\left\{ \exists t\geq 1 : \sup_{\lambda\in\mathcal{M}} \left|\sum_{s=1}^t Z_s(\lambda)\right| > h(t) \right\} = \mathbb{P}\left\{ \exists t\geq 1 : \sup_{\lambda\in\mathcal{M}}\inf_{\bar{\lambda}\in\bar{\mathcal{M}}_{\xi_t}} \left|\sum_{s=1}^t Z_s(\lambda)\pm Z_s(\bar{\lambda})\right| > h(t) \right\}\\ & \qquad\qquad \stackrel{(a)}{\leq}  \mathbb{P}\left\{ \exists t\geq 1 : \sup_{\bar{\lambda}\in\bar{\mathcal{M}}_{\xi_t}} \left|\sum_{s=1}^t Z_s(\bar{\lambda})\right| + \sup_{\lambda\in\mathcal{M}}\inf_{\bar{\lambda}\in\bar{\mathcal{M}}_{\xi_t}} \left|\sum_{s=1}^t Z_s(\lambda)- Z_s(\bar{\lambda})\right| > h(t) \right\}\\ & \qquad\qquad \stackrel{(b)}{\leq}   \mathbb{P}\left\{ \exists t\geq 1, \bar{\lambda}\in\bar{\mathcal{M}}_{\xi_t} : \left|\sum_{s=1}^t Z_s(\bar{\lambda})\right| > h(t) - 2t{\xi_t}^2 - 8tM{\xi_t} \right\}\\ & \qquad\qquad \stackrel{(c)}{\leq}  \sum_{t=1}^\infty \sum_{\bar{\lambda}\in\bar{\mathcal{M}}_{\xi_t}} \mathbb{P}\left\{\left|\sum_{s=1}^t Z_s(\bar{\lambda})\right| > h(t) - 2t{\xi_t}^2 - 8tM{\xi_t} \right\},
\end{align*}
where (a) follows by the triangle inequality, (b) from the property of the cover, and (c) from the union bound and the fact that the cover is finite. Let $\delta'_t \in (0,1)$. If we choose $h(t) := 2M^2 \sqrt{\frac{t}{2}\log(2/\delta'_t)} + 2t{\xi_t}^2 + 8tM{\xi_t}$, using Azuma's inequality, each probability in the sum above is bounded by $\delta'_t$. Hence, choosing $\delta'_t := \frac{\delta'}{2|\bar{\mathcal{M}}_{\xi_t}|t^2}$,
\begin{align*}
\sum_{t=1}^\infty \sum_{\bar{\lambda}\in\bar{\mathcal{M}}_{\xi_t}} \mathbb{P}\left\{\left|\sum_{s=1}^t Z_s(\bar{\lambda})\right| > h(t) - 2t{\xi_t}^2 - 8tM{\xi_t} \right\} \leq \sum_{t=1}^\infty \frac{\delta'}{2t^2}\leq \delta',
\end{align*}
where the last inequality can be verified easily. Therefore, putting everything together, we proved that
\begin{align*}
\mathbb{P}\left\{ \exists t\geq 1 : \sup_{\lambda\in\mathcal{M}} \left|\sum_{s=1}^t Z_s(\lambda)\right| > 2M^2 \sqrt{\frac{t}{2}\log\frac{4|\bar{\mathcal{M}}_{\xi_t}|t^2}{\delta'}} + 2t{\xi_t}^2 + 8tM{\xi_t}\right\} \leq \delta'.
\end{align*}
It only remains to build the cover, compute its size, and specify the value of  $\xi_t$. Recall that each model $\lambda\in\mathcal{M}$ can be written as $\lambda = A\theta' + \eta'$, where $\|\eta'\|_\infty \leq \varepsilon$ and $\|\lambda\|_\infty \leq M$. Using Lemma \ref{lemma:norm-linear-part} below, we have that $\|\theta'\|_2 \leq \bar{B} := \frac{M + \varepsilon}{\sqrt{\Gamma(A)}}$.
Then, we can build two separate covers for the linear and deviation parts. Specifically, we build a $\xi_t/(2L)$-cover $\bar{\mathcal{M}}_t^{\mathrm{lin}}$ in $\ell_2$-norm for the linear part and a $\xi_t/2$-cover $\bar{\mathcal{M}}_t^{\mathrm{dev}}$ in $\ell_\infty$-norm for the deviation part.
Then, we take the full cover as the (finite) set $\bar{\mathcal{M}}_t := \{\bar{\lambda} = A\bar{\theta} + \bar{\eta} : \bar{\theta} \in \bar{\mathcal{M}}_t^{\mathrm{lin}}, \bar{\eta} \in \bar{\mathcal{M}}_t^{\mathrm{dev}}\}$. With this choice, we have that, for any $\lambda = A\theta' + \eta'$, there exits $\bar{\lambda} \in \bar{\mathcal{M}}_t$ such that
\begin{align*}
\|\lambda-\bar{\lambda}\|_\infty  = \|A\theta' + \eta'-A\bar{\theta} - \bar{\eta}\|_\infty \leq L\|\theta'-\bar{\theta}\|_2 + \|\eta'-\bar{\eta}\|_\infty \leq \xi_t.
\end{align*}
Let us compute the size of the cover $\bar{\mathcal{M}}_t$. It is easy to see that this is $|\bar{\mathcal{M}}_t| = |\bar{\mathcal{M}}_t^{\mathrm{lin}}||\bar{\mathcal{M}}_t^{\mathrm{dev}}|$.
For the linear one, it is known that the $\xi_t/(2L)$-covering number (in $\ell_2$-norm) of a ball in $\mathbb{R}^d$ with radius $\bar{B}$ is at most $(6L\bar{B}/\xi_t)^d$. For the deviation, we can have a $\xi_t/2$ cover in $\ell_\infty$-norm with at most $\max\{(4 \varepsilon/\xi_t)^K, 1\}$ points, where the maximum is to deal with too small values of $\varepsilon$ (\eg $\varepsilon=0$).
Then, the final cover has size at most $|\bar{\mathcal{M}}_{t}| \leq (6\bar{B}L/\xi_t)^d\max\{(4 \varepsilon/\xi_t)^K, 1\}$.
Setting $\xi_t = 1/t$, we get the desired bound. 
\end{proof}

\section{$\delta$-correctness and sample complexity analysis}
\label{app:sample_complexity_upper_bound}


\subsection{Correctness}
\label{sub:correctness}

We prove Lemma~\ref{lem:delta-correct} in the main paper, restated below.
\begin{lemma*}
Let $W_{-1}$ be the negative branch of the Lambert W function and let $\overline{W}(x) = - W_{-1}(-e^{-x}) \approx x + \log x$. For $\delta\in(0,1)$, define
\begin{align}
\beta_{t,\delta}^{\mathrm{uns}} &:= 2 K \overline{W}\left( \frac{1}{2K}\log\frac{2e}{\delta} + \frac{1}{2}\log(8eK\log t)\right),
\\
\beta_{t,\delta}^{\mathrm{lin}}
&:= \frac{1}{2}\left( 4\sqrt{t}\varepsilon
	+ \sqrt{2}\sqrt{1 + \log\frac{1}{\delta} + \left(1+\frac{1}{\log(1/\delta)} \right) \frac{d}{2}\log\!\left(1+\frac{t}{2d}\log\frac{1}{\delta} \right)} \right)^2
\: .
\end{align}
Then, for the choice $\beta_{t,\delta} := \min\{\beta_{t,\delta}^{\mathrm{uns}}, \beta_{t,\delta}^{\mathrm{lin}}\}$, \algo is $\delta$-correct.
\end{lemma*}

\begin{proof}
	$\delta$-correctness is composed of two properties: stopping in a finite time with probability one and verifying, for all instances $\mu \in \mathcal{M}$, $\mathbb{P}(\hat{S}_m \not\subseteq S^\star(\mu)) \le \delta$. The fact that the stopping time is finite almost surely is a consequence of the sample complexity bound (see further down in this section). We now prove the bound on the probability of error in identification.

We first relate the event that the algorithm does not return a correct answer to a large deviation, by writing that for the algorithm to make a mistake, there must be a time at which the stopping condition is met and $\tilde{\mu}_t$ is in the alternative to $\mu$:
\begin{align*}
	\mathbb{P}(\hat{S}_m \not\subseteq S^\star(\mu))
	&\le \mathbb{P}\left(\exists t \in \mathbb{N}, \ \inf_{\lambda\in\Lambda_m(\tilde{\mu}_{t})}\norm{\tilde{\mu}_{t} - \lambda }_{D_{N_{t}}}^2 > 2\beta_{t,\delta} \ \land \ \tilde{\mu}_t \in \Lambda_m(\mu) \right) \: .
\end{align*}
If the two conditions of the right-hand side happen, then $\mu \in \Lambda_m(\tilde{\mu}_t)$ and we get
\begin{align*}
	\mathbb{P}(\hat{S}_m \not\subseteq S^\star(\mu))
	&\le \mathbb{P}\left(\exists t \in \mathbb{N}, \ \norm{\tilde{\mu}_{t} - \mu }_{D_{N_{t}}}^2 > 2\beta_{t,\delta} \right) \: .
\end{align*}
It then suffices to prove that we have both
\begin{align}
	\mathbb{P}\left(\exists t \in \mathbb{N}, \ \frac{1}{2}\norm{\tilde{\mu}_{t} - \mu }_{D_{N_{t}}}^2 > \beta_{t,\delta}^{\mathrm{lin}} \right) \le \delta/2 \: ,
\label{eq:correctness_lin}
\end{align}
\begin{align}
\text{and} \quad 
\mathbb{P}\left(\exists t \in \mathbb{N}, \ \frac{1}{2}\norm{\tilde{\mu}_{t} - \mu }_{D_{N_{t}}}^2 > \beta_{t,\delta}^{\mathrm{uns}} \right) \le \delta/2
\label{eq:correctness_uns}
\: .
\end{align}

	The result for \eqref{eq:correctness_uns} is Lemma~\ref{lem:unstructured_concentration} (and the remark below that lemma stating that it applies to $\tilde{\mu}_t$). We now prove the concentration inequality using the linear term \eqref{eq:correctness_lin}.


let $\tilde{\theta}_{t,\varepsilon}$ and $\tilde{\eta}_{t,\varepsilon}$ be parameters for $\tilde{\mu}_t$ with $\Vert \tilde{\eta}_{t,\varepsilon} \Vert \le \varepsilon$, which exist since $\tilde{\mu}_t \in \mathcal M$. On the other hand, let $\tilde{\theta}_{t}$ and $\tilde{\eta}_{t}$ be the orthogonal parameters of $\tilde{\mu}_t$ with respect to $N_t$.
\begin{align*}
\Vert \tilde{\mu} - \mu \Vert_{D_{N_t}}
&= \Vert A(\tilde{\theta}_{t,\varepsilon} - \theta) + \tilde{\eta}_{t,\varepsilon} - \eta \Vert_{D_{N_t}}
\\
&\le \Vert A(\tilde{\theta}_{t,\varepsilon} - \theta) \Vert_{D_{N_t}} + \Vert \tilde{\eta}_{t,\varepsilon} - \eta \Vert_{D_{N_t}}
\\
&= \Vert \tilde{\theta}_{t,\varepsilon} - \theta \Vert_{V_t} + \Vert \tilde{\eta}_{t,\varepsilon} - \eta \Vert_{D_{N_t}}
\\
&\le \Vert \tilde{\theta}_{t,\varepsilon} - \tilde{\theta}_t \Vert_{V_t} + \Vert \tilde{\theta}_t - \theta_t \Vert_{V_t} + \Vert \theta_t - \theta \Vert_{V_t} + \Vert \tilde{\eta}_{t,\varepsilon} - \eta \Vert_{D_{N_t}}
\: .
\end{align*}
Lemma~\ref{lem:different_param_linear} bounds the first and third terms by $\sqrt{t}\varepsilon$. The last term is bounded by $\sqrt{t}\Vert \tilde{\eta}_{t,\varepsilon} - \eta \Vert_\infty \le 2 \sqrt{t} \varepsilon$ since both vectors have $\ell_\infty$ norm bounded by $\varepsilon$.


Finally
\begin{align*}
\mathbb{P}\left(\exists t \in \mathbb{N}, \ \frac{1}{2}\norm{ \tilde{\mu}_{t} - \mu }_{D_{N_{t}}}^2 > \beta_{t,\delta}^{\mathrm{lin}} \right)
\le \mathbb{P}\left(\exists t \in \mathbb{N}, \ \frac{1}{2}\norm{ \widehat{\theta}_t - \theta_t }^2_{V_t} > \frac{1}{2}(\sqrt{2 \beta_{t,\delta}^{\mathrm{lin}}} - 4\sqrt{t}\varepsilon)^2\right)
\le \delta/2
\end{align*}
by Corollary~\ref{cor:linear_stopping_threshold}.

\end{proof}

\subsection{Restriction to a good event}
\label{sub:restriction_to_a_good_event}

\paragraph{Assumption.} We start by pulling arms deterministically until $t_0$, such that $V_{t_0} \ge 2L^2 I_d$. See paragraph ``Initialization phase'' in Subsection~\ref{sub:algorithm} in the main paper.

\paragraph{Definition of the good event.}

For $t \ge t_0$ and $k \in [K]$, define 
\begin{align*}
\alpha_t^{\mathrm{lin}}
&:= \log(\numel t^2) + d\log\!\left(1+\frac{t}{2d} \right) \; ,
\
\alpha_t^{\mathrm{uns}}
:= 2 K \overline{W}\left( \frac{1}{2K}\log(2e\numel t^3) + \frac{1}{2}\log(8eK\log t)\right)
\: .
\end{align*}

Consider the following events. Each of these holds with probability at least $1 - \frac{1}{\numel t^2}$ by the indicated concentration result.
\begin{enumerate}
\item Concentration of the projected linear part (Corollary \ref{cor:conc-theta})
\begin{align*}
	\mathcal E_t^1 := \left\{ \forall s \ge t_0
	: \frac{1}{2}\norm{ \tilde{\theta}_s - \theta_s }_{V_s}^2 
	\leq \alpha_t^{\mathrm{lin}}\right\} \: ,
\end{align*}
\item Unstructured concentration of the projected estimator (Lemma \ref{lem:unstructured_concentration})
\begin{align*}
	\mathcal E_t^2 := \left\{ \forall s \leq t : \frac{1}{2} \norm{ \tilde{\mu}_s - \mu }_{D_{N_s}}^2 
\leq \alpha_t^{\mathrm{uns}} \right\} \: ,
\end{align*}

\item Martingale concentration for sampling (Lemma \ref{lemma:martingale-sampling})
\begin{align*}
\mathcal E_t^3 := \left\{ \sup_{\lambda\in\mathcal{M}} \left|\sum_{s=1}^t Z_s(\lambda)\right| \leq r\left(t\right) \right\} \: ,
\end{align*}
where $r(t)$ is obtained by setting $\delta' = \frac{1}{\numel t^2}$ in $r(t,\delta')$ in Lemma \ref{lemma:martingale-sampling}, which yields
\begin{align*}
r(t) = 2M^2 \sqrt{\frac{t}{2}\left(\log(4 \times \numel t^4) + d\log\frac{6(M + \varepsilon)Lt}{\sqrt{\Gamma(A)}} + K\log\max\{4\varepsilon t, 1\}\right)} + 2 + 8M.
\end{align*}

\item Elliptical potential with sampling (Lemma \ref{lem:elliptical_s-1_w})
\begin{align*}
\mathcal E_t^4
	:= \left\{ \sum_{s=t_0+1}^t \sum_{k=1}^K \omega_s^k \norm{ \phi_k }_{V_{s-1}^{-1}}^2
\le \sqrt{2 t \log(\numel t^2)} + d \log \left( 1 + \frac{t}{d} \right)
\right\} \: ,
\end{align*}

\item Elliptical potential with sampling for the unstructured bound (Lemma~\ref{lem:unstructured_elliptical})
\begin{align*}
\mathcal E_t^5
:= \left\{ \sum_{s=t_0+1}^t \sum_{k=1}^K \omega_s^k \min\left\{ 4 M^2, \frac{2 \alpha_t^{\mathrm{uns}}}{N_{s-1}^k} \right\}
\le 4 M^2 \sqrt{2 t \log (\numel t^2)} + 4 M^2 K + 2 K \alpha_t^{\mathrm{uns}}\log (e t)
\right\}
\: .
\end{align*}

\end{enumerate}

Then, we define the ``good'' event $\mathcal E_t := \bigcap_{i=1}^{\numel} \mathcal{E}_t^i$.

\begin{lemma}\label{lemma:bound_good_event}
For all $t \ge 1$,
$
\mathbb{P}(\mathcal E_t^c) \le 1/t^2 \: .
$
\end{lemma}
\begin{proof}
Apply an union bound by noting that each event $\mathcal{E}_t^i$ fails with probability at most $1/(\numel t^2)$.
\end{proof}

\begin{lemma}\label{lemma:ccl_sample_complexity}
Let $T_0(\delta) \in \mathbb{N}$ be such that for $t \ge T_0(\delta)$, $\mathcal E_t \subseteq \{\tau_\delta \le t\}$. Then 
$
\mathbb{E}[\tau_\delta]
\le T_0(\delta) + 2
\: .
$
\end{lemma}
\begin{proof}
Successively using the definition of $T_0(\delta)$ and Lemma~\ref{lemma:bound_good_event}:
\begin{align*}
\mathbb{E}[\tau_\delta]
= \sum_{t=0}^{+\infty} \mathbb{P}(\tau_\delta > t)
\le T_0(\delta) + \sum_{t=T_0(\delta)}^{+\infty} \mathbb{P}(\mathcal E_t^c)
\leq T_0(\delta) + \sum_{t=1}^{+\infty} \frac{1}{t^2}
\le T_0(\delta) + 2
\: .
\end{align*}
\end{proof}

\paragraph{Consequences of the good event.}

\begin{lemma}\label{lemma:bound_ckt}
For $t \ge t_0$ and $k \in [K]$, define 
\begin{align*}
c_t^k
&:= \min\left\{8 \etaboundsq + 4 \alpha_{\correction{t^2}}^{\mathrm{lin}}  \norm{ \phi_k }_{V_{t}^{-1}}^2, \frac{2\alpha_{\correction{t^2}}^{\mathrm{uns}}}{N_t^k}, 4M^2 \right\}
\: ,
\end{align*}
where we use the convention that $2 \alpha^{\mathrm{uns}}/N_t^k = + \infty$ if $N_t^k = 0$.
Then under $\mathcal E_t$, for all $s \in \{\max\{t_0, \sqrt{t}\}, \ldots, t\}$ and $k \in [K]$, $(\tilde{\mu}_s^k - \mu^k)^2 \le c_s^k$.
\end{lemma}

\begin{proof}
We know that 
$
\frac{1}{2}\norm{ \tilde{\theta}_s - \theta_s }_{V_t}^2 
\leq \alpha_t^{\mathrm{lin}}
$
holds for all $s\geq t_0$ by definition of $\mathcal E_t^1$.
\correction{For $s \ge \max\{t_0, \sqrt{t}\}$ we also have $\alpha_{s^2}^{\mathrm{lin}} \ge \alpha_t^{\mathrm{lin}}$, hence 
$
\frac{1}{2}\norm{ \tilde{\theta}_s - \theta_s }_{V_t}^2 
\leq \alpha_{s^2}^{\mathrm{lin}}
$
.}
	Using first $(a+b)^2\le 2a^2+2b^2$ then the Cauchy-Schwarz inequality on $(V_s^{-1/2}\phi_k)^\top(V_s^{1/2}(\tilde{\theta}_s - \theta_s))$ \correction{and Lemma~\ref{lem:bound-eta-t}},
\begin{align*}
(\tilde{\mu}_s^k - \mu^k)^2
\le 2(\phi_k^\top (\tilde{\theta}_s - \theta_s))^2 + 2 (\tilde{\eta}_s^k - \eta_s^k)^2
&\le 2 \norm{ \phi_k }_{V_s^{-1}}^2 \norm{ \tilde{\theta}_s - \theta_s }_{V_s}^2 + 8 \etaboundsq
\\
&\le 8 \etaboundsq + 4 \alpha_{s^2}^{\mathrm{lin}} \norm{ \phi_k }_{V_s^{-1}}^2
\: .
\end{align*}

Moreover by definition of $\mathcal E_t^2$, for all $s \leq t$, 
$
\frac{1}{2} \norm{ \tilde{\mu}_s - \mu }_{D_{N_s}}^2
\le \alpha_t^{\mathrm{uns}}
\: .
$
\correction{For $s \ge \max\{t_0, \sqrt{t}\}$ we have $\alpha_{s^2}^{\mathrm{uns}} \ge \alpha_t^{\mathrm{uns}}$, hence 
$
\frac{1}{2} \norm{ \tilde{\mu}_s - \mu }_{D_{N_s}}^2
\le \alpha_{s^2}^{\mathrm{uns}}
$
.}
Therefore,
\begin{align*}
(\tilde{\mu}_s^k - \mu^k)^2
= (e_k^\top(\tilde{\mu}_s - \mu)) ^2
\leq \norm{e_k}_{D_{N_s}^{-1}}^2\norm{\mu - \tilde{\mu}_s}_{D_{N_s}}^2
\leq \frac{2\alpha_{s^2}^{\mathrm{uns}}}{N_s^k}.
\end{align*}
Finally, $(\tilde{\mu}_s^k - \mu^k)^2 \le \norm{ \tilde{\mu}_s - \mu}^2_\infty \le 4 M^2$.
\end{proof}

\begin{lemma}\label{lem:def-ft}
For all $t \geq 1$, under the good event $\mathcal{E}_t$,
\begin{align*}
	\forall s \in \{ t_0, t_0+1,\dots, t\} : \norm{ \tilde{\mu}_s - \mu }_{D_{N_s}}^2 
\leq f(t) := 2\min\{\alpha_t^{\mathrm{uns}}, \alpha_t^{\mathrm{lin}} + 2t \etaboundsq\}.
\end{align*}
\end{lemma}
\begin{proof}
	That for all $s \leq t$, $\norm{ \tilde{\mu}_s - \mu }_{D_{N_s}}^2 
\leq 2\alpha_t^{\mathrm{uns}}$ directly follows from the definition of $\mathcal{E}_t^2$. To see the second inequality, we first decompose the norm on the lefthand-side into its linear and deviation components
\begin{align*}
	\norm{ \tilde{\mu}_s - \mu }_{D_{N_s}}^2
	= \norm{ \tilde{\theta}_s - \theta_s }^2_{V_s} + \norm{ R_{N_s} D_{N_s}^{1/2}\tilde{\eta}_s - R_{N_s} D_{N_s}^{1/2} \eta_s }^2 \: .
\end{align*}
	The deviation part can be bounded by $4t \etaboundsq$ for all $s\leq t$ using \correction{Lemma~\ref{lem:bound-eta-t}}. The linear part can be bounded by $2\alpha_t^{\mathrm{lin}}$ for all $s\geq t_0$ by the definition of $\mathcal{E}_t^1$.
\end{proof}

\subsection{Analysis under a good event}
\label{sub:analysis_under_a_good_event}

Fix any time step $t\geq t_0$. Suppose that the good event $\mathcal{E}_t$ of Section \ref{sub:restriction_to_a_good_event} holds and the algorithm does not stop at time $t$. We proceed in different steps.

\paragraph{Stopping rule analysis.}

\begin{theorem}\label{th:stopping_rule}
	If the algorithm does not stop at time $t$ then under $\mathcal{E}_t$, using stopping threshold $\beta_{t,\delta}$ as defined in Lemma~\ref{lem:delta-correct} in the main paper,
\begin{align*}
2\beta_{t,\delta}
\ge 
	\: \inf_{\lambda\in\Lambda_m(\mu)} \norm{{\mu} - \lambda }_{D_{W_{t}}}^2 - h_\delta(t) - r(t) \: .
\end{align*}
where
\begin{itemize}[nosep]
	\item $h_\delta(t) = \sqrt{8\beta_{t,\delta}f(t)} + f(t)$, with $f(t)$ a bound on $\Vert \mu - \tilde{\mu}_t\Vert_{D_{N_t}}^2$ (see Lemma \ref{lem:def-ft}) \: ,
	\item $r(t) = 2M^2 \sqrt{\frac{t}{2}\left(\log(4 \times \numel t^4) + d\log\frac{6(M + \varepsilon)Lt}{\sqrt{\Gamma(A)}} + K\log\max\{4 \varepsilon t, 1\}\right)} + 2 + 8M$ \: ,
\end{itemize}
	and $W_t := \sum_{s\leq 1}^t \omega_s$ is the sum over time of the weight vectors played by the learner.
\end{theorem}
The proof of this theorem is detailed in Steps $1$ to $3$ below.

\paragraph{Step $1$. From $\Lambda_m(\tilde{\mu}_{t})$ to $\Lambda_m({\mu})$.}

\begin{lemma}\label{lem:change_alternative_set}
For all $\mu, \mu' \in \mathcal M$, for any non-negative function $f: \mathbb{R}^K\times \mathbb{R}^K \to \mathbb{R}$ with $f(x,x) = 0$,
\begin{align*}
\inf_{\lambda \in \Lambda_m(\mu)} f(\mu, \lambda) \ge \inf_{\lambda \in \Lambda_m(\mu')} f(\mu, \lambda)
\: .
\end{align*}
\end{lemma}
\begin{proof}
Either $\Lambda_m(\mu) = \Lambda_m(\mu')$ and the two expressions are equal, or $\Lambda_m(\mu) \ne \Lambda_m(\mu')$. In the second case, $\mu \in \Lambda_m(\mu')$. The right-hand side is then equal to zero, which is lower than the left-hand side since $f$ is non-negative.
\end{proof}

Since the algorithm does not stop at time $t$, from the stopping rule
\begin{align*}
	2\beta_{t,\delta} \geq \inf_{\lambda \in \Lambda_m(\tilde{\mu}_{t})}\norm{\tilde{\mu}_{t} - \lambda}_{D_{N_{t}}}^2 \: ,
\end{align*}
where $\Lambda_m(\tilde{\mu}_{t})$ is the set of alternative models to $\tilde{\mu}_t$.
We change the alternative set over which the minimization is performed using Lemma~\ref{lem:change_alternative_set}:
\begin{align}\label{eq:step-1-final}
	2\beta_{t,\delta} \geq \inf_{\lambda\in\Lambda_m(\tilde{\mu}_{t})} \norm{\tilde{\mu}_{t} - \lambda}_{D_{N_{t}}}^2 \geq \inf_{\lambda\in\Lambda_m({\mu})} \norm{\tilde{\mu}_{t} - \lambda }_{D_{N_{t}}}^2 \: .
\end{align}

\paragraph{Step $2$. From $\tilde{\mu}_t$ to $\mu$.}

The next step is to replace the estimated mean $\tilde{\mu}_t$ in the norm with the true mean $\mu$. For all $\lambda \in \mathcal{M}$, using the triangle inequality,
\begin{align*}
	\norm{ \tilde{\mu}_t - \lambda }_{D_{N_t}}
	\ge  \norm{ \mu - \lambda }_{D_{N_t}} - \norm{ \tilde{\mu}_t - \mu }_{D_{N_t}}
	\ge \norm{ \mu - \lambda }_{D_{N_t}} - \sqrt{f(t)} \: ,
\end{align*}
where the last inequality uses Lemma~\ref{lem:def-ft} to concentrate $\norm{ \tilde{\mu}_t - \mu }_{D_{N_t}}^2$. Using this for the specific choice of $\lambda_t \in \argmin_{\lambda\in\Lambda_m({\mu})}\norm{\tilde{\mu}_{t} - \lambda }_{D_{N_{t}}}^2$ in combination with \eqref{eq:step-1-final}, we obtain
\begin{align}
&\left(\sqrt{2\beta_{t,\delta}} + \sqrt{f(t)} \right)^2
	\ge \norm{ \mu - \lambda_t }_{D_{N_t}}^2 
\quad \Rightarrow \quad
&2 \beta_{t, \delta}
	\ge \inf_{\lambda\in\Lambda_m({\mu})}\norm{{\mu} - \lambda }_{D_{N_{t}}}^2 - h_\delta(t)
\: ,\label{eq:step-2-final}
\end{align}
where $h_\delta(t) := \sqrt{8\beta_{t,\delta}f(t)} + f(t)$ is a sub-linear function of both $t$ and $\log(1/\delta)$.

\paragraph{Step $3$. From $N_t$ to $W_t$.}

We now show that it is possible to replace $N_t$ with $W_t := \sum_{s=1}^t \omega_s$ in the norm at the price of subtracting another low-order term. Let $Z_s(\lambda) := ({\mu}^{k_s} - \lambda^{k_s})^2 - \mathbb{E}_{k\sim\omega_s}[({\mu}^{k} - \lambda^{k})^2]$.
Note that $\norm{{\mu} - \lambda }_{D_{N_{t}}}^2 = \sum_{s=1}^t (\mu^{k_s}-\lambda^{k_s})^2$
and $\norm{{\mu} - \lambda }_{D_{W_{t}}}^2 = \sum_{s=1}^t \norm{\mu - \lambda}_{D_{\omega_s}}^2 = \sum_{s=1}^t \mathbb{E}_{k\sim\omega_s}[({\mu}^{k} - \lambda^{k})^2]$.
Therefore, from \eqref{eq:step-2-final},
\begin{align*}
2\beta_{t,\delta}
	&\geq \inf_{\lambda\in\Lambda_m({\mu})}\left(\norm{\mu - \lambda}_{D_{N_{t}}}^2 - \norm{\mu - \lambda }_{D_{W_{t}}}^2 + \norm{\mu - \lambda }_{D_{W_{t}}}^2\right) - h_\delta(t)\\
	&= \inf_{\lambda\in\Lambda_m({\mu})}\left( \norm{\mu - \lambda }_{D_{W_{t}}}^2 + \sum_{s=1}^t Z_s(\lambda)\right) - h_\delta(t)\\
	&\geq \inf_{\lambda\in\Lambda_m(\mu)} \norm{\mu - \lambda }_{D_{W_{t}}}^2 - \sup_{\lambda\in\mathcal{M}} \abs{\sum_{s=1}^t Z_s(\lambda)} - h_\delta(t)\: .
\end{align*} 
Using the good event $\mathcal{E}_t^3$, we can finally write
\begin{align}\label{eq:step-3-final}
	2\beta_{t,\delta} \geq \inf_{\lambda\in\Lambda_m({\mu})}\norm{\mu - \lambda }_{D_{W_{t}}}^2 - h_\delta(t) - r(t) \: ,
\end{align}
which ends proving Theorem~\ref{th:stopping_rule}.

\paragraph{Sampling rule analysis.} Let $H_\mu = \sup_{\omega \in \Delta} \inf_{\lambda \in \Lambda_m(\mu)} \frac{1}{2}\norm{\mu - \lambda }_{D_\omega}^2$ (the inverse complexity at $\mu$).
In the first part of the sampling rule analysis, we introduce the optimistic estimates $g_t(\omega)$ mentioned in Algorithm~\ref{alg:mislingame-topm} in the main paper, which will be used by the learner for $\omega_t$.

\begin{theorem}\label{th:sampling_rule}
Let $(\tilde{\mu}_s)_{s \le t} \in \mathcal M^{[t]}$ be estimates such that under $\mathcal E_t$, we have a bound $c_s^k$ on $(\tilde{\mu}_s^k - \mu^k)^2$ for all $k \in [K]$ and $s \in [t]$.
Then define the \emph{optimistic estimate}
\begin{align*}
	g_s(\omega)
	:= \sum_{k=1}^K \omega^k \left(\abs{\tilde{\mu}_{s-1}^k - \lambda_s^k} + \sqrt{c_{s-1}^k} \right)^2
\quad\text{ where } \lambda_s
	& := \argmin_{\lambda\in\Lambda_m(\tilde{\mu}_{s-1})} \norm{ \tilde{\mu}_{s-1} - \lambda }_{D_{\omega_s}}^2
\: .
\end{align*}
Under $\mathcal E_t$,
\begin{align*}
	\inf_{\lambda\in\Lambda_m(\mu)} \norm{\mu - \lambda }_{D_{W_{t}}}^2
	&\ge \sum_{s=t_0+1}^t g_s(\omega_s) - 4C_t - 4 \sqrt{2 t C_t H_\mu}
\: ,
\end{align*}
with $C_t := \sum_{s=t_0+1}^t \sum_{k=1}^K \omega_s^k c_{s-1}^k$.
\end{theorem}

The proof of this theorem is detailed in the Steps $4$ to $7$ below. Once this result is established, we will use the regret property of the learner to exhibit the final bound (Steps $8$ to $10$).

\paragraph{Step $4$. From $\Lambda_m({\mu})$ back to $\Lambda_m(\tilde{\mu}_{s-1})$ for $s\in[t]$.}

We now start moving from $\inf_{\lambda\in\Lambda_m(\mu)} \norm{\mu - \lambda}_{D_{W_{t}}}^2$ to the actual gain fed into the online learner at time $t$. We first need to go back to the estimated set of alternative models at each time $s=0,\dots,t-1$. We have
\begin{align}
	\inf_{\lambda\in\Lambda_m(\mu)} \norm{\mu - \lambda }_{D_{W_{t}}}^2
	= \inf_{\lambda\in\Lambda_m(\mu)}\sum_{s=1}^t\norm{\mu - \lambda }_{D_{\omega_{s}}}^2
	&\geq \sum_{s=1}^t \inf_{\lambda\in\Lambda_m(\mu)} \norm{\mu - \lambda }_{D_{\omega_{s}}}^2\\
	&\geq \sum_{s=1}^t \inf_{\lambda\in\Lambda_m(\tilde{\mu}_{s-1})} \norm{\mu - \lambda }_{D_{\omega_{s}}}^2 \: , \label{eq:from_Lambda_mu_to_Lambda_hat_mu}
\end{align}
where the first inequality follows by the concavity of the infimum, and the second one is an application of Lemma~\ref{lem:change_alternative_set}. 

\paragraph{Step $5$. Drop the first rounds.} The first $t_0$ rounds are dedicated to making sure that $V_t$ is sufficiently large (for the partial order on positive definite matrices). \correction{Also, our upper bounds on the deviation of $\tilde{\mu}_t^k$ from $\mu^k$ are valid from $\max\{t_0, \sqrt{t}\}$. We define $t_0'(t) = \max\{t_0, \sqrt{t}\}$}. We drop the corresponding nonnegative terms from the sum to keep only the rounds for which $t$ is large enough:
\begin{align*}
	\sum_{s=1}^t \inf_{\lambda\in\Lambda_m(\tilde{\mu}_{s-1})} \norm{\mu - \lambda }_{D_{\omega_{s}}}^2
	&\ge \sum_{s=\correction{t_0'(t)}+1}^t \inf_{\lambda\in\Lambda_m(\tilde{\mu}_{s-1})} \norm{\mu - \lambda }_{D_{\omega_{s}}}^2
\: .
\end{align*}

\paragraph{Step $6$. From $\mu$ back to $\tilde{\mu}_{s-1}$ for $s\in[t]$.}

We can now use the concentration of $\tilde{\mu}_{s-1}$ to replace $\mu$ in all terms $\norm{\mu - \lambda }_{D_{\omega_{s}}}^2$ for $s\in[t]$.
Let $\lambda^\mu_s := \argmin_{\lambda\in\Lambda_m(\tilde{\mu}_{s-1})} \norm{\mu - \lambda }_{D_{\omega_{s}}}^2$. Using first the triangle inequality, then the inequality $\Vert a - b \Vert \ge \Vert a \Vert - \Vert b \Vert$ for an $\ell_2$ norm in dimension $t - \correction{t_0'(t)}$,
\begin{align*}
\sqrt{\sum_{s=\correction{t_0'(t)}+1}^t \norm{\mu - \lambda^\mu_s }_{D_{\omega_{s}}}^2}
&\ge \sqrt{\sum_{s=\correction{t_0'(t)}+1}^t \left(\norm{\tilde{\mu}_{s-1} - \lambda^\mu_s }_{D_{\omega_{s}}} - \norm{ \mu - \tilde{\mu}_{s-1} }_{D_{\omega_s}} \right)^2}
\\
&\ge \sqrt{\sum_{s=\correction{t_0'(t)}+1}^t \norm{\tilde{\mu}_{s-1} - \lambda^\mu_s }_{D_{\omega_{s}}}^2}
	- \sqrt{\sum_{s=\correction{t_0'(t)}+1}^t\norm{ \mu - \tilde{\mu}_{s-1} }_{D_{\omega_s}}^2}
\: .
\end{align*}
We now remark that $\sum_{s=\correction{t_0'(t)}+1}^t \norm{ \tilde{\mu}_{s-1} - \mu }_{D_{w_s}}^2 \le C_t$ and get, by the definition of $\lambda^\mu_s$
\begin{align}
	\sqrt{\sum_{s=\correction{t_0'(t)}+1}^t \norm{\mu - \lambda^\mu_s }_{D_{\omega_{s}}}^2} + \sqrt{C_t}
	&\ge \sqrt{\sum_{s=\correction{t_0'(t)}+1}^t\norm{ \tilde{\mu}_{s-1} - \lambda^\mu_s }_{D_{w_s}}^2}
\\
	&\ge \sqrt{\sum_{s=\correction{t_0'(t)}+1}^t \inf_{\lambda\in\Lambda_m(\tilde{\mu}_{s-1})} \norm{ \tilde{\mu}_{s-1} - \lambda }_{D_{w_s}}^2}
\:. \label{eq:from_mu_to_hat_mu}
\end{align}

\paragraph{Step $7$. Optimistic gains.}

We now replace the term on the right-hand side in \eqref{eq:from_mu_to_hat_mu} by the optimistic gains fed into the online learner. At time $s$, we define optimistic estimates
\begin{align*}
	g_s(\omega)
	:= \sum_{k=1}^K \omega^k \left(\abs{\tilde{\mu}_{s-1}^k - \lambda_s^k} + \sqrt{c_{s-1}^k} \right)^2
\quad\text{where } \lambda_s
	& := \argmin_{\lambda\in\Lambda_m(\tilde{\mu}_{s-1})} \norm{ \tilde{\mu}_{s-1} - \lambda }_{D_{\omega_s}}^2
\: .
\end{align*}

\begin{lemma}\label{lem:optimism_lower_bound}
	For all $\omega \in \Delta_K$ and $s \ge \correction{t_0'(t)}$,
	$g_s(\omega) \ge \inf_{\lambda \in \Lambda_m(\mu)}\norm{ \mu - \lambda }_{D_\omega}^2$ .
\end{lemma}

\begin{proof}
	For all $k \in [K]$, $s > \correction{t_0'(t)}$, and $\lambda \in \mathbb{R}^K$, using Lemma~\ref{lemma:bound_ckt} (to write $(\mu^k - \tilde{\mu}_{s-1}^k)^2 \le c_{s-1}^k$):
\begin{align*}
\left(\mu^k - \lambda^k \right)^2
= \left(\tilde{\mu}_{s-1}^k - \lambda^k + \mu^k - \tilde{\mu}_{s-1}^k \right)^2
	&\le \left(\abs{\tilde{\mu}_{s-1}^k - \lambda^k} + \abs{\mu^k - \tilde{\mu}_{s-1}^k} \right)^2
\\
	&\le \left(\abs{\tilde{\mu}_{s-1}^k - \lambda^k} + \sqrt{c_{s-1}^k}\right)^2
\: .
\end{align*}

	Then, for any $\omega \in \Delta_K$, by noticing that function $f : \lambda \mapsto \sum_{k}^K \omega^k \left( \abs{\tilde{\mu}^k_{s-1} - \lambda^k}^2 + 2\abs{\tilde{\mu}^k_{s-1}-\lambda^k}\sqrt{c^k_{s-1}} \right)$ is nonnegative and that $f(\tilde{\mu}^k_{s-1})=0$:
\begin{align*}
	g_s(\omega) & := \sum_{k=1}^K \omega^k \left(\abs{\tilde{\mu}_{s-1}^k - \lambda_s^k} + \sqrt{c_{s-1}^k}\right)^2
\\
&\ge \inf_{\lambda \in \Lambda_m(\tilde{\mu}_{s-1})}\sum_{k=1}^K \omega^k \left(\abs{\tilde{\mu}_{s-1}^k - \lambda^k} + \sqrt{c_{s-1}^k}\right)^2
\\
&= \sum_{k=1}^K \omega^k c_{s-1}^k + \inf_{\lambda \in \Lambda_m(\tilde{\mu}_{s-1})}\sum_{k=1}^K \omega^k \left( \abs{\tilde{\mu}_{s-1}^k - \lambda^k}^2 + 2 \abs{\tilde{\mu}_{s-1}^k - \lambda^k} \sqrt{c_{s-1}^k}\right)
\\
&\ge \sum_{k=1}^K \omega^k c_{s-1}^k + \inf_{\lambda \in \Lambda_m(\mu)}\sum_{k=1}^K \omega^k \left( \abs{\tilde{\mu}_{s-1}^k - \lambda^k}^2 + 2 \abs{\tilde{\mu}_{s-1}^k - \lambda^k} \sqrt{c_{s-1}^k}\right)
\\
& \mbox{(due to Lemma~\ref{lem:change_alternative_set})}\\
&= \inf_{\lambda \in \Lambda_m(\mu)} \sum_{k=1}^K \omega^k \left(\abs{\tilde{\mu}_{s-1}^k - \lambda^k} + \sqrt{c_{s-1}^k}\right)^2
\\
&\ge \inf_{\lambda \in \Lambda_m(\mu)} \sum_{k=1}^K \omega^k \left(\mu^k - \lambda^k\right)^2 \quad 
 \mbox{ (using the previously derived coordinate-wise majoration)}\\
&= \inf_{\lambda \in \Lambda_m(\mu)}\norm{ \mu - \lambda }_{D_\omega}^2
\: .
\end{align*}

\end{proof}

\correction{We now prove an upper bound on $g_w(\omega)$, which will be useful later.}
\begin{lemma}\label{lem:gain_upper_bound}
For all $s \ge t_0$ and all $\omega \in \triangle_K$, $g_s(\omega) \le 36 M^2$
\end{lemma}
\begin{proof}
	\correction{Using the definition of $c_{t}^k, k \in [K], t \geq 0$ in Lemma~\ref{lemma:bound_ckt}, and $\mu, \lambda_s \in \mathcal{M}$:}
$g_s(\omega) = \sum_{k=1}^K \omega^k \left(\abs{\tilde{\mu}_{s-1}^k - \lambda_s^k} + \sqrt{c_{s-1}^k} \right)^2 \le \sum_{k=1}^K \omega^k \left(\abs{\mu^k - \lambda_s^k} + 2\sqrt{c_{s-1}^k} \right)^2 \le \sum_{k=1}^K \omega^k \left(6M \right)^2 = 36 M^2$.
\end{proof}

We have proved that the estimates are indeed optimistic in the sense that they are an upper-bound to the value of interest, as mentioned in paragraph ``Optimistic gains'' in Subsection~\ref{sub:algorithm} in the main paper. We now bound by how much they overestimate the empirical value.

\begin{lemma}
\begin{align}
\sqrt{\sum_{s=\correction{t_0'(t)}+1}^t \inf_{\lambda\in\Lambda_m(\tilde{\mu}_{s-1})} \norm{ \tilde{\mu}_{s-1} - \lambda }_{D_{\omega_s}}^2}
&\ge \sqrt{\sum_{s=\correction{t_0'(t)}+1}^t g_s(\omega_s)} - \sqrt{C_t}
\: . \label{eq:introduce_optimism}
\end{align}
\end{lemma}

\begin{proof}
We start by a bound for a single $s \in \mathbb{N}$. Using the triangle inequality for an $\ell_2$ norm,
\begin{align*}
\sqrt{g_s(\omega_s)}
&= \sqrt{\sum_{k=1}^K \omega_s^k \left( \abs{\tilde{\mu}_{s-1}^k - \lambda_s^k} + \sqrt{c_{s-1}^k}\right)^2}
\\
&\le \sqrt{\sum_{k=1}^K \omega_s^k \left(\tilde{\mu}_{s-1}^k - \lambda_s^k\right)^2} + \sqrt{\sum_{k=1}^K \omega_s^k c_{s-1}^k}
\: .
\end{align*}
Reordering this inequality, we get
\begin{align*}
\inf_{\lambda\in\Lambda_m(\tilde{\mu}_{s-1})} \norm{ \tilde{\mu}_{s-1} - \lambda }_{D_{\omega_s}}^2
\ge \left( \sqrt{g_s(\omega_s)} - \sqrt{\sum_{k=1}^K \omega_s^k c_{s-1}^k} \right)^2
\: .
\end{align*}

Then, summing over $s \in [\correction{t_0'(t)}+1, t]$ and using $\Vert a - b \Vert \ge \Vert a \Vert - \Vert b \Vert$,
\begin{align*}
\sqrt{\sum_{s=\correction{t_0'(t)}+1}^t \inf_{\lambda\in\Lambda_m(\tilde{\mu}_{s-1})} \norm{ \tilde{\mu}_{s-1} - \lambda }_{D_{\omega_s}}^2}
&\ge \sqrt{\sum_{s=\correction{t_0'(t)}+1}^t \left( \sqrt{g_s(\omega_s)} - \sqrt{\sum_{k=1}^K \omega_s^k c_{s-1}^k} \right)^2}
\\
&\ge \sqrt{\sum_{s=\correction{t_0'(t)}+1}^t g_s(\omega_s)}
	- \sqrt{\sum_{s=\correction{t_0'(t)}+1}^t \sum_{k=1}^K \omega_s^k c_{s-1}^k}
\: .
\end{align*}

\end{proof}

\paragraph{Summary of Steps $4$ to $7$.}

Putting together Equations~\eqref{eq:from_Lambda_mu_to_Lambda_hat_mu}, \eqref{eq:from_mu_to_hat_mu} and \eqref{eq:introduce_optimism}, we proved that under event $\mathcal{E}_t$, for estimates $(\tilde{\mu}_s)_{s \le t}$ such that we have a bound $c_s^k$ on $(\tilde{\mu}_s^k - \mu^k)^2$ for all $s \in \{\correction{t_0'(t)}+1, \ldots, t\}$ and $k \in [K]$,
\begin{align*}
	\sqrt{\inf_{\lambda\in\Lambda_m({\mu})}\norm{\mu - \lambda }_{D_{W_{t}}}^2} + 2\sqrt{C_t}
	\ge \sqrt{\sum_{s=\correction{t_0'(t)}+1}^t g_s(\omega_s)} \: .
\end{align*}
Note that $\inf_{\lambda\in\Lambda_m({\mu})}\norm{\mu - \lambda }_{D_{W_{t}}}^2 \le t \max_{\omega \in \Delta_K} \inf_{\lambda\in\Lambda_m({\mu})}\norm{\mu - \lambda }_{D_\omega}^2 = 2 t H_\mu$. We then get
\begin{align*}
	\inf_{\lambda\in\Lambda_m({\mu})}\norm{\mu - \lambda }_{D_{W_{t}}}^2
	&\ge \sum_{s=\correction{t_0'(t)}+1}^t g_s(\omega_s) - 4C_t - 4 \sqrt{2 t C_t H_\mu}
\: ,
\end{align*}
which ends proving Theorem~\ref{th:sampling_rule}.

\paragraph{Step $8$. No-regret property.}

The first $t_0$ rounds are used to initialize our algorithm. After that, we use a learner with small regret. \correction{We will bound the gains between $t_0$ and $t_0'(t) = \max\{t_0, \sqrt{t}\}$ by $36M^2$ (see Lemma~\ref{lem:gain_upper_bound}).}
We use the regret bound of the learner (refer to Definition~\ref{def:no-regret} in the main paper) to get that, for some additional low-order term $C_{\mathcal{L}}(K,B)\sqrt{t}$, and by combining Theorems~\ref{th:stopping_rule} and \ref{th:sampling_rule}:
\begin{align*}
2\beta_{t,\delta}
	&\ge \sum_{s=\correction{t_0'(t)}+1}^t g_s(\omega_s) - h_\delta(t) - r(t) - 4C_t - 4 \sqrt{2 t C_t H_\mu}
\\
&\ge \correction{\sum_{s=t_0+1}^t g_s(\omega_s) - h_\delta(t) - r(t) - 4C_t - 4 \sqrt{2 t C_t H_\mu} - \max\{\sqrt{t} - t_0, 0\}36 M^2}
\\
	&\ge \max_{\omega\in\Delta_K}\sum_{s=t_0+1}^t g_s(\omega) - h_\delta(t) - r(t) - 4C_t - 4 \sqrt{2 t C_t H_\mu} - C_{\mathcal{L}}(K,B)\sqrt{t}
	\\&\qquad \correction{- \max\{\sqrt{t} - t_0, 0\}36 M^2}
\: .
\end{align*}
A specific upper bound on the regret for the learner AdaHedge used in the implementation of \algo is mentioned in Lemma~\ref{lem:adahedge}. 

\paragraph{Step $9$. From the optimal gain to the lower bound value.}

Finally, we can relate the optimal optimistic gain of the learner to the value of the lower bound. Using the optimism (Lemma~\ref{lem:optimism_lower_bound}),
\begin{align*}
	\max_{\omega\in\Delta_K}\sum_{s=t_0+1}^t g_s(\omega)
	\ge \max_{\omega\in\Delta_K}\sum_{s=t_0+1}^t\inf_{\lambda\in\Lambda_m({\mu})} \norm{\mu - \lambda }_{D_{\omega}}^2
	= (t-t_0) \underbrace{\max_{\omega\in\Delta_K}\inf_{\lambda\in\Lambda_m({\mu})} \norm{\mu - \lambda }_{D_{\omega}}^2}_{=~ 2 H_\mu}.
\end{align*}

\paragraph{Step $10$. Computing the sample complexity.}

We thus have obtained an inequality of the form
\begin{align*}
2\beta_{t,\delta}
&\geq 2 t H_\mu - h_\delta(t) - r(t) - 4C_t - 4 \sqrt{2 t C_t H_\mu} - C_{\mathcal{L}}(K,B)\sqrt{t} - 2 t_0 H_\mu
	\\ & \qquad \correction{- \max\{\sqrt{t} - t_0, 0\}36 M^2}
\: ,
\end{align*}
from which we can obtain the desired sample complexity bound. Remember that

\begin{itemize}[noitemsep]
	\item $\beta_{t,\delta} := \min\left\{\beta_{t,\delta}^{\mathrm{uns}}, \beta_{t,\delta}^{\mathrm{lin}}\right\}$
	\item $r(t) := M^2 \sqrt{2t\left(\log(4 \times \numel t^4) + d\log\frac{6(M + \varepsilon)Lt}{\sqrt{\Gamma(A)}} + K\log\max\{4 \varepsilon t, 1\}\right)} + 2 + 8M$
	\item $h_\delta(t) := \sqrt{8\beta_{t,\delta}f(t)} + f(t)$
	\item $f(t) := 2\min\left\{\alpha_t^{\mathrm{uns}}, \alpha_t^{\mathrm{lin}} + 2t \etaboundsq \right\}$
\end{itemize}

where:

\begin{align*}
\beta_{t,\delta}^{\mathrm{uns}}
&:= 2 K \overline{W}\left( \frac{1}{2K}\log\frac{2e}{\delta} + \frac{1}{2}\log(8eK\log t)\right)
\: , \\
\beta_{t,\delta}^{\mathrm{lin}}
&:= \frac{1}{2}\left( 4\sqrt{t}\varepsilon
	+ \sqrt{2}\sqrt{1 + \log\frac{1}{\delta} + \left(1+\frac{1}{\log(1/\delta)} \right) \frac{d}{2}\log\!\left(1+\frac{t}{2d}\log\frac{1}{\delta} \right)} \right)^2
\: , \\
\alpha_t^{\mathrm{uns}}
&:= 2 K \overline{W}\left( \frac{1}{2K}\log(14 e t^3) + \frac{1}{2}\log(8eK\log t)\right)
= \beta_{t, 1/\numel t^3}^{\mathrm{uns}}
\: , \\
\alpha_t^{\mathrm{lin}}
&:= \log(\numel t^2) + d\log\!\left(1+\frac{t}{2d} \right)
\: , \\
c_t^k
&:= \min\left\{8 \etaboundsq + 4 \alpha_{\correction{t^2}}^{\mathrm{lin}}  \norm{ \phi_k }_{V_{t}^{-1}}^2, \frac{2\alpha_{\correction{t^2}}^{\mathrm{uns}}}{N_t^k}, 4M^2 \right\}
\: , \\
C_t
&:= \sum_{s=t_0+1}^t \sum_{k=1}^K w_s^k c_{s-1}^k
\le 8 \etaboundsq t + 2 \alpha_{\correction{t^2}}^{\mathrm{lin}} \left( \sqrt{2 t \log(\numel t^2)} + d \log \left( 1 + \frac{t}{d} \right) \right)
\: , \\
C_t
&\le 4 M^2 \sqrt{2 t \log (\numel t^2)} + 4 M^2 K + 2 K \alpha_{\correction{t^2}}^{\mathrm{uns}}\log (e t)
\: .
\end{align*}

Combining this bound with Lemma~\ref{lemma:ccl_sample_complexity} proves Theorem~\ref{th:sample_complexity} in the main paper.


\subsection{Using several estimates}

If we employ two sets of estimates, with corresponding optimism functions $(g^i_{s}(\omega))_{i\in\{1,2\}}$ and bounds $c_{i,s}^k$, we get for $i \in \{1, 2\}$,
\begin{align*}
	\inf_{\lambda\in\Lambda_m({\mu})}\norm{\mu - \lambda }_{D_{W_{t}}}^2
	&\ge \max_{i \in \{1, 2\}} \left( \sum_{s=t_0+1}^t g^i_s(\omega_s) - 4C^i_{t} - 4 \sqrt{2 t C^i_{t} H_\mu} \right)
\\
	&\ge \sum_{s=\correction{t_0'(t)}+1}^t \min_{i \in \{1, 2\}} g^i_s(\omega_s) - \min_{i \in \{1, 2\}} \left(4C^i_{t} + 4 \sqrt{2 t C^i_{t} H_\mu} \right)
\: ,
\end{align*}
where the quantity $C^i_{t}$ is similarly defined as $C_t$, with respect to gains $g^i_t$.

Since the minimum of concave functions is concave, $g_s: \omega \mapsto \min_{i \in \{1, 2\}} g^i_{s}(\omega)$ is concave (which allows the use of a regret-minimizing algorithm, see Subsection~\ref{sub:regret_of_adahedge}). It satisfies the inequality of Lemma~\ref{lem:optimism_lower_bound} and its gradient is the gradient of $g^{i^\star}_s(\omega)$ for $i^\star \in \argmin_{i \in \{1,2\}} g^i_s(\omega)$.

\subsection{Aggressive Optimism} 

If we are happy with an algorithm which is within a factor $2$ of the lower bound for the $\log \frac{1}{\delta}$ term instead of insisting on a factor $1$, we can use a different, more aggressive optimism.
Take
\begin{align*}
\widehat{g}_s(\omega_s)
:= 2\sum_{k=1}^K \omega^k \left((\tilde{\mu}_{s-1}^k - \lambda_s^k)^2 + c_{s-1}^k\right)
\quad\text{ where } \lambda_s
	& := \argmin_{\lambda\in\Lambda_m(\tilde{\mu}_{s-1})} \norm{ \tilde{\mu}_{s-1} - \lambda }_{D_{\omega_s}}^2
\: .
\end{align*}
The main difference is that the term added to $(\tilde{\mu}_{s-1}^k - \lambda_s^k)^2$ is of order $c_{s-1}^k$ instead of $\sqrt{c_{s-1}^k}$. In an unstructured bandit, that means $1/N_t^k$ instead of $1/\sqrt{N_t^k}$. Let us prove the counterpart to Lemma~\ref{lem:optimism_lower_bound} for these new gains:

\begin{lemma}\label{lem:optimism_lower_bound_aggressive}
	For all $\omega \in \Delta_K$,
	$\widehat{g}_s(\omega_s) \ge \inf_{\Lambda \in \lambda_m(\mu)}\norm{ \mu - \lambda }_{D_\omega}^2$.
\end{lemma}

\begin{proof}
	For all $k \in [K]$ and $\lambda \in \mathbb{R}^K$, using Lemma~\ref{lemma:bound_ckt}
\begin{align*}
(\mu^k - \lambda^k)^2
\le 2(\tilde{\mu}_{s-1}^k - \lambda^k)^2 + 2(\mu^k - \tilde{\mu}_{s-1}^k)^2
&\le 2(\tilde{\mu}_{s-1}^k - \lambda^k)^2 + 2c_{s-1}^k
\: .
\end{align*}
Then, since $\omega \in \Delta_K$
\begin{align*}
2\sum_{k=1}^K \omega^k \left((\tilde{\mu}_{s-1}^k - \lambda^k)^2 + c_{s-1}^k\right)
\ge \sum_{k=1}^K \omega^k (\mu^k - \lambda_s^k)^2
	= \norm{ \mu - \lambda_s }_{D_\omega}^2
	&\ge \inf_{\Lambda \in \lambda_m(\mu)}\norm{ \mu - \lambda }_{D_\omega}^2
\: .
\end{align*}
\end{proof}

Then, using Lemma~\ref{lem:optimism_lower_bound_aggressive} and the definition of $\lambda_s$, we have
\begin{align*}
	\widehat{g}_s(\omega) - 2\inf_{\lambda\in\Lambda_m(\tilde{\mu}_{s-1})} \norm{\tilde{\mu}_{s-1} - \lambda }_{D_{\omega_s}}^2
	&= \sum_{k=1}^K \omega_s^k \left[ 2(\tilde{\mu}_{s-1}^k - \lambda_s^k)^2 + 2c_{s-1}^k - 2( \tilde{\mu}_{s-1}^k - \lambda_s^k )^2 \right]
\\
&= 2\sum_{k=1}^K \omega_s^k c_{s-1}^k
\: .
\end{align*}
So now we can prove a counterpart to Step $7$ in the proof of Theorem~\ref{th:sampling_rule}:
\begin{align*}
&\sum_{s=t_0+1}^t \inf_{\lambda\in\Lambda_m(\tilde{\mu}_{s-1})} \norm{ \tilde{\mu}_{s-1} - \lambda }_{D_{\omega_s}}^2
\\
&= \frac{1}{2}\sum_{s=t_0+1}^t \widehat{g}_s(\omega_s)
	- \frac{1}{2}\sum_{s=t_0+1}^t \left( \widehat{g}_s(\omega_s) - 2\inf_{\lambda\in\Lambda_m(\tilde{\mu}_{s-1})} \norm{\tilde{\mu}_{s-1} - \lambda }_{D_{\omega_s}}^2 \right)
\\
&\ge \frac{1}{2}\sum_{s=t_0+1}^t \widehat{g}_s(\omega_s) - C_t
\: .
\end{align*}

\subsection{Regret of AdaHedge}
\label{sub:regret_of_adahedge}

\begin{lemma}[\cite{de2014follow}]\label{lem:adahedge}
On the online learning problem with $K$ arms and gains $g_s(\omega) := \sum_{k\in[K]} \omega^k U_s^k$ for $s\in[t]$, AdaHedge, predicting $(\omega_s)_{s\in[t]}$, has regret
\begin{align*}
R_t&:= \max_{\omega \in \Delta_K}\sum_{s=1}^t g_s(\omega) -g_s(\omega_s)
\le 2\sigma\sqrt{t\log(K)} + 16\sigma(2+\log(K)/3) \: ,\\
\text{where }
\sigma & := \max_{s\le t}~ \left(\max_{k\in[K]}U_s^{k}- \min_{k\in[K]}U_s^{k} \right) \:.
\end{align*}
\end{lemma}

We recall here the ``gradient trick'', which we can use to employ AdaHedge on any concave gains.
If for any time $t > 0$, the loss function $\ell_t$ at that time is convex, then for all $\omega \in \Delta_K$,
\begin{align*}
\sum_{s=1}^t \ell_t(\omega_t) - \ell_t(\omega)
\le \sum_{s=1}^t (\omega_t - \omega)^\top \nabla \ell_t(\omega_t)
\end{align*}
Running a regret-minimizing algorithm with loss $\bar{\ell}_t(\omega) = \omega^\top \nabla \ell_t(\omega_t)$ then leads to a regret bound on $\ell_t$.

\subsection{Technical tools}
\label{sub:technical_tools}

\paragraph{Generic bounds on vector norms.}

\begin{lemma}\label{lemma:norm-linear-part}
Let $\theta \in \mathbb{R}^d, \eta \in \mathbb{R}^K$ be such that $\norm{ \eta }_\infty \le \varepsilon$ and $\norm{ A \theta + \eta}_\infty \le M$. Then
\begin{align*}
	\norm{\theta}_2 \leq \frac{M + \varepsilon}{\sqrt{\Gamma(A)}}
\: ,
\end{align*}
where $\Gamma(A) := \max_{\omega \in \Delta_K} \sigma_{\min}\left(\sum_{k=1}^K \omega^k \phi_k\phi_k^\top \right)$, where $\sigma_{\min}(M)$ is the minimal singular value of matrix $M$.

\end{lemma}

\begin{proof}
For $\lambda := A \theta + \eta$ with $\norm{ \eta }_\infty \le \varepsilon$ and $\norm{ \lambda }_\infty \le M$,
\begin{align}\label{eq:norm-linear-part1}
\norm{A\theta}_\infty
&= \max_{k \in [K]} \abs{\phi_k^\top \theta}
\ge \norm{\theta}_2 \min_{u \in \mathbb{R}^d:\norm{u}_2=1}  \max_{k \in [K]} \abs{\phi_k^\top u}
\: ,
\end{align}
using that the value for $\theta/\norm{\theta}_2$ is larger than the minimum over $u \in \mathbb{R}^d$ with $\norm{u}_2=1$. On the other hand, successively using the triangle inequality and the boundedness assumptions,
\begin{align}\label{eq:norm-linear-part2}
	\norm{A\theta}_{\infty}
	&\le \norm{A\theta + \eta}_{\infty} + \norm{\eta}_{\infty}
\le M + \varepsilon
\: .
\end{align}
Note also that
	\begin{align}\label{eq:norm-linear-part3}
	\min_{u \in \mathbb{R}^d:\norm{u}_2=1}  \max_{k \in [K]} \abs{\phi_k^\top u}^2
	= \min_{u \in \mathbb{R}^d:\norm{u}_2=1} \max_{\omega \in \Delta_K} \norm{u}^2_{\left(\sum_{k=1}^K \omega^k \phi_k\phi_k^\top \right)}
\geq \underbrace{\max_{\omega \in \Delta_K} \sigma_{\min}\left(\sum_{k=1}^K \omega^k \phi_k\phi_k^\top \right)}_{:=~ \Gamma(A)}
\: ,
\end{align}
	where the inequality stems from the min-max theorem (principle for singular values). Finally, by combining the three inequalities \eqref{eq:norm-linear-part1}, \eqref{eq:norm-linear-part2} and \eqref{eq:norm-linear-part3}, $\norm{\theta}_2 \leq \frac{M + \varepsilon}{\sqrt{\Gamma(A)}}$ .

\end{proof}

The term $\Gamma(A)$ depends only on the set of linear features $\left\{\phi_k \right\}_{k \in [K]}$. In the unstructured case (where $\phi_k = e_k$), we have $\Gamma(A) = \frac{1}{K}$. However, in a structured case with $d \ll K$, $\Gamma(A)$ can be much smaller. For instance, when $A$ contains the canonical basis of $\mathbb{R}^d$, we have $\Gamma(A) \geq \frac{1}{d}$.

\section{Experimental evaluation}
\label{app:experiments}


\subsection{Computational architectures}

Experiments on simulated datasets (Experiments (A), (B), (C)) were run on a personal computer (processor: Intel Core i$7-8750$H, cores: $12$, frequency: $2.20$GHz, RAM: $16$GB).

Experiment (D) was run on a personal computer (processor: Intel Core i$7-9700$K, cores: $8$, frequency: $3.60$GHz, RAM: $16$GB).

Experiment (E) was run on a internal cluster (processor: Westmere E$56$xx/L$56$xx/X$56$xx (Nehalem$-$C), cores: $24$, frequency: $3.2$GHz, RAM: $155$GB).

\subsection{License for the assets}

\textbf{Experiment (D).} The drug repurposing dataset for epilepsy was proposed in~\cite{reda2021top}, and made publicly available under the MIT license.

\textbf{Experiment (E).} The original dataset Last.fm is publicly available online at \texttt{https://www.last.fm/} under CC BY-SA 4.0.

\textbf{Experimental code.} The code hosted at \href{https://github.com/clreda/misspecified-top-m}{\texttt{https://github.com/clreda/misspecified-top-m}} is under MIT license.

\subsection{Extracting representations from real datasets}

We describe in detail the procedure we adopted to extract misspecified linear representations from the real-world datasets of Experiment (D) and (E). In both cases, we adopted a very similar procedure based on training neural networks as the one used in \citep{papini2021leveraging}. We describe all its steps for the sake of completeness.

\paragraph{Step $1$. (Data preprocessing)} 

First, we start from preprocessing the raw data to obtain a dataset containing tuples of the form $(\phi,x)$, where $\phi \in \mathbb{R}^d$ is an arm feature and $x \in \mathbb{R}$ is a reward. The drug repurposing dataset used in~\citep{reda2021top} (hosted on their repository) is already available in this form, with a total of $509$ arms representing different drugs, $d=67$ features representing genes, and, for each of them, $18$ reward samples representating the responses of $18$ different patients to such drugs. Out of those $509$ arms, we filter out those which outcomes are unknown (associated ``true'' scores are set to $0$, according to the file of scores available on the same repository). Then $175$ arms (representing either antiepileptics, with score equal to $1$, and proconvulsants, with score equal to $-1$) are left.

On the other hand, the Last.fm dataset is in a different form; it contains information about the music artists listened by each user of the system. As done in~\citep[][Appendix F.4]{papini2021leveraging},  we first preprocessed the data by keeping only artists listened by at least $120$ users and users that listened at least to $10$ different artists.
We thus obtained $U=1,322$ users and $A=103$ artists. The result is a matrix in $\mathbb{R}^{U\times A}$ containing the number of times each user listened to each artist (which we treat as reward). We then extract user-artist features by applying low-rank Singular Value Decomposition on this matrix and keeping only the top $80$ singular values. This yields $U$ $d$-dimensional user features, and $A$ $d$-dimensional artist features, where $d=80$. The final user-artist features are the concatenation of the two, which yields a dataset with $U \times A$ tuples $(\phi,x) \in \mathbb{R}^d \times \mathbb{N}$ in our desired form.

\paragraph{Step $2$. (Neural-network training)}

For both datasets, the second step consists in training a neural network to regress from $\phi$ to $x$. First, we split the datasets randomly into $80\%$ training set and $20\%$ test set. Then, we train a neural network with two hidden layers of size $256$, rectified linear unit activations, and a linear output layer of $8$ neurons. We obtain an $R^2$ score on the test set of $0.92$ for the drug repurposing data, and $0.85$ for the Last.fm one.

\paragraph{Step $3$. (Extracting a linear representation)}

Finally, we extract a linear model from the trained neural network by taking, for each input $\phi \in \mathbb{R}^d$ in our data, the $8$-dimensional features (i.e., activations) computed in the last layer together with the corresponding parameters. When specified, a subset of arm features is considered instead of the whole dataset. Then, in that case, we apply a lossless dimensionality reduction to make sure these features span the whole space. The reduced features are the one we feed into our learning algorithms (Experiment (D): $d=5$, $K=10$; Experiment (E.i): $d=8$, $K=103$; Experiment (E.ii): $d=7$, $K=50$). Moreover, we compute the maximum absolute error of this linear model in predicting the original data, and use that as a proxy for $\varepsilon$.

Note that, since the Last.fm data is in the form of user-artist features and, in our problem, we consider the artists only as arms, the representation we select for our experiments is obtained by choosing a user randomly among the available $U=1,322$ ones.

\correction{Moreover, in Step $3$, we apply a dimension reduction procedure on features to ensure the feature matrix is not ill-conditioned, at the cost of increasing the norm of misspecification $\varepsilon$. This is needed in order to reduce the length $t_0$ of the initialization sequence~; remember that in Appendix~\ref{subapp:initialization} we showed that $t_0$ is upper-bounded by quantity $d \left\lceil \frac{2L^2}{\Gamma'(A)} \right\rceil$, where $\Gamma'(A) := \min \left\{ \sigma_{\min}\left( \sum_{k \in \mathcal{B}} \phi_k\phi_k^\top \right) \mid \mathcal{B} \text{ barycentric spanner of $A$ of size }d \right\}$ crucially relies on the conditioning of $A$. How much misspecification is required to improve the conditioning of the matrix is an open question (which has also been raised in other recent works~\cite{papini2021leveraging}). Ideally, one would want to learn a representation of the data which balances those two effects, but we leave such a method to future investigations.}

\subsection{\correction{Numerical results for sample complexity}}

\begin{table}[h!]
	\caption{\correction{Statistics (mean $\pm$ standard deviation rounded up to the next integer) for Experiment (A). Names are similar to those in the first two leftmost plots of Figure~\ref{fig:experimentABC}. Values are averaged across $500$ iterations. LinGapE is not $\delta$-correct in the setting where $\varepsilon=5$ (with $\delta=0.05$).}}
	\label{tab:experimentA}
	\centering
	\hspace{-1cm}
	\begin{tabular}{lcc}
		\toprule
		\multicolumn{3}{c}{}\\
		 Sample complexity    &	LinGapE	& \algo\\
		\midrule
		$\varepsilon=0$	&	$577 \pm 348$	&	$890 \pm 546$\\
		$\varepsilon=5$	&	$553 \pm 536$	&	$5,156 \pm 3,629$\\
		\bottomrule
	\end{tabular}
\end{table}

\begin{table}[h!]
	\caption{\correction{Statistics (mean $\pm$ standard deviation rounded up to the next integer) for Experiments (D) and (E). Names are similar to those in the plots of Figure~\ref{fig:experimentDE}. Values are averaged across $100$ iterations. Note that LinGapE is not $\delta$-correct (with $\delta=0.05$) in Experiment E.}}
	\label{tab:experimentDE}
	\centering
	\hspace{-1cm}
	\begin{tabular}{lcc}
		\toprule
		\multicolumn{3}{c}{}\\
		 Sample complexity    &	LinGapE	& \algo\\
		\midrule
		Experiment D	&	$21,593 \pm 8,296$	&	$42,751 \pm 13,942$\\
		Experiment E	&	$10,907 \pm 4,474$	&	$289,703 \pm 185,205$\\
		\bottomrule
	\end{tabular}
\end{table}


\subsection{Tricks to reduce sample and computational complexity on large instances (D) and (E)}

In large instances (more particularly on our real-life datasets in Section~\ref{sec:experimental_evaluation} in the main paper), the number of arms can be large, and the theoretically supported version of the algorithm \algo might become too slow. Based on our experiments, we have decided to change some parts of the algorithm.

\textbf{No optimism.} As shown in the rightmost plot in Figure~\ref{fig:experimentABC} in the main paper, empirical gains (i.e., without any optimistic bonus) actually considerably improve sample complexity.

\textbf{Restriction of the set of arms used in the sampling rule.} In order to compute the gains which are fed to the learner, \algo needs to compute the closest alternative, which implies solving $m(K-m)$ quadratic optimization problems, one for each pair of arms $(i,j)$, with $i$ among the $m$ best arms and $j$ among the $K-m$ worse arms (as defined in Theorem~\ref{th:lower_bound} in the main paper).
We observed that the majority of arms never realize the minimum over $(i,j)$ of the distance to the alternative, and in hindsight they could be ignored.
We mimicked that behavior by only considering a subset of arms at each step. We kept $m+d$ arms in memory, consisting of the recent argmins $i,j$ for the closest alternative model, and sampled $d$ more among the $K-m$ worse arms.
The resulting set of at most $m+2d$ arms is then used to compute the closest alternative. 
The gain in computational complexity is large when $K \gg d$, since we solve $m(m+2d)$ minimization problems instead of $m(K-m)$.
We don't use that trick to compute the stopping rule, since we would not be guaranteed to preserve $\delta$-correctness.


\textbf{Geometric grid for testing the stopping rule.} Instead of checking the stopping criterion at each learning round of the algorithm, we suggest testing it on a geometric grid (that is, testing it for the first time at $t_1$, and then retest it at $\gamma t_1$, then at $\gamma^2 t_1$, etc. where $1 < \gamma \leq 1.3$ in practice), and restrict the computation of the stopping rule to a random subset of arms. In our experiments, we have actually used $\gamma = 1.2$. When using the geometric grid, we can obtain a sample complexity bound of the same form as in Theorem~\ref{th:sample_complexity} in the main paper, except that $T_0(\delta)$ is replaced by $\gamma T_0(\delta)$.

Together, the sampling and stopping rule changes reduce the time needed to complete a run of the algorithm by a factor $29$ on Experiment (D), while increasing the sample complexity by a factor $1.2$ (refer to Table~\ref{tab:experimentd_stats}, comparing algorithmic versions named ``AdaHedge'' and ``Default''). See the middle plot of Figure~\ref{fig:compare_empirical_algorithm} for a comparison of the sample complexity.

\begin{table}[h!]
	\caption{\correction{Statistics (mean $\pm$ standard deviation rounded up to the next integer) for Experiment (D), with different versions of \algo.} Names are similar to those in the \correction{center}
	plot of Figure~\ref{fig:compare_empirical_algorithm}. Values are averaged across $100$ iterations.}
	\label{tab:experimentd_stats}
	\centering
	\hspace{-1cm}
	\begin{tabular}{lccc}
		\toprule
		\multicolumn{4}{c}{}\\
		Per run     &	AdaHedge	& Greedy	&	Default\\
		\midrule
		Average runtime (in sec.)	&	$69 \pm 20$	&	$76 \pm 178$	&	$1,993 \pm 1,311$\\
		Average sample complexity	&	$51,965 \pm 15,260$	&	$52,108 \pm 125,230$	&	$42,751 \pm 13,943$\\
		\bottomrule
	\end{tabular}
\end{table}

We have also tested another learner which is less conservative than AdaHedge, to check if this improves sample complexity (note that we did not show any experiment using this trick in the main paper):

\textbf{Change of learner.} We replace AdaHedge by a Greedy/Follow-The-Leader learner combination for the computation of $(\omega_t, \lambda_t)$. 

We have run three versions of \algo on the dataset of Experiment (D): the default \algo, the modified version with learner AdaHedge, and another modified version with the Greedy learner. We have also launched the latter two on Experiment \correction{(E)}. 
See Figure~\ref{fig:compare_empirical_algorithm}. 

\begin{figure}
	\centering
	\includegraphics[scale=0.15]{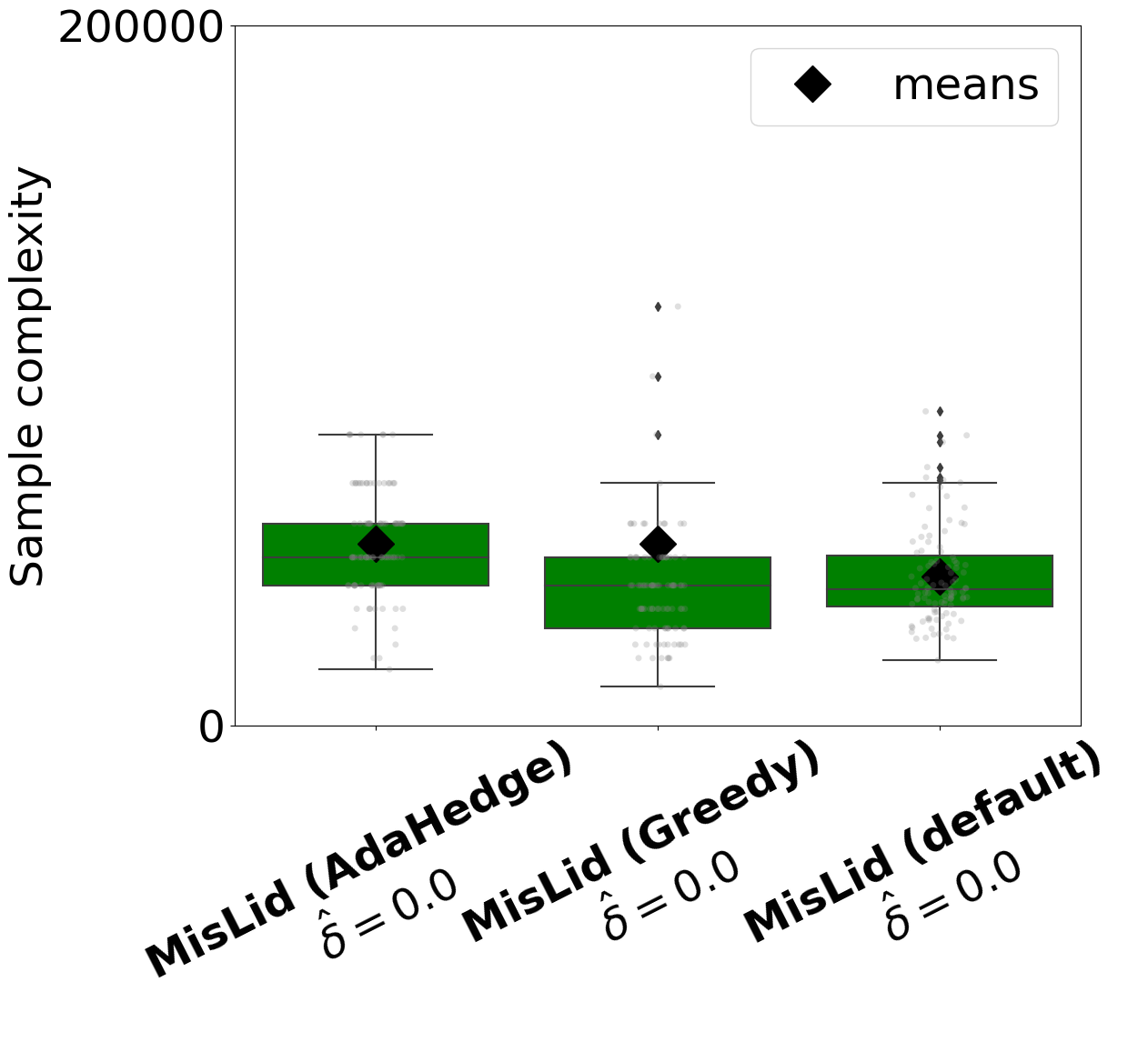}
	\includegraphics[scale=0.15]{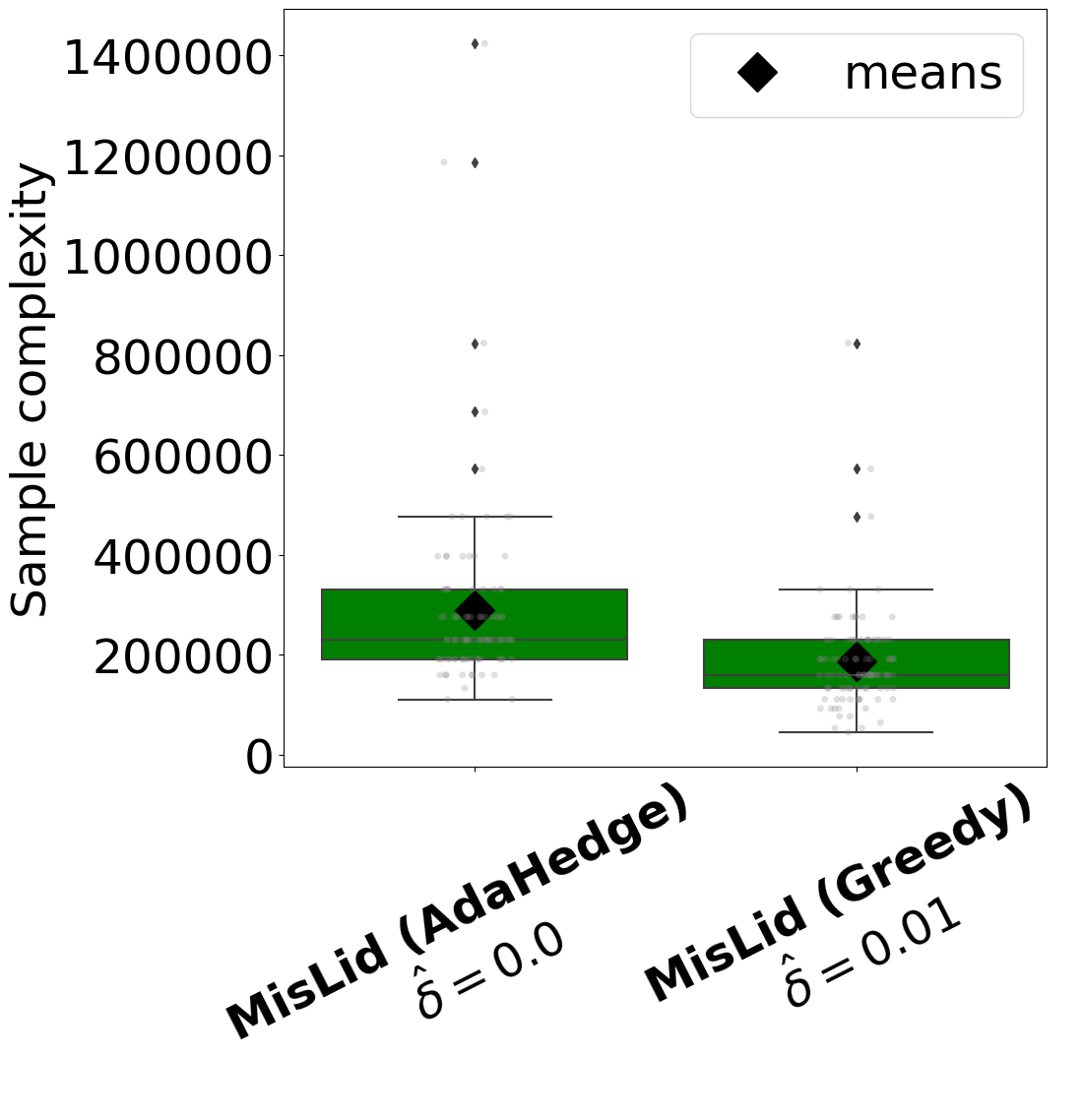}
	\caption{
		Comparison between default \algo, modified \algo using learner AdaHedge, and modified \algo using learner Greedy (Experiment (D) \correction{\textit{(left)}}, Experiment \correction{(E)})
		. Unfortunately, one outlier in the runs using learner Greedy in Experiment (D), above $1,200,000$ rounds, would prevent the readability of the plot if figured. To overcome this issue, we have cropped out the $y$-axis above $200,000$ in this plot.}
	\label{fig:compare_empirical_algorithm}
\end{figure}

\end{document}